\documentclass{article}

\usepackage{microtype}
\usepackage{graphicx}
\usepackage{subfigure}
\usepackage{booktabs} 

\usepackage{hyperref}

\usepackage[accepted]{icml2023}

\usepackage{amsmath}
\usepackage{amssymb}
\usepackage{mathtools}
\usepackage{amsthm}
\usepackage{booktabs} 
\usepackage{threeparttable} 
\usepackage{parskip}
\usepackage{microtype}
\usepackage{graphicx}
\usepackage{subfigure}

\usepackage[capitalize,noabbrev]{cleveref}

\theoremstyle{plain}
\newtheorem{theorem}{Theorem}[section]
\newtheorem{proposition}[theorem]{Proposition}
\newtheorem{lemma}[theorem]{Lemma}

\theoremstyle{definition}

\newtheorem{assumption}[theorem]{Assumption}
\theoremstyle{remark}
\newtheorem{remark}[theorem]{Remark}

\DeclareRobustCommand\optionalsec[1]{%
  \ifnum\pdfstrcmp{#1}{\thesection}=0\else#1.\fi
}

\newcommand{\indep}{\raisebox{0.05em}{\rotatebox[origin=c]{90}{$\models$}}}

\usepackage[textsize=tiny]{todonotes}

\icmltitlerunning{Meta-Learners for Multi-Valued Heterogeneous Effects}

\begin{document}

\twocolumn[
\icmltitle{Comparison of meta-learners for estimating multi-valued treatment heterogeneous effects}



\icmlsetsymbol{equal}{†}
\icmlsetsymbol{namr}{*}

\begin{icmlauthorlist}
\icmlauthor{Naoufal Acharki†}{yyy,comp}
\icmlauthor{Ramiro Lugo‡}{comp}
\icmlauthor{Antoine Bertoncello}{comp}
\icmlauthor{Josselin Garnier}{yyy}
\end{icmlauthorlist}

\icmlaffiliation{yyy}{CMAP, Ecole polytechnique, Institut Polytechnique de Paris, Palaiseau, France}
\icmlaffiliation{comp}{TotalEnergies One Tech, Palaiseau, France}

\icmlcorrespondingauthor{Naoufal Acharki}{naoufal.acharki@gmail.com}

\icmlkeywords{Machine Learning, Causal Inference, Heterogeneous Effects, Multi-valued treatments}

\vskip 0.3in
]



\printAffiliationsAndNotice{} 

\begin{abstract}
Conditional Average Treatment Effects (CATE) estimation is one of the main challenges in causal inference with observational data. In addition to Machine Learning based-models, nonparametric estimators called \textit{meta-learners} have been developed to estimate the CATE with the main advantage of not restraining the estimation to a specific supervised learning method. This task becomes, however, more complicated when the treatment is not binary as some limitations of the naive extensions emerge. This paper looks into meta-learners for estimating the heterogeneous effects of multi-valued treatments. We consider different meta-learners, and we carry out a theoretical analysis of their error upper bounds as functions of important parameters such as the number of treatment levels, showing that the naive extensions do not always provide satisfactory results. We introduce and discuss meta-learners that perform well as the number of treatments increases. We empirically confirm the strengths and weaknesses of those methods with synthetic and semi-synthetic datasets.
\end{abstract}

\section{Introduction}
\label{introduction}

\vspace{0.08in}

With the rapid development of Machine Learning (ML) and its efficiency in predicting outcomes, the question of counterfactual prediction \textit{"what would happen if ?"} arises. Engineers and specialists want to know how the outcome would be affected after an intervention on a parameter. It will help them personalize the parameter at efficient levels and optimize the outcome. Recently, much effort has been devoted to supervised ML to find the optimal intervention strategy. Yet, the results are not always convincing. These models cannot distinguish between correlations and causal relationships in the data \citep{Pearl2019}.

Based on the Potential Outcomes theory \citep{Neyman1923, Rubin1974}, epidemiologists and statisticians have developed a set of tools that reduce the inference of causal effects to statistical inference under certain assumptions about, for example, the data-generating process. They have been successfully applied in many fields such as medicine \citep{ALaaGP2017}, economics \citep{Knaus_2020}, public policy \citep{Imai_strauss_2011} and marketing \citep{Diemert2018} to infer causal effects. However, most of these studies are limited to a binary treatment setting whereas many causal questions in real-world cases are not binary. It would be helpful to give an in-depth analysis of the impact of the treatment across its possible levels instead of just considering a binary scenario where the treatment is either assigned or not. In addition, the heterogeneity of effects may provide valuable information regarding this treatment's effectiveness and help users personalize their intervention policies and strategies. 

The problem of Causal Inference beyond binary treatment settings is gaining attention from the Causal Inference community. There are two major challenges: Firstly, the lack of the ground truth due to the fundamental problem of causal inference \citep{Holland1986} makes Heterogeneous Treatment Effects estimation more challenging \citep{Alaa18} as standard metrics can not be used to assess performances. Secondly, binarizing the multi-valued treatment setting leads to a violation of the Stable Unit Treatment Value Assumption (SUTVA) as it violates the principle of "no hidden variations of treatment". It may yield a bias known as position bias in recommendation systems \citep{Chen2020,Wu2022}. It happens when some units tend to have/select high treatment values and, therefore, there is a hidden variation of the treatment that is not taken into consideration when one attempt to binarize the treatment. From a statistical point of view, this bias was established by \citet{Heiler2021} who show that the binarization of multi-valued treatments does not disassociate the heterogeneity of the treatment from the heterogeneity of the effects of each value.

The extension of the binary setting does not seem trivial as several versions are possible and turn out to have different behaviours. Moreover, -- to the best of our knowledge -- there is no result so far on the impact of the number of possible treatments on the performances of heterogeneous treatment effect (nonparametric or ML-based) estimators. 

In this paper, we study the problem of nonparametric estimation of Heterogeneous Treatment Effects, also known as Conditional Average Treatment Effects (CATEs), for multi-valued treatments. We consider nonparametric estimators, also referred to as \textit{meta-learners} \citep{Kunzel_2019} or \textit{model-agnostic algorithms} \citep{Curth21a}. We put our main focus on discussing the theoretical properties of meta-learners for estimating CATEs. Finally, along the lines of \citet{Curth21a}, we consider it significant to draw strengths and weaknesses theoretically and compare scenarios in which some methods would perform better than others.


\paragraph{Contributions.}

The paper considers meta-learners for multi-valued treatments. First, we generalize \textit{meta-learners} to the multi-treatment setting for CATE estimation. We overview, in particular, Debiased Machine Learning (DML) estimators in observational studies and we establish a new version for the X-learner based on regression adjustment for multi-valued treatments. We also highlight the multi-treatments R-learner's main drawbacks. Second, and this is the major contribution of the paper, we conduct a theoretical analysis of the proposed meta-learners for multiple treatments based on an asymptotic bias-variance analysis (see \citet{Silverman1986} for an example of this analysis for kernel density estimation). We analyze the biases and upper bounds on errors of the M-, DR-, X-learners as well as the T- and naive X-learners. Thanks to this analysis, we can identify the effect of the number of possible treatment levels, in addition to other parameters such as the propensity score lower bound and the outcome model estimation error. This approach is different from what has been conducted in binary treatment with the minimax approach \citep{Kunzel_2019, Curth21a, Kennedy2020} as it allows a direct analysis of the roles of the important parameters (e.g. the impact of the number of possible treatments $K$) instead of relying on the smoothness of treatment effects. Following this analysis, we present some key points about the expected performances of each meta-learner then we present a summarizing table of our findings. We note also that our analysis sheds new light on binary meta-learners' performances as it also clarifies the influence of the sampling probability for both T- and naive X-learners.

\section{Related work}
\label{sec:2}

\vspace{0.08in}

\textbf{Meta-learners for CATEs estimation.} The recent interest in the CATE's estimation has motivated the Causal Inference community to develop numerous algorithms and methods (see \citet{caron2021ITE} for a review). This includes a wide variety of statistical and ensemble methods \citep{Hill2011, Athey2016, ALaaGP2017, Wager2018, Powers2018, Hahn2020, Caron2022b} as well as neural networks \citep{Johansson16,Shalit2017,Yoon2018, Shi2019} (see \citet{Dorie2019} for a review of -hybrid- ML models for causal inference). In contrast, some methods, known as \textit{meta-learners}, are nonparametric and do not require a specific ML method. The theory of \textit{meta-learners} was initially introduced and discussed by \citet{Kunzel_2019} for the CATE estimation in the binary setting with three meta-learners: the S-learner, the T-learner (which use either a \textit{Single} or \textit{Two} models) and the X-learner. Later, \citet{Kennedy2020} proposes the DR-learner (Doubly-Robust) to overcome the problem of model misspecifications when estimating nuisance functions (e.g. the propensity score and outcome models). \citet{Wager2020} present the R-learner that estimates the CATE by minimizing an orthogonalized loss function. \citet{Curth21a} consider the PW-learner (Propensity Weighting, also known as the M-learner) and the RA-learner (Regression-Adjustment), which is an improved version of the X-learner. They show that, under some conditions, the RA-, PW-, and DR-learners can attain the oracle convergence rate.

\textbf{Multiple and continuous treatments.} Recently, there has also been an increased interest in causal inference with multi-valued and continuous treatments. The theoretical work of \citet{Imbens2000, Lechner2001, Frlich2002ProgrammeEW, ImaiVanDyk2004} extended the potential outcome framework and the propensity score to the non-binary treatment setting, including also continuous treatments. The average dose-response estimation was considered and successfully applied in many domains in medicine and economics \citep{Dominici2002, Flores2007, Kallus2017, Saini2019, Lin2019, Hu2020, Knaus2022}. Additionally, \citet{Colangelo2020double} apply doubly debiased machine learning methods to dose-response modelling with continuous treatment. The CATE's estimation, however, remains less prominently studied in the literature. \citet{Hill2011} (briefly) and \citet{Hu2020} consider Bayesian additive regression trees for the estimation of counterfactual response and causal effects. An extension of GRF to multi-valued treatments is developed by \citet{grf2022} (which can be seen as an M-learner with multi-valued treatments). \citet{Schwab2020, Harada2020, Nie2021} and \citet{Kaddour2021} applied neural networks and representations learning to estimate counterfactual response curves for multiple continuous treatments (more precisely for graph-structured treatments). \citet{Zhao2019} naively extended binary meta-learners (X- and R-learners) without any theoretical analysis of their behaviour. 

The remainder of the paper is structured as follows. Section \ref{sec:3} presents the CATE estimation for multi-treatments. In Section \ref{sec:4} we introduce CATE nonparametric estimators (meta-learners) and we discuss their consistency.  In Section \ref{sec:5}, we establish the theoretical analysis of meta-learners' error bounds and provide some discussions. We present numerical experiments and results in Section \ref{sec:6}. Finally, we present our conclusion in Section \ref{sec:7}.

\section{Problem setting}
\label{sec:3}

To address the problem of causal inference under multiple treatments, we follow the  \citeauthor{Rubin1974}-\citeauthor{Neyman1923} model as extended by \citet{Imbens2000, Lechner2001, ImaiVanDyk2004} and we consider the following statistical problem.

We suppose the existence of $Y(t)$, the real-valued counterfactual outcome that would have been observed under treatment level $t \in \mathcal{T}= \{t_0,t_1,\ldots, t_{K}\}$. We consider $(\boldsymbol{X},T,Y(t)_{t \in {\cal T}}) \sim \mathbb{P}$ where $\boldsymbol{X}=(X^{(1)},\ldots,X^{(d)}) \in  \mathcal{D} \subseteq \mathbb{R}^d$ denotes a random vector of covariates and $T$ denotes the treatment assignment random variable. We suppose finally that we observe data that has the form of an independent and identically distributed sample of $n$ units $\mathbf{D}_{\mathrm{obs}}=({D}_{{\mathrm{obs}},i})_{i=1}^n$ where ${D}_{{\mathrm{obs}},i}=\left(\boldsymbol{X}_i, T_i, Y_{\mathrm{obs},i} \right)$ is distributed as $\left(\boldsymbol{X}, T, Y_{\mathrm{obs}} \right)$ and $Y_{\mathrm{obs}}=Y(T)$ (consistency assumption). We define the Generalized Propensity Score (GPS) $r(t,\boldsymbol{x}) := \mathbb{P}(T=t|\boldsymbol{X} = \boldsymbol{x})$ \citep{Imbens2000} as the generalization of the classical Propensity Score with the same balancing propriety \citep{RR1983} to remove \textit{selection bias} in observational studies.

We aim to infer the effect of the treatment $T$ on the outcome $Y$. More precisely, we want to estimate the CATE between two levels of treatment $t_k$ and $t_0$, defined as
\begin{equation}
\label{def:CATEMult}
    \tau_{k}(\boldsymbol{x}) = \mathbb{E}[Y(t_k)-Y(t_0) | \boldsymbol{X} = \boldsymbol{x}],
\end{equation}
which can be interpreted as the expected treatment effect between levels $t_0$ (defined as the baseline treatment value) and $t_k$ given covariates $\boldsymbol{X} = \boldsymbol{x}$. Note that other definitions and alternatives of the CATE are possible \citep{Kaddour2021}.

Unfortunately, it is impossible to infer this effect directly. We observe only one potential outcome corresponding to the treatment $T$ (i.e. the real outcome) for every unit. All other potential outcomes are missing (inherently unobservable). Consequently, to identify the causal effects from the observed sample data, we shall consider the following assumptions (Assumption \ref{assump:unconfound} is unfortunately untestable).


\begin{assumption}[Unconfoundedness]
\label{assump:unconfound}
    The treatment mechanism is unconfounded given the observed covariates  $Y(t) \ \indep \ {\bf{1}} \{T=t\} \ \mid \ \boldsymbol{X}$ for all $t \in \mathcal{T}$.
\end{assumption}

\begin{assumption}[Overlap]
\label{assump:overlap}
    The probability of receiving the treatment given the observed covariates is positive i.e. there exists $r_{\mathrm{min}}>0$ such that $ r_{\mathrm{min}} \leq \mathbb{P}(T=t|\boldsymbol{X} = \boldsymbol{x})$ for all $t \in \mathcal{T}$ and $\boldsymbol{x} \in \mathcal{D}$.
\end{assumption}

With the previous assumptions, the expected potential outcome satisfies $\mathbb{E}(Y(t) \mid \boldsymbol{X} = \boldsymbol{x}) = \mathbb{E}(Y_{\mathrm{obs}} \mid T = t, \boldsymbol{X} = \boldsymbol{x})$ and the CATE can be estimated.

The problem of the CATE estimation can be seen as a non-parametric estimation problem. We tackle it by generalizing the notion of meta-learners to derive consistent estimators. This task can be achieved either by modelling the CATE directly in one or two steps: by decomposing it into regularized regression problems or by addressing a minimization problem with respect to an appropriate loss function. Moreover, all considered meta-learners below, except the R-learner, have the advantage of supporting any supervised regression method (e.g. random forest, gradient boosting methods, neural networks). 

\section{Generalizing meta-learners to multi-valued treatments}
\label{sec:4}

\vspace{0.08in}

In the following, we adopt a similar taxonomy of CATEs estimators as \citet{Curth21a} and \citet{Knaus2020}. Namely, direct plug-in (one-step) meta-learners, pseudo-outcome (two-step) meta-learners and Neyman-Orthogonality-based learners (R-learner).

\subsection{Direct plug-in meta-learners}

\vspace{0.08in}

This subsection presents direct plug-in meta-learners, also known as \textit{one-step} learners. They estimate the CATE in (\ref{def:CATEMult}) by targeting it directly from $\mathbf{D}_{\mathrm{obs}}$. They are the naive extension of the T- and S-learners in the binary case.

\paragraph{T-learner with multiple treatments.} 

T-learner is a naive approach to estimating CATEs. It consists on estimating the \textit{two} conditional response surfaces $\mu_t(\boldsymbol{x})=\mathbb{E}(Y(t) \mid \boldsymbol{X} = \boldsymbol{x})$ using $\mathbf{S}_t=\{i, T_i=t \}$ for $t \in \{t_k,t_0\}$, then estimates the CATE as  ${\widehat{\tau}}^{\mathrm{(T)}}_{k}(\boldsymbol{x}) := {\widehat{\mu}}_{t_k}(\boldsymbol{x}) - {\widehat{\mu}}_{t_0}(\boldsymbol{x})$.

The T-learning approach does not account for the interaction between $T$ and $Y$ and creates different models for different treatments. Still, it may suffer from selection bias \citep{Curth21b}, i.e. when the outcome models $\mu_t$ are estimated with respect to the wrong distribution when sampling $(D_{{\mathrm{obs}},i})_{i \in  \mathbf{S}_t}$. Therefore, $\widehat{\mu}_{t}$ should be estimated by minimizing the expected squared error on the nominal \textit{weighted} distribution using Importance Sampling \citep{Hassanpour2019}; see Appendix \ref{App:proofsA} for details.

\paragraph{S-learner with multiple treatments.} 

The S-learner in multi-valued treatments uses the identification of the CATE and considers a \textit{single} model $\mu(t, \boldsymbol{x}) =  \mathbb{E}(Y_{\mathrm{obs}} \mid T = t, \boldsymbol{X} = \boldsymbol{x})$. $\mu$ is estimated using the whole dataset $\mathbf{D}_{\mathrm{obs}}$. The CATE can be computed therefore as ${\widehat{\tau}}^{\mathrm{(S)}}_{k}(\boldsymbol{x}) := {\widehat{\mu}}(\boldsymbol{x}, t_k) - {\widehat{\mu}}(\boldsymbol{x}, t_0)$.

Including the treatment $T$ as a feature and sharing information between covariates $\boldsymbol{X}$ and $T$ may provide better predictions. 
However, this advantage is conditioned by the ability of the base learner to capture and distinguish contributions of both $\boldsymbol{X}$ and $T$ on $Y_{\mathrm{obs}}$ as we will see in Section \ref{sec:5}. Note that the S-learner may also suffer from confounding and regularization biases \citep{Chernozhukov2018, Hahn2020} when estimating $\widehat{\mu}$.

\subsection{Pseudo-outcome meta-learners}

To overcome \textit{selection bias}, a usual alternative is to consider specific representations of the observed outcome $Y_{\mathrm{obs}}$, called \textit{pseudo-outcome} $Z_k$. They incorporate \textit{nuisance components} that generally include valuable information such as the dependence between covariates $\boldsymbol{X}$ and $T$ (i.e. the GPS $r$) and the occurrence of a particular treatment assignment. Under the \textit{well-specification} of nuisance components, regressing $Z_k$ on $\boldsymbol{X}$ produces a \textit{consistent} estimator i.e. $\mathbb{E}( Z_k \mid \boldsymbol{X} = \boldsymbol{x}) = \tau_k(\boldsymbol{x})$ while keeping the same sample size as $\mathbf{D}_{\mathrm{obs}}$. In the following, we say that an estimator  ($\widehat{\mu}$ or $\widehat{r}$) is \textit{well-specified} if it is based on a well-specified statistical model, that is, the class of distributions assumed for modelling contains the unknown probability distribution from which the sample used for estimation is drawn.

\paragraph{M-learner with multiple treatments.}

The \textit{M-learner} \citep{HorvitzThompson1952}, where M refers to the \textit{modified} learned pseudo-outcome in the algorithm, is inspired from the Inverse Propensity Weighting (IPW). It is defined in the multi-valued setting, for $k=1,\ldots,K$, as the regression of $Z^{M}_{k}$ such that
\begin{equation*}
\label{eq:Z_M}
    Z^{M}_{k} = \frac{\mathbf{1} \{T=t_k\}}{\widehat{r}(t_k,\boldsymbol{X})} Y_{\mathrm{obs}} - \frac{\mathbf{1} \{T=t_0\}}{\widehat{r}(t_0,\boldsymbol{X})} Y_{\mathrm{obs}},
\end{equation*}
where $\widehat{r}$ is an estimator of the GPS $r$.

\paragraph{DR-learner with multiple treatments.}

The \textit{Doubly Robust} (DR) method \citep{RobinsDR1994,Kennedy2017,Kennedy2020} is helpful in overcoming the problem of the model's misspecification by estimating two components, the outcome model $\mu_{\cdot}$ and the GPS $r$, instead of relying on the correctness of one (and the only) parameter. If $\widehat{\mu}$ and $\widehat{r}$ denote arbitrary estimators of the outcome $\mu$ and the GPS $r$ (we assume that $\widehat{r}$ satisfies Assumption \ref{assump:overlap}), then the DR-learner regresses $Z^{DR}_{\widehat{\mu},\widehat{r},k}$ such that: 
\begin{equation*}
    \label{eq:Z_DR}
    \begin{aligned}
        Z^{DR}_{\widehat{\mu},\widehat{r},k} = &\frac{Y_{\mathrm{obs}} - \widehat{\mu}_T(\boldsymbol{X}) }{\widehat{r}(t_k,\boldsymbol{X})} \mathbf{1} \{T=t_k\} + \widehat{\mu}_{t_k}(\boldsymbol{X})  \\ &  - \frac{Y_{\mathrm{obs}} - \widehat{\mu}_T(\boldsymbol{X}) }{\widehat{r}(t_0,\boldsymbol{X})} \mathbf{1} \{T=t_0\} - \widehat{\mu}_{t_0}(\boldsymbol{X}).
    \end{aligned}
\end{equation*}

\paragraph{X-learner with multiple treatments.}

The X-learner \citep{Kunzel_2019}, also known as \textit{Regression-Adjustment} (RA)-learning in a version developed by \citet{Curth21a}, has been proposed as an alternative to T-learning in the case where one treatment group is over-represented. The idea consists of a \textit{cross} procedure of estimation between observations $Y_{\mathrm{obs}}$ and outcome models when one of the treatments occurs. For $k=1,\ldots,K$, we define the \textit{Regression-Adjustment} pseudo-outcome $Z_k^X$ as 
\begin{equation*}
\label{eq:Z_X}
    \begin{aligned}
    Z_{k}^X &= \mathbf{1} \{T=t_k\}(Y_{\mathrm{obs}}-\widehat{\mu}_{t_0}(\boldsymbol{X})) + \sum_{l \neq k}\mathbf{1} \{T=t_l\} \times \\ &(\widehat{\mu}_{t_k}(\boldsymbol{X})- Y_{\mathrm{obs}}) + \sum_{l \neq k}\mathbf{1} \{T=t_l\} (\widehat{\mu}_{t_l}(\boldsymbol{X})  - \widehat{\mu}_{t_0}(\boldsymbol{X})).
    \end{aligned}
\end{equation*}

For comparison purposes, we consider also the naive extension of the binary X-learner \citep{Kunzel_2019} to multi-treatments as proposed by \citet{Zhao2019}. This extension considers two random variables ${D}^{(k)} := Y(t_k) - \widehat{\mu}_{t_0}(\boldsymbol{X})$ and ${D}^{(0)} := \widehat{\mu}_{t_k}(\boldsymbol{X}) - Y(t_0)$ where $\widehat{\mu}_{t_k}$ and $\widehat{\mu}_{t_0}$ are trained on the samples $\mathbf{S}_{t_k}$ and $\mathbf{S}_{t_0}$. Then it regresses $({D}_i^{(k)})_{i \in \mathbf{S}_{t_k}}$ and $({D}_i^{(0)})_{i \in \mathbf{S}_{t_0}}$ on $\boldsymbol{X}$ to obtain ${\widehat{\tau}}^{(k)}$ and ${\widehat{\tau}}^{(0)}$, and estimates the CATE as:
\begin{equation*}
\label{eq:Z_nvX}
    \begin{aligned}
        {\widehat{\tau}}^{\mathrm{(X,nv)}}_{k}(\boldsymbol{x}) &:= \frac{\widehat{r}(t_k,\boldsymbol{x})}{\widehat{r}(t_k,\boldsymbol{x}) + \widehat{r}(t_0,\boldsymbol{x})} {\widehat{\tau}}^{(k)}(\boldsymbol{x}) + \\
        & \qquad \frac{\widehat{r}(t_0,\boldsymbol{x})}{\widehat{r}(t_k,\boldsymbol{x}) + \widehat{r}(t_0,\boldsymbol{x})} {\widehat{\tau}}^{(0)}(\boldsymbol{x}).
    \end{aligned}
\end{equation*}

We note that the consistency of the M- and DR-learners is already established in the literature \cite{Knaus2022}. We show it also for the extended X-learner in Appendix \ref{App:proofsA}. We note also that there are three main approaches possible to learn nuisance components ($r$ and $\mu$) and then estimate the $\tau_k$, namely, Full-Sample, Sample-Split and Cross-Fit methods \citep{Okasa2022}. This paper does not discuss estimation procedures and adopts the Full-Sample strategy.

\subsection{Neyman-Orthogonality based learner: R-learner}
\label{sec:R_learning}

The R-learner is based mainly on the \citet{Robinson1988} decomposition and was proposed by \citet{Wager2020} to provide a flexible CATE estimator avoiding regularization bias. It states that the potential outcome error  $\epsilon=Y_{\mathrm{obs}} - \mu_{T}(\boldsymbol{X})$ satisfies $\mathbb{E}(\epsilon \mid T,\boldsymbol{X})=0$ and 
\begin{equation*}
    \epsilon = Y_{{\mathrm{obs}}} - m(\boldsymbol{X}) -\sum_{k=1}^K \big(\mathbf{1}\{T = t_k\} - r(t_k,\boldsymbol{X}) \big) \tau_{k}(\boldsymbol{X}),
\end{equation*}
where $m(\boldsymbol{x}) = \mathbb{E} ( Y_{\mathrm{obs}} \mid \boldsymbol{X}=\boldsymbol{x})$ is the observed outcome model. Therefore, considering the mean squared error of $\epsilon$ (the generalized R-loss function) and minimizing it estimate $K$ models $\{\widehat{\tau}^{\mathrm{(R)}}_{k}\}_{k=1}^K$ simultaneously such that
\begin{equation*}
\label{eq:RLearnerMulti}
    \begin{aligned} \{\widehat{\tau}^{\mathrm{(R)}}_{k}\}&_{k=1}^K :=   \operatorname{argmin}_{ \{\overline{\tau}_{k}\} \in \mathcal{F}} \ \frac{1}{n} \sum_{i=1}^n  \Big[  \left( Y_{{\mathrm{obs}},i} - \widehat{m}(\boldsymbol{X}_i) \right) \\ 
    &\quad - \sum_{k=1}^K \big(\mathbf{1}\{T_i = t_k\} -\widehat{r}(t_k,\boldsymbol{X}_i) \big) \overline{\tau}_{k}(\boldsymbol{X}_i) \Big]^2,
    \end{aligned}
\end{equation*}
where $\widehat{m}$ (respectively, $\widehat{r}$) is an estimator of $m$ (respectively, $r$) and $\mathcal{F}$ is the space of candidate models $[\{\overline{\tau}_{k}\}_{k=1}^K]$.

The generalized R-learner suffers from two major limitations: First, it cannot be written as \textit{weighted} supervised learning problem with a specific pseudo-outcome. Only parametric families $\mathcal{F}$ can be considered in the multi-treatment regime. The second drawback is the non-identifiability of the generalized R-loss described previously in (\ref{eq:RLearnerMulti}) without a regularization. Indeed, this problem does not have a unique solution (see Appendix \ref{App:proofsA} for details) and leads thus to poor estimation performance. This point is also shown recently by \citet{Zhang2022} for continuous treatments and our numerical results in Appendix \ref{App:details_analytics} also confirm that the R-learner fails to estimate CATEs $(\tau_k)_{k=1}^K$.

\section{Theoretical analysis of the error upper bound}
\label{sec:5}

\vspace{0.08in}

In this section, we analyze the error's upper bounds for different meta-learners. The theoretical analysis will be carried out under the following framework:

\begin{assumption}
\label{assump:model_Y}
    We assume that $(T,\boldsymbol{X})$ satisfies the overlap assumption \ref{assump:overlap} and that, for $t \in \mathcal{T}$, the outcome $Y(t)$ is generated from a function $f:\mathbb{R} \times \mathbb{R}^d \rightarrow \mathbb{R}$ such that
    \begin{equation}
    \label{eq:model_Y}
        Y(t) = f(t,\boldsymbol{X}) + \varepsilon(t),
    \end{equation}
    where $\varepsilon(t)$ are i.i.d. Gaussian $\mathcal{N}(0, \sigma^2)$ and independent of $(T,\boldsymbol{X})$.
\end{assumption}

\begin{assumption}
\label{assump:beta_Y}
    We assume the existence of $\boldsymbol{\beta}_t \in \mathbb{R}^p$ such that, for all $t \in \mathcal{T}$ and $\boldsymbol{x} \in \mathcal{D}$
    \begin{equation*}
        f(t,\boldsymbol{x}) = \sum_{j=0}^{p-1} \beta_{t,j} f_j(\boldsymbol{x}) = \boldsymbol{f}(\boldsymbol{x})^{\top} \boldsymbol{\beta}_t,
    \end{equation*}
    where $f_j$ are some bounded predefined basis functions. 
\end{assumption}

The assumption of a product effect is reasonable. One can show the universality of this representation in the Reproducing Kernel Hilbert Space (RKHS) (Proposition 1 of \citet{Kaddour2021}) if we allow the dimension $p$ to be large enough.

Under these two assumptions, the CATE $\tau_k$ can be written as:
\begin{equation*}
    \tau_k(\boldsymbol{x}) = \sum_{j=0}^{p-1} \beta^*_{k,j} f_j(\boldsymbol{x}) =  \boldsymbol{f}(\boldsymbol{x})^{\top} \boldsymbol{\beta}^*_k,
\end{equation*}
where $\boldsymbol{\beta}_k^*=(\beta^*_{k,j})_{j=0}^{p-1} =\boldsymbol{\beta}_{t_k}-\boldsymbol{\beta}_{t_0} \in \mathbb{R}^p$.

From a theoretical point of view, the S-learner corresponds to the naive Ordinary Least Square (OLS) $\widehat{\boldsymbol{\beta}}_k$ of $\boldsymbol{\beta}^*_k$. The statistical task of the CATE's estimation holds immediately. However, we cannot properly analyse the base-learner's ability to learn $\boldsymbol{\beta}^*_k$ under confounding effects.

\begin{theorem}
\label{th:T_learn_err}
    Under Assumptions (\ref{assump:model_Y}-\ref{assump:beta_Y}), the OLS estimators $\widehat{\boldsymbol{\beta}}^*_k$ of the T-learner and the naive X-learner are unbiased and have an asymptotic covariance matrix $\mathbb{V}\big(\widehat{\boldsymbol{\beta}}^*_k\big)= \mathbf{C}/n $, whose terms $\mathbf{C}_{i j}$ are bounded by:
    \begin{gather*}
        \mathcal{E}^{T} = \mathcal{E}^{X,nv} = \mathcal{O}\left( \frac{1}{\rho(t_k)} + \frac{1}{\rho(t_0)} \right),
    \end{gather*}
    where $\mathbb{P}(T=t) =\rho(t)>0$ for all $t \in \mathcal{T}$.
\end{theorem}

We consider now pseudo-outcome meta-learners (M-, DR- and X-learners). When investigating the pseudo-outcomes $Z_k$, one can see that, for $k=1,\ldots, K$, these pseudo-outcomes are linear with respect to $Y_\mathrm{obs}$ i.e.
\begin{equation*}
    Z_{k} = A_{t_k}(T,\boldsymbol{X}) Y_{\mathrm{obs}} + B_{t_k}(T,\boldsymbol{X}),
\end{equation*}
where $A_{t_k}(T,\boldsymbol{X})$ and $B_{t_k}(T,\boldsymbol{X})$ are given for each pseudo-outcome meta-learner.

\begin{theorem}
\label{th:meta_learner_err}
    Under Assumptions (\ref{assump:model_Y}-\ref{assump:beta_Y}), the OLS estimator $\widehat{\boldsymbol{\beta}}^*_k$ is unbiased if the nuisance parameters ($\widehat{\mu}$ and $\widehat{r}$) are well-specified, and has an asymptotic covariance matrix $\mathbb{V}\big(\widehat{\boldsymbol{\beta}}^*_k\big)= \mathbf{C}/n$, whose terms, for all $\epsilon>0$, are bounded by:
    \begin{gather*}
        \mathcal{E}^{M} = \mathcal{O}\left(\frac{1}{r_{\mathrm{min}}^{1+\epsilon}}\right) \text{ for the M-learner,} \\[.10cm]
        \mathcal{E}^{DR} = \mathcal{O}\left(\frac{\mathrm{err}(\widehat{\mu}_{t_k}) + \mathrm{err}(\widehat{\mu}_{t_0})}{r_{\mathrm{min}}^{1+\epsilon}}\right) \text{ for the DR-learner,} \\[.10cm]
        \mathcal{E}^{X} = \mathcal{O}\Big(K^2 \sum_{l\neq k} \mathrm{err}(\widehat{\mu}_{t_l})\Big) \text{ for the X-learner,} 
    \end{gather*}
    where $\mathrm{err}(\widehat{\mu}_t) = \mathbb{E}_{\boldsymbol{X}}\big[ \big( f(t,\boldsymbol{X}) - \widehat{\mu}_t(\boldsymbol{X}) \big)^2 \big]$ is the expected mean squared error of $\widehat{\mu}_t$.
\end{theorem}

\paragraph{Sketch of the proofs of Theorems \ref{th:T_learn_err} and \ref{th:meta_learner_err}.} Both proofs are similar and are structured in 3 steps: 1) Express the OLS estimator $\widehat{\boldsymbol{\beta}}^*_k$ as a function of the true $\boldsymbol{\beta}^*_k$; 2) Apply the multivariate Central Limit Theorem and the Slutsky theorem (or the Delta method in the general case with a biased $\widehat{\boldsymbol{\beta}}^*_k$); 3) Bound the asymptotic covariance matrix terms. The full proofs can be found in Appendix \ref{App:Error_estim}.

\begin{remark}
    For $K=2$, one recovers the error's upper bounds of pseudo-outcome meta-learners and the important parameters as shown by \citet{Curth21a}. However, the influence of the sampling probability $\rho$ for the T- and naive X-learners is original.
\end{remark}

Finally, regarding the binarized \citep{Kaddour2021} or generalized R-learners, they cannot be expressed as a \textit{model-agnostic} regression problem but rather as a minimization problem of a loss function (generalized R-loss or binarized R-loss). Therefore, the actual theoretical analysis of the error's upper bound cannot be conducted similarly.

Using Theorems \ref{th:T_learn_err} and \ref{th:meta_learner_err}, we establish the following discussions (see Table \ref{Tab:Comparison_Meta-Learners} for a summary).

\subsection{General comments and insights}

\vspace{0.08in}

\textbf{About Theorem \ref{th:T_learn_err}.} The significant implication of Theorem \ref{th:T_learn_err} is the ability to anticipate the performances of the T-learner with respect to the treatment distribution when the probability $\rho(t)$ of sampling the treatment value $t$ is small. 
The performances of the T- and naive X-learners become poor when the number $K$ of treatments increases because these learners imply learning in small samples $\mathbf{S}_{t}$. Moreover, the naive X-learner also does not seem to really differ from the T-learner in any meaningful way (see Appendix \ref{App:T_Xnv_learners}).
We note that this result is original and cannot be obtained by the minimax approach.

\textbf{About Theorem \ref{th:meta_learner_err}.} From this theorem, one can establish the relationship between the number of observations $n$ and the number of treatment values $K$ for a given error bound. The same theorem is also useful for an RCT setting ($r_{\mathrm{min}}$ is known) and when the error of the outcome model is small for all treatment levels except for one treatment value $t$.

Moreover, although pseudo-outcome meta-learners have the advantage of learning on the full sample, they, unfortunately, may lead to high error and poor performance in how nuisance components intervene. On the one hand, the GPS is in the M- and DR-learners denominators. The error bound is likely to be high when there is a lack of overlap or when the number of treatments $K$ increases (we have necessarily $r_{\mathrm{min}} \leq 1/K$ ). On the other hand, the upper bounds of the X-and DR-learners depend on the quality of the estimated $\widehat{\mu}$. One can expect that the more precise the outcome models are, the lower the error is.

\subsection{Specific comments for each meta-learner}

\vspace{0.08in}

\textbf{M-learner.} Without surprise, the M-learner is very sensitive to the estimated GPS $\widehat{r}$ and suffers from high variance. This is even more critical as the number $K$ of treatments increases. 

\textbf{DR-learner.} The error's estimation term $\mathrm{err}(\widehat{\mu}_t)$ of $\widehat{\mu}_{t}$ and $\widehat{\mu}_{t_0}$ in the numerator dramatically improves the DR-learner's performances. 

\textbf{X-learner.} The X-learner incorporates only $\widehat{\mu}_{.}$ and does not imply the GPS $r$. Thus, the X-learner is likely to have the smallest error compared to other meta-learners when the overlap assumption is not sufficiently respected. However, the consistency of $(\widehat{\mu}_t)_{t \neq t_k}$ is required to estimate CATEs correctly.

\textbf{M-learner vs DR-learner.} If the outcome models $\widehat{\mu}_{t}$ and $\widehat{\mu}_{t_0}$ are well-specified, the error's upper bound is expected to be smaller for the DR-learner than for the M-learner. However, if they are misspecified (but the propensity score is well-specified), then there is no guarantee that the DR-learner would perform better than M-learner. It may perform even worse, as we will see in Appendix \ref{App:details_analytics} (Table \ref{tab:Table5}). 

\textbf{M-learner vs X-learner.} The X-learner is likely to have a lower error upper bound if the expected squared error $\mathrm{err}(\widehat{\mu}_t)$ is small and if some conditions on $K$ and $r_{\mathrm{min}}$ hold. 

\textbf{X-learner vs DR-learner.} Analytically, it is difficult to anticipate which meta-learner would perform better. This depends mainly on $\mathrm{err}(\widehat{\mu}_{.})$, $K$ and $r_{\mathrm{min}}$, whom, in some cases, make the X-learner have less error than the DR-learner and the opposite in other cases. Still, our numerical results in Appendix \ref{App:details_analytics} show that the X-learner outperforms the DR-learner in most cases. 

\textbf{X-learner vs Naive X-learner.} We cannot theoretically compare these meta-learners without knowledge about the distribution of $T$. However, we can see numerically that the X-learner clearly outperforms the naive X-learner. 

We conclude this subsection with a small discussion about the generalized and the binarized R-learners (see Appendix \ref{App:bin_R_learner}). Indeed, the binary R-loss function may be solved separately in a low sample (instead of $\mathbf{D}_{\mathrm{obs}})$ and one can obtain a unique solution, unlike the generalized R-loss. However, the optimization procedure is two-stage iterative and computationally heavy. We do not consider it in Section \ref{sec:6}.

\begin{table}[h!]
\caption{Summary table of multi-treatments meta-learners.}
\centering
\label{Tab:Comparison_Meta-Learners}
\begin{tabular}{ c  c  c  }
    \toprule
        Meta-learner & Advantages  & Disadvantages \\
    \midrule
        T-learner &  Simple approach  &   Selection bias \\
        (nv X-learner)        &      &  Low samples \vspace{0.08in}\\
        S-learner &  Simple approach  &  Confounding effects \\
                &      &  Regularization bias \\ 
    \midrule
        M-learner &  Consistency &  High variance \vspace{0.08in} \\
        DR-learner &  Consistency &  Possibly high   \\
                &   Doubly Robust  &  variance \vspace{0.08in} \\          
        X-learner  &   Consistency &  Non-intuitive \\
                &  Low variance & \\
    \midrule
        R-learner &  Interaction  &  Non-identifiability  \\
         &  effects  &  \vspace{0.08in} \\
        Bin R-learner &  Identifiability  &  Computational cost   \\
    \bottomrule
\end{tabular}
\end{table}

In Table \ref{Tab:Comparison_Meta-Learners}, \textit{Bin R-learner} refers to the binarized R-learner \cite{Kaddour2021} and \textit{nv X-learner} refers to the naive extension of the X-learner described in subsection \ref{eq:Z_nvX}. \textit{Possibly high variance} refers to the case where the variance can be significantly high due to the lack of overlap caused by the inverse propensity weighting in some DGPs. \textit{Selection bias} refers to the bias that occurs when sampling $\mathbf{S}_t$ and comparing units directly, as we describe in Proposition 4.1 in Appendix \ref{App:proofsA}. The confounding effects represent the statistical and intrinsic dependence between the treatment $T$ and the covariates $\mathbf{X}$, which prevent some base-learner (e.g. random forest) from distinguishing and disassociating the treatment $T$ and the covariates $\mathbf{X}$. Finally, \textit{low samples} refer to cases where the samples $\mathbf{S}_t = \{ i, T_i=t \}$ become small for some treatment levels and the model has then few observations to learn from.

\subsection{Practical recommendations for selecting a meta-learner}

\vspace{0.08in}

In this subsection, based on previous insights and our numerical findings (see Section \ref{sec:6} for more detail), we provide some instructions and recommendations for selecting meta-learners given a dataset $\mathbf{D}_{\mathrm{obs}}$:

\begin{itemize}
    \item Pseudo-outcomes meta-learners and the S-learner are preferred in low sample regimes.
    \item It is recommended to examine the distribution of $T$ to understand at which values of $t$ the T-learner may fail to learn CATEs.
    \item The S-learner remains a simple and reasonable choice, especially when $K \geq 10$.
    \item The X-learner is useful to learn the CATE $\tau_t$ when there is not enough information about a specific treatment value $t$.
    \item If one is interested in quantifying uncertainties, then the DR-learner is recommended as it allows having a consistent estimation of the CATE.
    \item After estimating the GPS $r$, one needs to check if there is a possible lack of overlap. If a lack of overlap is detected, one should avoid considering the M- and DR-learners.
\end{itemize}

\section{Numerical experiments}
\label{sec:6}

\vspace{0.08in}

In this section, we assess the performances of the different meta-learners, additional numerical results are shown in Appendix \ref{App:details_analytics}.

In synthetic or semi-synthetic examples where the CATEs are known, the error in estimation is given by \textbf{mPEHE} (respectively, \textbf{sdPEHE}), the mean (respectively, the standard deviation) of the Precision in Estimation of Heterogeneous Effect (PEHE) \citep{Hill2011,Shalit2017} defined as the mean squared error in the estimation of the treatment effect $\widehat{\tau}_{k}$, over all possible treatment levels $t_k$ for $k=1,\ldots,K$: 
\begin{equation*}
    \textbf{mPEHE} = \sqrt{\frac{1}{K} \sum_{k=1}^K PEHE(\widehat{\tau}_{k})^2},
\end{equation*}
where $PEHE(\widehat{\tau}_{k})^2 = \frac{1}{n} \sum_{i=1}^n \left( \widehat{\tau}_{k}(\boldsymbol{X}_i) - \tau_k(\boldsymbol{X}_i) \right)^2$, and,
\begin{equation*}
    \textbf{sdPEHE} = \sqrt[4]{  \frac{1}{K-1} \sum_{k=1}^K \Big( PEHE(\widehat{\tau}_{k})^2 - \textbf{mPEHE}^2 \Big)^2} 
\end{equation*}

Those metrics will be used to compare meta-learners under different scenarios (sample size $n$, number of possible treatments $K$, the correctness of nuisance parameters). We do not consider here model-fitting of base-learners. All hyperparameters (e.g. the number of trees, depth etc.) are fixed to their default values during all experiments. In addition, we do not consider Neural Networks because it would require choosing between at least five possible architectures \cite{Curth21a} to define tasks of learning nuisance components and estimating CATEs $\tau_k$ while the main focus of the paper is on the choice of the meta-learner.

\subsection{Synthetic datasets: analytical functions in randomized and non-randomized studies}
\label{sec:6.1}

\vspace{0.08in}

In this subsection, we begin by empirically evaluating the performances of meta-learners when the treatment $T$ is taking $K+1=10$ possible values in $[0,1]$ on a linear outcome:
\begin{gather*}
\label{eq:lin_model}
    Y(t) \text{ of the form (\ref{eq:model_Y}) with } \\ f(t,X)= (1+t) X, \ \text{ and } X \sim \mathcal{U}[0,1],
\end{gather*}
in Randomized Controlled Trials (RCT) setting where the $T$ and $X$ are independent. Second, we evaluate meta-learners on the hazard rate outcome:
\begin{gather*}
\label{eq:hazard_model}
    Y(t) \text{ of the form (\ref{eq:model_Y}) with } \\ f(t,\boldsymbol{X}) = t+\|\boldsymbol{X}\| \exp{(-t \|\boldsymbol{X}\|)} \text{ and } \boldsymbol{X} \sim \mathcal{N}(\boldsymbol{0},\mathbf{I}_5)
\end{gather*}
in a non-randomized setting as will be described below. 

\begin{table*}[h!] 
    \begin{center}
    \caption{\textbf{mPEHE} and \textbf{sdPEHE} for XGBoost and RandomForest; linear model in a RCT setting with $n=2000$ units.}
    \vskip 0.10in
    \centering
    \label{tab:Table_lin_RCT}
    \begin{tabular}{ c  c  c }
      \toprule
      Meta-learner & XGBoost & RandomForest \\
      \midrule
      T-Learner & 0.065 (0.019) & 0.041 (0.016) \\
      S-Learner & \textbf{0.033 (0.018) } & 0.032 (0.028) \\
      \textit{Nv}X-Learner & 0.060 (0.019) & \textbf{0.037 (0.016)} \\
      \midrule
      M-Learner & 1.25 (0.610) & 1.22 (0.621) \\
      DR-Learner & 0.068 (0.019)  -- 0.063 (0.020) & 0.068 (0.018) -- 0.068 (0.018) \\
      X-Learner & 0.063 (0.020)  -- \textbf{0.033 (0.017)} & 0.045 (0.016) -- 0.061 (0.040)  \\
      \midrule
      RLin-Learner & 0.135 (0.130) & 0.137 (0.128) \\
      \bottomrule
    \end{tabular}
    \end{center}
    \begin{tablenotes}[flushleft]
	    \scriptsize
	    \item For the DR- and  X-learners: $\mu_t$ are estimated by T- (left value) or S- (right value). The bold font indicates the best meta-learner (row) per base-learner (column).
    \end{tablenotes}
\end{table*}

\begin{table*}[h!]
    \begin{center}
    \caption{\textbf{mPEHE} and \textbf{sdPEHE} for XGBoost and RandomForest. Hazard rate model in an observational setting with $n=10000$ units.}
    \vskip 0.10in
    \centering
    \label{tab:Table_HR_nonRand} 
    \begin{tabular}{ c  c  c }
      \toprule
      Meta-learner & XGboost & RandomForest  \\
      \midrule
      T-Learner & 0.183 (0.039) & 0.286 (0.155)  \\
      \textit{Reg}T-Learner & 0.176 (0.044) & 0.286 (0.155)  \\
      S-Learner & 0.176 (0.056) & 0.306 (0.153) \\
      \textit{Nv}X-Learner & 0.190 (0.096) & 0.336 (0.200)  \\
      \midrule
      M-Learner & 1.61 (0.505) & 1.58 (0.472) \\
      DR-Learner & 0.168 (0.045) - 0.178 (0.048) &  0.304 (0.158) -- 0.322 (0.162)  \\
      X-Learner & \textbf{0.167 (0.053)} -- 0.172 (0.057)  &  0.302 (0.169) -- 0.332 (0.167) \\
      \midrule
      RLin-Learner & 0.231 (0.081) & \textbf{0.186 (0.123)} \\
      \bottomrule
    \end{tabular}
    \end{center}
    \begin{tablenotes}[flushleft]
	    \scriptsize
	    \item For the DR- and  X-learners: $\mu_t$ are estimated by T- (left value) or S- (right value). The bold font indicates the best meta-learner (row) per base-learner (column).
    \end{tablenotes}
\end{table*}

To simulate observational data, instead of removing some rows, we create a selection bias in the data by selecting preferentially only observations with specific characteristics (see subsection \ref{App:D_1} in Appendix \ref{App:details_analytics}). This strategy comes in line with the findings and recommendations of \citet{Curth2021c} about creating a biased sub-sample and evaluating CATEs' estimators.

The GPS is estimated using the XGBoost model, and the outcome models $\mu_t$ are either estimated by the T- or S-learning approaches. In Tables 2 and 3 and Appendix \ref{App:details_analytics}, RLin-learner denotes the generalized R-learner with linear regression models in (\ref{eq:RLearnerMulti}) with $p=2$. For each meta-learner (row) and base-learner (column), we indicate the \textbf{mPEHE} followed by the \textbf{sdPEHE} between brackets.

In Tables \ref{tab:Table_lin_RCT} and \ref{tab:Table_HR_nonRand}, we find that, as expected, the M-learner predicts poorly. The T- and naive X-learners give better predictions for Random Forest, whereas the S-learner gives better results for XGBoost. Regularizing T-learner (\textit{Reg}T-Learner) against selection bias increases its performance. The X- and DR-learners improve the predictions of the S-learner for XGBoost, but this improvement is not always observable for Random Forests. Unfortunately, the actual results (and also additional numerical experiments in Appendix \ref{App:details_analytics}) confirm the claim: The RLin-learner generally fails to identify CATEs correctly. 

Despite these satisfying results, we highlight the problem of over-fitted gradient boosting models and Random Forest by comparing them with the linear model in Appendix \ref{App:details_analytics}. This problem should be taken further while estimating CATEs. We think that using out-sample prediction supervised models might solve this problem.

We consider now the effect of increasing $K$ on the hazard rate function with XGBoost. The results are shown in Figure \ref{fig:K_asymp_XGB} in Appendix \ref{App:D_3}. On the one hand, the performances of the T- and the naive X-learners become compromised. The regularized T-learner suffers also from the same issue (with a linear effect with respect to $K$), which can be also quantified on the DR- and X-learners when applying the regularized T-learning. For $K \geq 20$, the T- and the naive X-learners perform better than the previous meta-learners. Regarding the M-learner, it has poor performances in all cases as can be expected with Theorem \ref{th:meta_learner_err}. On the other hand, the S-learner stabilises once $K$ is large enough and, consequently, the DR- and X-learners also stabilize when applying the S-learning for large $K$.  Therefore, we recommend the S-learner's estimated potential outcome model when $K \geq 10$ for pseudo-outcome meta-learners. To conclude, two-step meta-learners are robust. In particular, the X-learner improves the quality of plug-in meta-learners; when it does not, the differences are very small.

\begin{table*}[h!]
    \begin{center}
    \caption{\textbf{mPEHE} and \textbf{sdPEHE} for XGBoost and RandomForest. Heat Extraction model in an observational setting.}
    \vskip 0.10in
    \centering
    \begin{tabular}{ c  c  c  c }
        \toprule
            Meta-learner & XGBoost & RandomForest & \\
        \midrule
            T-learner & 0.172 (0.052) & 0.157 (0.067)\\
            \textit{Reg}T-Learner & 0.156 (0.042) & 0.154 (0.067) \\
            S-learner & 0.101 (0.040) & 0.218 (0.129) \\
            \textit{Nv}X-Learner & 0.102 (0.042) & \textbf{0.143 (0.067)}  \\
        \midrule
            M-learner & 1.04 (0.423) & 0.898 (0.417)  \\
            DR-learner & 0.148 (0.042) -- 0.097 (0.029) & 0.164 (0.068) -- 0.203 (0.108) \\
            X-learner & 0.142 (0.041) -- \textbf{0.094 (0.034)} & 0.173 (0.077) --  0.211 (0.120) \\
        \midrule
            RLin-learner & 0.357 (0.274) & 0.362 (0.278)  \\
        \bottomrule
    \end{tabular}
    \label{Tab:mPEHE_heat_CATEs}
    \end{center}
    \begin{tablenotes}[flushleft]
	    \scriptsize
	    \item For the DR- and  X-learners: $\mu_t$ are estimated by T- (left value) or S- (right value). The bold font indicates the best meta-learner (row) per base-learner (column).
    \end{tablenotes}
\end{table*}

\subsection{Semi-synthetic dataset: estimating heterogeneous treatment effects on a non-randomized dataset.}
\label{sec:6.2}

\vspace{0.08in}

In this subsection, we consider a multistage fracturing Enhanced Geothermal System (EGS) \citep{Olasolo2016}. We assume that the heat extraction performance satisfies the physical model: $Q_{well}(\ell_L) = Q_{fracture} \times {\ell_L}/{d} \times \eta_{d}$,
where $Q_{fracture}$ is the \textit{unknown} heat extraction performance from a single fracture, that can be generated using a numerical model with eight input parameters including reservoir characteristics and fracture design. $\ell_L$ is the Lateral Length of the well, $d$ is the average spacing between two fractures and $\eta_{d}$ is the stage efficiency penalizing the individual contribution when fractures are close to each other. We refer to Appendix \ref{App:EGS_data} for a detailed description of the model and the semi-synthetic dataset.

We consider the Lateral Length as treatment $T$ with $K+1=13$ possible values and the covariates $\boldsymbol{X} \in \mathbb{R}^{11}$ are the remaining variables. We also consider a logarithmic transformation of the heat performance for a meaningful \textbf{mPEHE}, and we normalize the treatment $T$. Following the \textit{preferential selection}, we sample $n=10000$ units such that wells with high lateral length are likely to have larger fractures and vice versa. The GPS is estimated using gradient boosting models. Table \ref{Tab:mPEHE_heat_CATEs} resumes the \textbf{mPEHE} and \textbf{sdPEHE} in brackets for different meta-learners. Most findings of subsection \ref{sec:6.1} remain valid: XGBoost model is generally a better choice than Random Forests (except for T-learning); the X-learner, followed by the DR-learner, outperforms all other learners.


\section{Conclusion}
\label{sec:7}

\vspace{0.08in}

We have investigated heterogeneous treatment effects estimation with multi-valued treatments. In addition to standard plug-in meta-learners, we have considered representations to build pseudo-outcome meta-learners, and we have proposed the generalized Robinson decomposition to build the R-learner. Using the bias-variance analysis, we have conducted an in-depth analysis of the error's upper bounds of pseudo-outcomes meta-learners. Thanks to this analysis, we could discuss the advantages and limitations of each pseudo-outcome meta-learner. In particular, we have identified the impacts of the number of treatment levels and the lower bound $r_\mathrm{min}$ on the M-, DR and X-learners. Through synthetic and semi-synthetic industrial datasets, we have illustrated the performances of different meta-learners in a non-randomized case where some covariates are confounded with the treatment. We have demonstrated the ability of the X-learner to reconstruct the ground truth model. We have also highlighted how the choice of base-learner can affect the quality of CATEs estimation.

\section{Software and Data}
\label{sec:8}
\vspace{0.08in}

The code and the semi-synthetic dataset in subsection \ref{sec:6.2} are available at \url{https://github.com/nacharki/multipleT-MetaLearners}.

\section{Acknowledgements}
\label{sec:9}
\vspace{0.08in}

The authors thank Marianne Clausel, Alessandro Leite, Audrey Poinsot, Georges Oppenheim and the anonymous reviewers for helpful feedback and discussions. This work was supported by TotalEnergies and the French National Agency for Research and Technology (ANRT) (Grant n° 2019/0714).

\vspace{0.08in}

\bibliography{example_paper}
\bibliographystyle{icml2023}

\newpage
\appendix
\onecolumn

The Appendix of the paper is divided into the following sections:
\begin{itemize}
    \item Appendix \ref{App:proofsA} contains the proofs of the main propositions in Sections 3: Regularization of the T-learner, the consistency of the naive and extended X-learners, and the generalization of the Robinson decomposition. It also includes a solution for the generalized R-learner with linear regression models.
    \item Appendix \ref{App:Error_estim} is divided into two parts: The first part \ref{App:pseudo_outcome} establishes the bias-variance analysis of pseudo-outcome meta-learners (proof of Theorem \ref{th:meta_learner_err}). It provides the main framework for the proof based on the Delta Method and Slutsky's theorem, which will then be applied to the M-, DR-, and X-learners to show their error's upper bound. The second part \ref{App:T_Xnv_learners} establishes the bias-variance of the T- and the naive X-learners (proof of Theorem \ref{th:T_learn_err}). This appendix follows the same logic as Appendix \ref{App:pseudo_outcome}'s main proof.
    \item Appendix \ref{App:bin_R_learner} discusses further the comparison of the generalized R-learner and the binarized R-learner. 
    \item Appendix \ref{App:details_analytics} concerns Section \ref{sec:6.1} with more insights about the numerical experiments and results.
    \item Appendix \ref{App:EGS_data} describes the Semi-synthetic dataset used in subsection \ref{sec:6.2}
\end{itemize}

\vspace{0.3in}

\section{Proofs of propositions of Section \ref{sec:4}}
\label{App:proofsA}

\begin{proposition}[Regularizing the T-learner against selection bias]
\label{prop:1}
    For a treatment level $t \in \mathcal{T}$, the expected squared error of the estimator $\widehat{\mu}_{t}$ on the outcome surface $\mu_{t}$ satisfies:
    \begin{equation}
        \begin{aligned}
        \mathbb{E}&_{\boldsymbol{X} \sim \mathbb{P}(\cdot)} \big[  (\widehat{\mu}_{t}(\boldsymbol{X}) - \mu_{t}(\boldsymbol{X}))^2 \big] = \\ &\mathbb{E}_{\boldsymbol{X} \sim \mathbb{P}(\cdot \mid T=t)} \left[\frac{\mathbb{P}(T=t) }{r(t, \boldsymbol{X})} \Big(\widehat{\mu}_{t}(\boldsymbol{X}) - \mu_{t}(\boldsymbol{X})\Big)^2 \right].
        \end{aligned}
    \end{equation}
    where $\mathbb{P}(\cdot)$ is the marginal distribution of $\boldsymbol{X}$ and $ \mathbb{P}(\cdot \mid T=t)$ is the conditional distribution of $\boldsymbol{X}$ given $T=t$.
\end{proposition}

\subsection{Proof of Proposition \ref{prop:1} }

This proof is similar to the proof of equation (5) in supplementary of \citet{Curth21a}. For simplicity, we assume that the distribution of X and the conditional distribution of $\boldsymbol{X}$ given $T=t$ are absolutely continuous with respect to the Lebesgue measure over $\mathbb{R}^d$. Let $p_{\boldsymbol{X}}(\boldsymbol{x})$ denote the probability distribution function of $\boldsymbol{X}$, let $p(\boldsymbol{x} \mid T=t)$ denote the probability distribution function of $\boldsymbol{X}$ given $T=t$ and let $R_t= \int{ (\widehat{\mu}_t(\boldsymbol{x}) - \mu_t(\boldsymbol{x}))^2 p( \boldsymbol{x} \mid T=t)d\boldsymbol{x} }$.

    \begin{equation}
        \begin{aligned}
            &\mathbb{E}_{\boldsymbol{X} \sim \mathbb{P}(\cdot)} \big[  (\widehat{\mu}_t(\boldsymbol{X}) - \mu_t(\boldsymbol{X}))^2 \big] = \int{ (\widehat{\mu}_t(\boldsymbol{x}) - \mu_t(\boldsymbol{x}))^2 p(\boldsymbol{x})d\boldsymbol{x} } \\
            &= \mathbb{P}(T=t)\int{ (\widehat{\mu}_t(\boldsymbol{x}) - \mu_t(\boldsymbol{x}))^2 p( \boldsymbol{x} \mid T=t)d\boldsymbol{x} }  + \sum_{t'\neq t} \mathbb{P}(T=t')\int{ (\widehat{\mu}_t(\boldsymbol{x}) - \mu_t(\boldsymbol{x}))^2 p( \boldsymbol{x} \mid T=t')d\boldsymbol{x} } \\
            &= \mathbb{P}(T=t) R_t  + \sum_{t'\neq t} \mathbb{P}(T=t')\int{ (\widehat{\mu}_t(\boldsymbol{x}) - \mu_t(\boldsymbol{x}))^2 \frac{p( \boldsymbol{x} \mid T=t')}{p( \boldsymbol{x} \mid T=t)}p( \boldsymbol{x} \mid T=t)d\boldsymbol{x} } \\
            &= \mathbb{P}(T=t) R_t  + \sum_{t'\neq t} \mathbb{P}(T=t')\int{ (\widehat{\mu}_t(\boldsymbol{x}) - \mu_t(\boldsymbol{x}))^2 \frac{ \frac{\mathbb{P}(T=t' \mid \boldsymbol{x}) p(\boldsymbol{x})}{\mathbb{P}(T=t')}}{\frac{\mathbb{P}(T=t \mid \boldsymbol{x}) p(\boldsymbol{x})}{\mathbb{P}(T=t)}}p( \boldsymbol{x} \mid T=t)d\boldsymbol{x} } \qquad \text{(Bayes rule)}\\
            &= \mathbb{P}(T=t) R_t  + \mathbb{P}(T=t) \sum_{t'\neq t} \int{ (\widehat{\mu}_t(\boldsymbol{x}) - \mu_t(\boldsymbol{x}))^2 \frac{ \mathbb{P}(T=t' \mid \boldsymbol{x})}{\mathbb{P}(T=t \mid \boldsymbol{x})}p( \boldsymbol{x} \mid T=t)d\boldsymbol{x} } \\
            &= \mathbb{P}(T=t) R_t  + \mathbb{P}(T=t) \int{ (\widehat{\mu}_t(\boldsymbol{x}) - \mu_t(\boldsymbol{x}))^2 \frac{\sum_{t'\neq t} \mathbb{P}(T=t' \mid \boldsymbol{x}) }{\mathbb{P}(T=t \mid \boldsymbol{x})}p( \boldsymbol{x} \mid T=t)d\boldsymbol{x} }
            \\
            &= \mathbb{P}(T=t) R_t  + \mathbb{P}(T=t) \int{ \frac{1-r(t, \boldsymbol{x}) }{r(t, \boldsymbol{x})} (\widehat{\mu}_t(\boldsymbol{x}) - \mu_t(\boldsymbol{x}))^2 p( \boldsymbol{x} \mid T=t)d\boldsymbol{x} } \\
            &= \mathbb{P}(T=t) \int{\left(  1 + \frac{1-r(t, \boldsymbol{x}) }{r(t, \boldsymbol{x})}\right) (\widehat{\mu}_t(\boldsymbol{x}) - \mu_t(\boldsymbol{x}))^2 p( \boldsymbol{x} \mid T=t)d\boldsymbol{x} } \\
            &= \mathbb{E}_{\boldsymbol{X} \sim p_{( \cdot \mid T=t)}} \left[\frac{\mathbb{P}(T=t) }{r(t, \boldsymbol{X})} (\widehat{\mu}_t(\boldsymbol{X}) - \mu_t(\boldsymbol{X}))^2 \right].
        \end{aligned}
    \end{equation}
    
\subsection{Consistency of the X-learner}

    By direct calculations, we show that
    \begin{align}
        \mathbb{E}(Z_k^X \mid \boldsymbol{X} = \boldsymbol{x}) &= \mathbb{E} \left[ \mathbf{1} \{T=t_k\} Y(t_k)\mid \boldsymbol{X} = \boldsymbol{x} \right] - r(t_k,\boldsymbol{x}) \mu_{t_0}(\boldsymbol{x})  + \sum_{l \neq k} r(t_l,\boldsymbol{x}) \Big( \mu_{t_k}(\boldsymbol{x})  \\& \qquad\qquad\qquad
        - \mathbb{E} \left[ \mathbf{1} \{T=t_l\} Y(t_l)\mid \boldsymbol{X} = \boldsymbol{x} \right] \Big)  + \sum_{l \neq k} r(t_l,\boldsymbol{x}) (\mu_{t_l}(\boldsymbol{x})  - \mu_{t_0}(\boldsymbol{x})) \\
        &= r(t_k,\boldsymbol{x}) \mu_{t_k}(\boldsymbol{x}) - r(t,\boldsymbol{x}) \mu_{t_0}(\boldsymbol{x})  + \sum_{l \neq k} \big( r(t_l,\boldsymbol{x})\mu_{t}(\boldsymbol{x}) - r(t_l,\boldsymbol{x})\mu_{t'}(\boldsymbol{x}) \big) \\& \qquad\qquad\qquad+ \sum_{l \neq k} r(t_l,\boldsymbol{x})(\mu_{t_l}(\boldsymbol{x})  - \mu_{t_0}(\boldsymbol{x})) \qquad\qquad \text{(by Assumption \ref{assump:unconfound})} \\
        &=  r(t_k,\boldsymbol{x}) \mu_{t_k}(\boldsymbol{x}) - r(t_k,\boldsymbol{x}) \mu_{t_0}(\boldsymbol{x}) + \sum_{l \neq k} r(t_l,\boldsymbol{x})\mu_{t_k}(\boldsymbol{x}) - \sum_{l \neq k} r(t_l,\boldsymbol{x}) \mu_{t_0}(\boldsymbol{x}) \\
        &= (\mu_{t_k}(\boldsymbol{x}) - \mu_{t_0}(\boldsymbol{x})) \Big(r(t_k,\boldsymbol{x}) + \sum_{l \neq k} r(t_l,\boldsymbol{x})\Big) \\ 
        &= \mu_{t_k}(\boldsymbol{x}) - \mu_{t_0}(\boldsymbol{x}) = \tau_{k}(\boldsymbol{x}).
    \end{align}

\subsection{Consistency of the naive extension of the X-learner}

    Let us consider the two random variables ${D}^{(k)} := Y(t_k) - {\mu}_{t_0}(\boldsymbol{X})$ and ${D}^{(0)} := {\mu}_{t_k}(\boldsymbol{X}) - Y(t_0)$. We have
    \begin{equation}
        \begin{aligned} \tau^{(k)}(\boldsymbol{x}) = \mathbb{E}({D}^{(k)} \mid \boldsymbol{X} = \boldsymbol{x}) &= \mathbb{E}(Y(t_k) - {\mu}_{t_0}(\boldsymbol{X}) \mid \boldsymbol{X} = \boldsymbol{x}) \\ 
        &= \mathbb{E}[Y(t_k)  \mid \boldsymbol{X} = \boldsymbol{x}] - {\mu}_{t_0}(\boldsymbol{x}) \\
        &= {\mu}_{t_k}(\boldsymbol{x}) - {\mu}_{t_0}(\boldsymbol{x}) = \tau_k(\boldsymbol{x}), \end{aligned}
    \end{equation}
    and
    \begin{equation}
        \begin{aligned} \tau^{(0)}(\boldsymbol{x}) = \mathbb{E}({D}^{(0)} \mid \boldsymbol{X} = \boldsymbol{x}) &= \mathbb{E}({\mu}_{t_k}(\boldsymbol{X}) - Y(t_k) \mid \boldsymbol{X} = \boldsymbol{x}) \\ 
        &= {\mu}_{t_k}(\boldsymbol{x}) - \mathbb{E}[ \mathbb{E}(Y(t_k) \mid \boldsymbol{X} = \boldsymbol{x}] \\
        &= {\mu}_{t_k}(\boldsymbol{x}) - {\mu}_{t_0}(\boldsymbol{x}) = \tau_k(\boldsymbol{x}). \end{aligned}
    \end{equation}
    
    Therefore,
    \begin{equation}
         {\tau}^{\mathrm{(X,nv)}}_{k}(\boldsymbol{x}) = \frac{r(t_k,\boldsymbol{x})}{r(t_k,\boldsymbol{x}) + r(t_0,\boldsymbol{x})} \tau^{(k)}(\boldsymbol{x}) +  \frac{r(t_0,\boldsymbol{x})}{r(t_k,\boldsymbol{x}) + r(t_0,\boldsymbol{x})} \tau^{(0)}(\boldsymbol{x}) = \tau_k(\boldsymbol{x}).
    \end{equation}
    
\subsection{Generalizing the Robinson decomposition}

We show first the Neyman-Orthogonality propriety, i.e. $\mathbb{E}\left(  \epsilon \mid T, \boldsymbol{X}\right) = 0$. Indeed, for $t \in \mathcal{T}$ and $\boldsymbol{x} \in \mathcal{D}$, we have
\begin{equation}
        \begin{aligned}
        \mathbb{E}\big[  \epsilon \mid T=t, \boldsymbol{X}=\boldsymbol{x}\big] 
        &= \mathbb{E}\big[ Y_{\mathrm{obs}} - \mu_{T}(\boldsymbol{X})\mid T=t, \boldsymbol{X}=\boldsymbol{x}\big] \\ 
        &= \mathbb{E}\big[ Y(t) - \mu_{T}(\boldsymbol{X})\mid T=t, \boldsymbol{X}=\boldsymbol{x}\big] \\
        &= \mu_{t}(\boldsymbol{x}) - \mu_{t}(\boldsymbol{x}) = 0.
    \end{aligned}
\end{equation}

Thus, the observed outcome model satisfies:
\begin{equation}
        \begin{aligned}
        \mathbb{E} ( Y_{\mathrm{obs}} \mid \boldsymbol{X}=\boldsymbol{x})
        &= \mathbb{E} \big[ \epsilon + \sum_{k=0}^K \mathbf{1}\{T=t_k\} \mu_{t_k}(\boldsymbol{X}) \mid \boldsymbol{X}=\boldsymbol{x}\big] \\ 
        &=\mathbb{E} \big[ \mathbb{E} [\epsilon \mid T,\boldsymbol{X}] \mid \boldsymbol{X}=\boldsymbol{x} \big] + \sum_{k=0}^K \mathbb{E} \big[ \mathbf{1}\{T=t_k\}\mid \boldsymbol{X}=\boldsymbol{x}\big] \mu_{t_k}(\boldsymbol{x}) \\
        &= \sum_{k=0}^K \mu_{t_k}(\boldsymbol{x}) r(t_k,\boldsymbol{x}) = \mu_{t_0}(\boldsymbol{x}) r(t_0,\boldsymbol{x}) + \sum_{k=1}^K \mu_{t_k}(\boldsymbol{x}) r(t_k,\boldsymbol{x}) \\
        &= \mu_{t_0}(\boldsymbol{x}) \big[1-\sum_{k=1}^K r(t_k,\boldsymbol{x})\big] + \sum_{k=1}^K \mu_{t_k}(\boldsymbol{x}) r(t_k,\boldsymbol{x}) \\
        &= \mu_{t_0}(\boldsymbol{x}) + \sum_{k=1}^K r(t,\boldsymbol{x}) \left[\mu_{t_k}(\boldsymbol{x}) - \mu_{t_0}(\boldsymbol{x}) \right] \\
        &= \mu_{t_0}(\boldsymbol{x}) + \sum_{k=1}^K r(t_k,\boldsymbol{x}) \tau_{k}(\boldsymbol{x}) = m(\boldsymbol{x}).
        \end{aligned}
    \end{equation}

By gathering both quantities :
    \begin{equation}
        \begin{aligned}
            Y_{\mathrm{obs}} - m(\boldsymbol{X}) &= \sum_{k=0}^K \mathbf{1}\{T =t_k\} \mu_{t_k}(\boldsymbol{X}) - \mu_{t_0}(\boldsymbol{X}) - \sum_{k=1}^K r(t_k,\boldsymbol{X}) \tau_{k}(\boldsymbol{X}) + \epsilon \\
            &= \mathbf{1}\{T =t_0\}\mu_{t_0}(\boldsymbol{X}) + \sum_{k=1}^K \mathbf{1}\{T =t_k\} \mu_{t_k}(\boldsymbol{X}) - \mu_{t_0}(\boldsymbol{X}) - \sum_{k=1}^K r(t_k,\boldsymbol{X}) \tau_{k}(\boldsymbol{X}) + \epsilon \\
            &= \big(\mathbf{1}\{T =t_0\}-1\big)\mu_{t_0}(\boldsymbol{X}) +\sum_{k=1}^K (\mathbf{1}\{T =t_k\} \mu_{t_k}(\boldsymbol{X})-r(t_k,\boldsymbol{X}) \tau_{k}(\boldsymbol{X})) + \epsilon \\
            &= \sum_{k=1}^K (\mathbf{1}\{T =t_k\} \mu_{t_k}(\boldsymbol{X})-r(t_k,\boldsymbol{X}) \tau_{k}(\boldsymbol{X})) - \sum_{k=1}^K \mathbf{1}\{T =t_k\} \mu_{t_0}(\boldsymbol{X}) + \epsilon \\
            &= \sum_{k=1}^K (\mathbf{1}\{T =t_k\} \mu_{t_k}(\boldsymbol{X}) - \mathbf{1}\{T =t_k\}  \mu_{t_0}(\boldsymbol{X})-r(t_k,\boldsymbol{X}) \tau_{k}(\boldsymbol{X})) + \epsilon \\
            &= \sum_{k=1}^K \big[\mathbf{1}\{T =t_k\}  -r(t_k,\boldsymbol{X}) \big] \tau_{k}(\boldsymbol{X}) + \epsilon.  
        \end{aligned}
    \end{equation}

Therefore, we obtain the generalized Robinson decomposition for the multi-treatment regime.

\subsection{Solving the generalized R-learner for linear models}

For $k=1,\ldots,K$, we assume that $\overline{\tau}_{k}$ belongs to the family of linear regression models such that:
    \begin{equation}
        \mathcal{F} = \Big\{  \big\{\overline{\tau}_{k}(\boldsymbol{x}) := \beta_{k,0} + \sum_{j=1}^{p-1} \beta_{k,j} f_j(\boldsymbol{x}) \big\}_{k=1}^K \ / \ \boldsymbol{\beta}_k = (\beta_{k,0},\ldots,\beta_{k,p-1})^{\top} \in \mathbb{R}^{p} \Big\}.
    \end{equation}
    $f_j$ are predefined functions (e.g. polynomial functions). It is also possible to use a matrix notation and write $\overline{\tau}_{k}(\mathbf{X}) = \mathbf{H}  \boldsymbol{\beta}_k$ where $ \textbf{H}= (f_j(\boldsymbol{X}_i)) \in \mathbb{R}^{n\times p}$ assumed to be full rank matrix $rank(\mathbf{H})=p \leq n$. 

    Let $\overline{Y}=(\overline{Y}_{i})_{i=1}^n$ and $\overline{T}_k=(\overline{T}_{i,k})_{i=1}^n$ such that $\overline{Y}_i=Y_{\mathrm{obs},i}-\widehat{m}(\boldsymbol{X}_i)$ and $\overline{T}_{i,k} = \mathbf{1}\{t_i =  t_k\} - \widehat{r}( t_k,\boldsymbol{X}_i)$. Let $\epsilon=\left( \epsilon_i \right)_{i=1}^n$ denote the vector of errors obtained for the generalized \citet{Robinson1988} decomposition in Proposition 3.3.

    We show immediately that $\mathcal{L}$, the generalized R-loss function associated with the mean squared error of $\epsilon$ in (2) in the paper, is quadratic with respect to $\boldsymbol{\beta}$. Indeed,
    \begin{equation}
        \begin{aligned}
            \mathcal{L}(\{ \overline{\tau}_{k} \}_{t \neq t^{(0)}}) &=  \frac{1}{n} \epsilon^{\top} \epsilon = \frac{1}{n} \Big( \overline{Y} - \sum_{k=1}^K \displaystyle \overline{T}_k \odot (\mathbf{H} \boldsymbol{\beta}_k) \Big)^{\top} \Big( \overline{Y} - \sum_{k=1}^K \displaystyle \overline{T}_k \odot (\mathbf{H} \boldsymbol{\beta}_k) \Big) \\
            &= \frac{1}{n} \left[ \overline{Y}^{\top} \overline{Y} - 2 \sum_{k=1}^K \overline{Y}^{\top} {\big(\displaystyle \overline{T}_k \odot (\mathbf{H} \boldsymbol{\beta}_k) \big)} + \sum^K_{k,k'=1} {\big(\displaystyle \overline{T}_k \odot (\mathbf{H} \boldsymbol{\beta}_k)\big)^{\top}}\ {\big(\displaystyle \overline{T}_{k'} \odot (\mathbf{H} \boldsymbol{\beta}_{k'})\big)} \right] \\
            &=  \frac{1}{n} \Big( \overline{Y}^{\top} \overline{Y} - 2 \sum_{k=1}^K \overline{Y}^{\top} {\mathbf{D}_{\overline{T}_{k}} \mathbf{H}  \boldsymbol{\beta}_{k}} + \sum^K_{k,k'=1}  \boldsymbol{\beta}_{k}^{\top} \mathbf{H} ^{\top} \mathbf{D}_{\overline{T}_{k}}   \mathbf{D}_{\overline{T}_{{k'}}} \mathbf{H}  \boldsymbol{\beta}_{{k'}} \Big),
        \end{aligned}  
    \end{equation}
    where $\odot$ is the Hadamard product (element-wise product). The last line holds because $\displaystyle \overline{T}_k \odot (\mathbf{H} \boldsymbol{\beta}_k)= \mathbf{D}_{\overline{T}_k} \mathbf{H} \boldsymbol{\beta}_k$ with $\mathbf{D}_{\overline{T}_k}$ is the diagonal matrix of the vector $\overline{T}_k=(\overline{T}_{i,k})_{i=1}^n$

    By differentiating  $\partial \mathcal{L}/\partial \boldsymbol{\beta}_k = 0$ for $k=1,\ldots,K$ : 
    \begin{gather}
    \label{eq:sys_Rlearner}
        \left\{  \begin{array}{l l}
            - \boldsymbol{a}_1 + \mathbf{B}_1 \widehat{\boldsymbol{\beta}}_1 + \sum_{k=2}^K \mathbf{C}_{1 k}  \widehat{\boldsymbol{\beta}}_k  = 0 \\[.15cm]
            \qquad \vdots \qquad \qquad \vdots \qquad \qquad \vdots \qquad =  0 \\[.15cm]
            - \boldsymbol{a}_K + \sum_{k=1}^K \mathbf{C}_{K k}  \widehat{\boldsymbol{\beta}}_k + \mathbf{B}_K  \widehat{\boldsymbol{\beta}}_K = 0  \\ \end{array} \right. \\
        \Longleftrightarrow 
        \begin{bmatrix}
        \mathbf{B}_1 & \mathbf{C}_{12} & \cdots & \mathbf{C}_{1 K}\\
        \mathbf{C}_{2 1} & \mathbf{B}_2 & \cdots & \mathbf{C}_{2 K}\\
        \vdots & \vdots & \ddots & \vdots\\
        \mathbf{C}_{K 1} & \mathbf{C}_{K 2} & \cdots & \mathbf{B}_K
        \end{bmatrix}
        \begin{bmatrix}
         \widehat{\boldsymbol{\beta}}_1\\ \widehat{\boldsymbol{\beta}}_2\\ \vdots\\ \widehat{\boldsymbol{\beta}}_K
        \end{bmatrix}
        =\begin{bmatrix}
        \boldsymbol{a}_1\\ \boldsymbol{a}_2\\ \vdots\\ \boldsymbol{a}_K
        \end{bmatrix},
    \end{gather}

    where 
    \begin{gather}
        \boldsymbol{a}_j = \frac{1}{n} \mathbf{H} ^{\top} \mathbf{D}_{\overline{T}_j} \overline{Y} \in \mathbb{R}^{p},\\ 
        \mathbf{B}_j = \frac{1}{n} \mathbf{H} ^{\top} \mathbf{D}^2_{\overline{T}_j}\ \mathbf{H} \in \mathbb{R}^{p \times p},\\
        \mathbf{C}_{i j} = \frac{1}{n} \mathbf{H} ^{\top} \mathbf{D}_{\overline{T}_i} \mathbf{D}_{\overline{T}_j} \mathbf{H}  \in \mathbb{R}^{ p\times p}.
    \end{gather}

    Let $\boldsymbol{\beta} = \left( \boldsymbol{\beta}^{\top}_1,\ldots,\boldsymbol{\beta}^{\top}_K\right)^{\top} \in \mathbb{R}^{K \times p}$ and consider the block matrix $\mathbf{A}$ defined as
    
    \begin{equation}
        \mathbf{A} = \begin{bmatrix}
        \mathbf{B}_1 & \mathbf{C}_{12} & \cdots & \mathbf{C}_{1 K}\\
        \mathbf{C}_{2 1} & \mathbf{B}_2 & \cdots & \mathbf{C}_{2 K}\\
        \vdots & \vdots & \ddots & \vdots\\
        \mathbf{C}_{K 1} & \mathbf{C}_{K 2} & \cdots & \mathbf{B}_K
        \end{bmatrix}.
    \end{equation}

    The matrix $\mathbf{A}$ is real symmetric and satisfies:
    \begin{equation}
        \begin{aligned}
            \boldsymbol{\beta}^{\top} \mathbf{A} \boldsymbol{\beta} &= \sum_{1\leq k,l \leq K} \boldsymbol{\beta}_k^{\top} \mathbf{H}^{\top} \mathbf{D}_{\overline{T}_k} \mathbf{D}_{\overline{T}_l}  \mathbf{H} \boldsymbol{\beta}_l \\
            &= \bigg \Vert \sum_{k=1}^K \mathbf{D}_{\overline{T}_k}  \mathbf{H} \boldsymbol{\beta}_k \bigg \Vert^2 \geq 0.
        \end{aligned} 
    \end{equation}
    
This result shows that $\mathbf{A}$ is positive semi-definite, all its eigenvalues are nonnegative and also proves the existence of a minimizer $\widehat{\boldsymbol{\beta}}$ to the loss function $\mathcal{L}$. However, this is not sufficient to prove the uniqueness of the solution as one cannot prove all eigenvalues are positive. 
    
The solution $\widehat{\boldsymbol{\beta}}$ to Problem (\ref{eq:RLearnerMulti}) in the main paper with the minimal norm is given by
\begin{equation}
    \widehat{\boldsymbol{\beta}} = \mathbf{A}^+ \boldsymbol{a},
\end{equation}
where $\mathbf{A}^+$ is the Moore–Penrose inverse of $\mathbf{A}$ and $\boldsymbol{a} = \left( \boldsymbol{a}^{\top}_1,\ldots,\boldsymbol{a}^{\top}_K\right)^{\top}$. 

\vspace{0.2in}

\section{Theoretical analysis of the error bounds.}
\label{App:Error_estim}

\vspace{0.1in}

\subsection{Error estimation of pseudo-outcome meta-learners.}
\label{App:pseudo_outcome}

\vspace{0.1in}

\subsection*{Step 0. Set-up of the theorem}

In the following subsection, we will analyze the error estimation of each two-step meta-learner. Given Assumption \ref{assump:model_Y} stating that the observations are generated from a function $f$ respecting the two causal assumptions (\ref{assump:unconfound}-\ref{assump:overlap}), each unit $i$ has the following observed and potential outcomes
\begin{equation}
    \begin{aligned}
        &Y_i(t_k) = f(t_k,\boldsymbol{X}_i) + \varepsilon_i(t_k), \\
        &Y_i(t_0) = f(t_0,\boldsymbol{X}_i) + \varepsilon_i(t_0).
    \end{aligned}
\end{equation}
where $\varepsilon_i(t)$ are i.i.d. Gaussian $\mathcal{N}(0, \sigma^2)$ and independent of $(T_i,\boldsymbol{X}_i)_{i=1}^n$. As a consequence, the noise $(\epsilon_i)_{i=1}^n=(\varepsilon_i(T_i))_{i=1}^n$ is also Gaussian $\mathcal{N}(0, \sigma^2)$ and is independent of $(T_i,\boldsymbol{X}_i)_{i=1}^n$.

The CATE model $\tau_k$ for each $k=1,\ldots,K$ can be written as:
\begin{equation}
    \begin{aligned}
        \tau_k(\boldsymbol{x}) &= \mathbb{E}(Y(t_k) - Y(t_0) \mid \boldsymbol{X} = \boldsymbol{x}) \\
        &= \mathbb{E}(f(t_k,\boldsymbol{X}) - f(t_0,\boldsymbol{X}) + \epsilon^* \mid \boldsymbol{X} = \boldsymbol{x})  \\
        &= f(t_k,\boldsymbol{x}) - f(t_0,\boldsymbol{x}) \\
    \end{aligned}
\end{equation}
where $\epsilon^*$ is a noise independent of $\boldsymbol{X}$ (and $T$) and satisfying $\mathbb{E}(\epsilon^*)=0$.

Under the assumption \ref{assump:beta_Y}, we write $\tau_k(\mathbf{X}) = f(t_k,\mathbf{X}) - f(t_0,\mathbf{X}) = \mathbf{H}\boldsymbol{\beta}_k^*$ where $\boldsymbol{\beta}_k^*=\boldsymbol{\beta}_{t_k}-\boldsymbol{\beta}_{t_0}$ and  $\mathbf{H} = (\mathbf{H}_{i j}) \in \mathbb{R}^{n\times p}$ is the regression matrix, assumed to be full rank matrix, such that $\mathbf{H}_{i j} = f_j(\boldsymbol{X}_i)$ for $ i=1,\ldots,n$ and $j= 0,\ldots,p-1$. 
With pseudo-outcome meta-learners, we consider a random variable $Z_k$ for a fixed $t_k$ such that 
\begin{equation*}
    Z_{k,i} = A_{t_k}(T_i,\boldsymbol{X}_i) Y_{\mathrm{obs},i} + B_{t_k}(T_i,\boldsymbol{X}_i), \quad i=1,\ldots,n,
\end{equation*}
where the functions $A_{t_k}(T,\boldsymbol{X})$ and $B_{t_k}(T,\boldsymbol{X})$ are given for each pseudo-outcome meta-learner.

\subsection*{Step 1. Identification of $\widehat{\boldsymbol{\beta}}_k$ and $\boldsymbol{\beta}_k^*$}

The regression coefficients $\widehat{\boldsymbol{\beta}}_k$ are given by the Ordinary Least Squares (OLS) method 
\begin{equation}
    \widehat{\boldsymbol{\beta}}_k= \big(\mathbf{H}^{\top}\mathbf{H}\big)^{-1} \mathbf{H}^{\top} \boldsymbol{z}_k,
\end{equation}
where $\boldsymbol{z}_k = (Z_{k,i})_{1 \leq i \leq n}$. Thus,
\begin{equation*}
    \begin{aligned}
    &\widehat{\boldsymbol{\beta}}_k = \big(\mathbf{H}^{\top}\mathbf{H}\big)^{-1} \mathbf{H}^{\top} \boldsymbol{z}_k \\
    &= \big(\mathbf{H}^{\top}\mathbf{H}\big)^{-1} \mathbf{H}^{\top} \big(  A_{t_k}(T_i,\boldsymbol{X}_i) Y_{\mathrm{obs},i} + B_{t_k}(T_i,\boldsymbol{X}_i) \big)^n_{i=1} \\ 
    &= \big(\mathbf{H}^{\top}\mathbf{H}\big)^{-1} \mathbf{H}^{\top} \big( A_{t_k}(T_i,\boldsymbol{X}_i) f(T_i, \boldsymbol{X}_i) + B_{t_k}(T_i,\boldsymbol{X}_i) + A_{t_k}(T_i,\boldsymbol{X}_i) \epsilon_i \big) ^n_{i=1} \\
    &= \big(\mathbf{H}^{\top}\mathbf{H}\big)^{-1} \mathbf{H}^{\top} \big( \tau_k(\boldsymbol{x})+ A_{t_k}(T_i,\boldsymbol{X}_i) f(T_i, \boldsymbol{X}_i) - \tau_k(\boldsymbol{x}) + B_{t_k}(T_i,\boldsymbol{X}_i) + A_{t_k}(T_i,\boldsymbol{X}_i) \epsilon_i \big)^n_{i=1} \\
    &= \big(\mathbf{H}^{\top}\mathbf{H}\big)^{-1} \mathbf{H}^{\top} \big( \mathbf{H}\boldsymbol{\beta}_k^* + A_{t_k}(T_i,\boldsymbol{X}_i) f(T_i, \boldsymbol{X}_i) - \tau_k(\boldsymbol{x}) + B_{t_k}(T_i,\boldsymbol{X}_i) + A_{t_k}(T_i,\boldsymbol{X}_i) \epsilon_i \big) ^n_{i=1} \\
    &= \boldsymbol{\beta}_k^* + \big(\mathbf{H}^{\top}\mathbf{H}\big)^{-1} \mathbf{H}^{\top} \big( A_{t_k}(T_i,\boldsymbol{X}_i) f(T_i, \boldsymbol{X}_i) - \tau_k(\boldsymbol{x}) + B_{t_k}(T_i,\boldsymbol{X}_i) + A_{t_k}(T_i,\boldsymbol{X}_i) \epsilon_i \big)^n_{i=1} \\ 
    &= \boldsymbol{\beta}_k^* + \big(\mathbf{H}^{\top}\mathbf{H}\big)^{-1} \mathbf{H}^{\top} \tilde{\boldsymbol{\epsilon}}_k
    \end{aligned}
\end{equation*}
where $\tilde{\epsilon}_{k,i} = \psi_k(T_i,\boldsymbol{X}_i) + A_{t_k}(T_i,\boldsymbol{X}_i) \epsilon_i$ and  $\psi_k(T_i,\boldsymbol{X}_i)= A_{t_k}(T_i,\boldsymbol{X}_i) f(T_i, \boldsymbol{X}_i) - \tau_k(\boldsymbol{X}_i) + B_{t_k}(T_i,\boldsymbol{X}_i)$ to simplify notations.

Let us consider the random vector $\boldsymbol{Z}_k^{(n)}$ such that
\begin{equation}
    \boldsymbol{Z}_k^{(n)} =  \big(\frac{1}{n}(\mathbf{H}^{\top} \tilde{\boldsymbol{\epsilon}}_k)_1, \ldots, \frac{1}{n}(\mathbf{H}^{\top} \tilde{\boldsymbol{\epsilon}}_k)_p, \frac{1}{n}(\mathbf{H}^{\top}\mathbf{H})_{11},\ldots ,\frac{1}{n}(\mathbf{H}^{\top}\mathbf{H})_{pp}\big)^{\top} \in \mathbb{R}^{p+p^2},
\end{equation}
that allows us to write $ \widehat{\boldsymbol{\beta}}_{k}$ as 
\begin{equation}
    \begin{aligned}
        \widehat{\boldsymbol{\beta}}_{k} &= \boldsymbol{\beta}_k^* +\big(\mathbf{H}^{\top}\mathbf{H}\big)^{-1} \mathbf{H}^{\top} \tilde{\boldsymbol{\epsilon}}_k\\
        &= \boldsymbol{\beta}_k^* +\Big( \frac{1}{n} \mathbf{H}^{\top}\mathbf{H}\Big)^{-1} \Big( \frac{1}{n} \mathbf{H}^{\top} \tilde{\boldsymbol{\epsilon}}_k \Big)\\
        &= \boldsymbol{\beta}_k^* + \phi(\boldsymbol{Z}_k^{(n)}),
    \end{aligned}
\end{equation}
where $\phi: \mathbb{R}^{p+p^2} \longrightarrow \mathbb{R}^{p}$ is a $\mathcal{C}^1$-function.

\subsection*{Step 2. The asymptotic behaviour of the OLS estimator's mean and covariance}

In order to apply the Central Limit Theorem (CLT) later, we show that the vector $\boldsymbol{Z}_k^{(n)}$ can be written as sum of \textit{i.i.d.} random vectors $\boldsymbol{Z}_{k,i}$.

\begin{equation}
    \begin{aligned} \boldsymbol{Z}_k^{(n)} &=  \big(\frac{1}{n}(\mathbf{H}^{\top} \tilde{\boldsymbol{\epsilon}}_k)_1, \ldots, \frac{1}{n}(\mathbf{H}^{\top} \tilde{\boldsymbol{\epsilon}}_k)_p, \frac{1}{n}(\mathbf{H}^{\top}\mathbf{H})_{11},\ldots ,\frac{1}{n}(\mathbf{H}^{\top}\mathbf{H})_{pp}\big)^{\top} \in \mathbb{R}^{p+p^2} \\
    &= \big( \frac{1}{n} \sum_{i=1}^n \mathbf{H}_{i 1} \tilde{\epsilon}_{k,i}, \ldots,\mathbf{H}_{i p} \tilde{\epsilon}_{k,i}, \frac{1}{n} \sum_{i=1}^n \mathbf{H}_{i 1} \mathbf{H}_{i 1},\ldots, \frac{1}{n} \sum_{i=1}^n \mathbf{H}_{i p} \mathbf{H}_{i p}\big)^{\top} \\
    &= \frac{1}{n} \sum_{i=1}^n \big( \mathbf{H}_{i 1} \tilde{\epsilon}_{k,i}, \ldots,\mathbf{H}_{i p} \tilde{\epsilon}_{k,i}, \mathbf{H}_{i 1} \mathbf{H}_{i 1},\ldots, \mathbf{H}_{i p} \mathbf{H}_{i p}\big)^{\top} = \frac{1}{n} \sum_{i=1}^n \boldsymbol{Z}_{k,i}.
    \end{aligned}
\end{equation}

The mean $\boldsymbol{m}$ of the vector $\boldsymbol{Z}_k^{(n)}$ satisfies
\begin{equation}
    \begin{aligned}
        \boldsymbol{m} &= \mathbb{E}(\boldsymbol{Z}_k^{(n)}) = \frac{1}{n} \sum_{i=1}^n \mathbb{E}(\boldsymbol{Z}_{k,i}) = \mathbb{E}(\boldsymbol{Z}_{k,i}) \\
        &= \Big( h_1, \ldots,  h_{p}, F_{1 1},\ldots,F_{p p}\Big)^{\top},
    \end{aligned}
\end{equation}
where, for $j,j'=1,\ldots,p$,
\begin{equation}
    \begin{aligned}
        &h_j = \mathbb{E}\big[ f_{j-1}(\boldsymbol{X}) \big( \psi_k(T, \boldsymbol{X}) + A_{t_k}(T,\boldsymbol{X}) \epsilon \big) \big] = \mathbb{E}\big(f_{j-1}(\boldsymbol{X}) \psi_k(T, \boldsymbol{X}) \big) \\
        &\qquad\qquad\qquad\qquad F_{j j'} =  \mathbb{E}\big( f_{j-1}(\boldsymbol{X}) f_{j'-1}(\boldsymbol{X}) \big).
    \end{aligned}
\end{equation}

The covariance matrix $\mathbf{C}$ of $\boldsymbol{Z}_k^{(n)}$ has entries
\begin{equation}
    \begin{aligned}
    &\mathbf{C}_{j j'} = \operatorname{Cov} \Big(\boldsymbol{Z}^{(k)}_{j}, \boldsymbol{Z}^{(k)}_{j'}\big) = \mathbb{E}(\boldsymbol{Z}^{(k)}_{j}, \boldsymbol{Z}^{(k)}_{j'}) - \mathbb{E}(\boldsymbol{Z}^{(k)}_{j}) \mathbb{E}(\boldsymbol{Z}^{(k)}_{j'})  \\
    &= \left\{ \begin{array}{l l}
        \mathbb{E}\big( f_{j-1}(\boldsymbol{X}) f_{j'-1}(\boldsymbol{X}) \big( \psi_k(T, \boldsymbol{X}) + A_{t_k}(T,\boldsymbol{X}) \epsilon \big)^2 \big)-h_j h_{j'} & \quad \text{if $j,j' \in \{1,\ldots,p\}$}\\[.15cm]
        \mathbb{E}\big( f_{\tilde k - 1}(\boldsymbol{X}) f_{\tilde k'-1}(\boldsymbol{X}) f_{l-1}(\boldsymbol{X}) f_{l'-1}(\boldsymbol{X}) \big)- F_{k k'} F_{l l'} & \quad \text{if $j,j' \in \{p+1,\ldots, p^2\}$}\\
        \mathbb{E}\big( f_{\tilde k-1}(\boldsymbol{X}) f_{\tilde k'-1}(\boldsymbol{X}) f_{j-1}(\boldsymbol{X}) \big( \psi_k(T, \boldsymbol{X}) + A_{t_k}(T,\boldsymbol{X}) \epsilon \big)  \big)- h_j F_{k k'}  & \quad \text{otherwise.}\\[.15cm]
        \end{array} \right. \\[.25cm]
    &= \left\{ \begin{array}{l l}
        \mathbb{E} \big(f_{j-1}(\boldsymbol{X}) f_{j'-1}(\boldsymbol{X}) \psi^2_k(T, \boldsymbol{X})\big) + \sigma^2 \mathbb{E} \big(f_{j-1}(\boldsymbol{X}) f_{j'-1}(\boldsymbol{X}) A^2_{t_k}(T,\boldsymbol{X}) \big) - h_j h_{j'} & \quad \text{if $j,j' \in \{1,\ldots,p\}$ }\\[.15cm]
        \mathbb{E}\big( f_{\tilde k-1}(\boldsymbol{X}) f_{\tilde k'-1}(\boldsymbol{X}) f_{l-1}(\boldsymbol{X}) f_{l'-1}(\boldsymbol{X}) \big)- F_{\tilde k \tilde k'} F_{l l'} & \quad \text{if $j,j' \in \{p+1,\ldots,p^2\}$}\\[.15cm]
        \mathbb{E}\big( f_{\tilde k-1}(\boldsymbol{X}) f_{\tilde k'-1}(\boldsymbol{X}) f_{j-1}(\boldsymbol{X}) \psi_k(T, \boldsymbol{X})  \big)- h_j F_{\tilde k \tilde k'}  & \quad \text{otherwise,}\\
         \end{array} \right.
    \end{aligned}
\end{equation}
where ${\tilde k},{\tilde k'}=\eta^{-1}(j)$ (respectively, $l,l'=\eta^{-1}(j')$) such that $\eta$ is the correspondence indexes map between $\boldsymbol{m}$ and $F$ in $\boldsymbol{m}_j = F_{{\tilde k} {\tilde k'}}$ (respectively, $\boldsymbol{m}_{j'}= F_{l l'}$)  when $j \geq p+1$ (respectively, when $j' \geq p+1$).

By considering now the vector 
\begin{equation}
    \boldsymbol{S}^{(n)} = \sqrt{n} \big( \boldsymbol{Z}_k^{(n)} - \boldsymbol{m} \big) = \frac{1}{\sqrt{n}} \sum_{i=1}^n  \big( \boldsymbol{Z}_{k,i} -  \boldsymbol{m}\big),
\end{equation}

one can show by the multivariate CLT that
\begin{equation}
    \boldsymbol{S}^{(n)} = \sqrt{n} \big( \boldsymbol{Z}_k^{(n)} - \boldsymbol{m} \big) \overset{\mathcal{L}}{\longrightarrow} \mathcal{N}(\boldsymbol{0}, \mathbf{C}).
\end{equation}

This allows us to write $\widehat{\boldsymbol{\beta}}_k$ as function of $\boldsymbol{S}^{(n)}$ and $\boldsymbol{m}$. Indeed,
\begin{equation}
    \begin{aligned}
        \widehat{\boldsymbol{\beta}}_k &= \boldsymbol{\beta}_k^* + \big(\mathbf{H}^{\top}\mathbf{H}\big)^{-1} \mathbf{H}^{\top} \tilde{\boldsymbol{\epsilon}} \\
        &= \boldsymbol{\beta}_k^* + \phi(\boldsymbol{Z}^{(n)}) \\
        &= \boldsymbol{\beta}_k^* + \phi \big( \boldsymbol{m} + \boldsymbol{S}^{(n)}/\sqrt{n} \big) \\
        &= \boldsymbol{\beta}_k^* + \Phi(\boldsymbol{S}^{(n)}, \boldsymbol{m}),
    \end{aligned}
\end{equation}
where $\Phi: \mathbb{R}^{p+p^2} \times \mathbb{R}^{p+p^2} \longrightarrow \mathbb{R}^{p}$ is also $\mathcal{C}^1$-function.

Since $\sqrt{n} \big( \boldsymbol{S}^{(n)} - \boldsymbol{0} \big)  \overset{\mathcal{L}}{\longrightarrow} \mathcal{N}(\boldsymbol{0}, \mathbf{C})$, one obtains by the Delta method
\begin{equation}
    \sqrt{n} \Big[ \Phi(S^{(n)}, \boldsymbol{m}) - \Phi(\boldsymbol{0}, \boldsymbol{m}) \Big] \overset{\mathcal{L}}{\longrightarrow} \mathcal{N}\left(\boldsymbol{0},J^{(1)}_\Phi(\boldsymbol{0}, \boldsymbol{m})^{\top}  \mathbf{C}  J^{(1)}_\Phi(\boldsymbol{0}, \boldsymbol{m})\right),
\end{equation}
where $J^{(1)}_\Phi(\boldsymbol{0}, \boldsymbol{m})$ is the Jacobian matrix at the first $p+p^2$ coordinates of $\Phi$ at $(\boldsymbol{0}, \boldsymbol{m})$. 

By denoting $\boldsymbol{g}_{n}$, a Gaussian noise with zero-mean and covariance matrix $\mathbf{C}' = J^{(1)}_\Phi(\boldsymbol{0}, \boldsymbol{m})^{\top}  \mathbf{C}  J^{(1)}_\Phi(\boldsymbol{0}, \boldsymbol{m})$, the previous equation is equivalent to
\begin{equation}
    \widehat{\boldsymbol{\beta}}_k = \boldsymbol{\beta}_k^* + \Phi(\boldsymbol{S}_{n}, \boldsymbol{m}) \approx \boldsymbol{\beta}_k^* + \Phi(\boldsymbol{0}, \boldsymbol{m}) + \boldsymbol{g}_{n}/\sqrt{n}.
\end{equation}

For $n$ large, the expansions of the first moments is of the form:
\begin{equation}
    \mathbb{E}(\widehat{\boldsymbol{\beta}}_k) \approx \boldsymbol{\beta}_k^* + \Phi(\boldsymbol{0}, \boldsymbol{m}).
\end{equation}
and, the asymptotic variance is also of the form:
\begin{equation}
    \begin{aligned}
        \mathbb{V}(\widehat{\boldsymbol{\beta}}_k) & \approx \frac{1}{n} \ J^{(1)}_\Phi(\boldsymbol{0}, \boldsymbol{m})^{\top} \mathbf{C}  J^{(1)}_\Phi(\boldsymbol{0}, \boldsymbol{m}).
    \end{aligned}
\end{equation}

This result holds whether the nuisance parameters in $A_t$ and $B_t$ are well-specified or not, so there is no guarantee that $\Phi(\boldsymbol{0}, \boldsymbol{m}) = 0$ and the estimator $\widehat{\boldsymbol{\beta}}_k$ may be biased.

In the following, we assume that the nuisance parameters in $A_t$ and $B_t$ are well-specified i.e. $\mathbb{E} \big(\psi_k(T, \boldsymbol{X})) \mid \boldsymbol{X} = \boldsymbol{x} \big) = 0$ in such way that $\mathbb{E} (Z_k \mid \boldsymbol{X} = \boldsymbol{x}) = \tau_k(\boldsymbol{x})$, or equivalently, $\mathbb{E} \big(\mathbf{H}^{\top} \tilde{\epsilon}_k \big) = \boldsymbol{0}$. Consequently, the estimator of $\widehat{\boldsymbol{\beta}}_k$ is unbiased. In this case, computing the variance $\mathbb{V}(\widehat{\boldsymbol{\beta}}_k)$ becomes much easier and more explicit. 

On the one hand, by the multivariate Central Theorem Limit (CTL)
\begin{equation}
    \frac{1}{\sqrt{n}} \mathbf{H}^{\top} \tilde{\boldsymbol{\epsilon}}_k \overset{\mathcal{L}}{\longrightarrow} \mathcal{N}(\boldsymbol{0}, \boldsymbol{\Sigma})
\end{equation}

which is equivalent to
\begin{equation}
    \frac{1}{\sqrt{n}} \ \mathbf{H}^{\top} \tilde{\boldsymbol{\epsilon}}_k \ \approx  \boldsymbol{g}_{n},
\end{equation}

where $\boldsymbol{g}_{n}$ is a Gaussian noise with zero-mean and covariance matrix of $\boldsymbol{\Sigma}$ with entries 
\begin{equation}
    \begin{aligned}
        \boldsymbol{\Sigma}_{j j'} &= \mathbb{E} \big[f_j(\boldsymbol{X}) f_{j'}(\boldsymbol{X}) \big( \psi_k(T, \boldsymbol{X}) + A_{t_k}(T,\boldsymbol{X}) \epsilon \big)^2 \big]  \\
        &= \mathbb{E} \big(f_j(\boldsymbol{X}) f_{j'}(\boldsymbol{X}) \psi^2_k(T, \boldsymbol{X})\big) + \sigma^2 \mathbb{E} \big(f_j(\boldsymbol{X}) f_{j'}(\boldsymbol{X}) A^2_{t_k}(T,\boldsymbol{X}) \big).
    \end{aligned}
\end{equation}

On the other hand, by the law of large numbers (LLN), we have $1/n \big(\mathbf{H}^{\top}\mathbf{H}\big) \overset{a.s}{\longrightarrow} \mathbf{F}$, thus $1/n \big(\mathbf{H}^{\top}\mathbf{H}\big) \overset{P}{\longrightarrow} \mathbf{F}$. Since $\mathbf{F}$ is invertible, then
\begin{equation}
    \begin{aligned}
        n \big(\mathbf{H}^{\top}\mathbf{H}\big)^{-1}
        &\overset{P}{\longrightarrow} \mathbf{F}^{-1},
    \end{aligned}
\end{equation}
where $\mathbf{F}=(F_{j j'})_{0 \leq j, j' \leq p-1}$ and $F_{j j'} =  \mathbb{E}\big( f_j(\boldsymbol{X}) f_{j'}(\boldsymbol{X}) \big)$.

By Slutsky’s theorem,
\begin{equation}
    \begin{aligned}
        \sqrt{n} \big( \widehat{\boldsymbol{\beta}}_k - \boldsymbol{\beta}_k^*\big) &= n \big(\mathbf{H}^{\top}\mathbf{H}\big)^{-1} \cdot 1/\sqrt{n} \ \mathbf{H}^{\top} \tilde{\boldsymbol{\epsilon}} \\
        &\overset{\mathcal{L}}{\longrightarrow}\mathcal{N}(\boldsymbol{0}, \mathbf{F}^{-1} \boldsymbol{\Sigma}\mathbf{F}^{-1} ).
    \end{aligned}
\end{equation}

We can deduce that the asymptotic mean and variance are of the form
\begin{equation}
    \begin{aligned}
        &\mathbb{E}(\widehat{\boldsymbol{\beta}}_k) = \boldsymbol{\beta}_k^*, \\ &\mathbb{V}(\widehat{\boldsymbol{\beta}}_k)\approx \frac{1}{n} \ \mathbf{F}^{-1} \boldsymbol{\Sigma} \mathbf{F}^{-1}.
    \end{aligned}
\end{equation}

\subsection*{Step 3. Obtaining the error upper bound}

The determinant of the variance matrix, also known as the generalized variance by \citet{Wilks1932, Wilks1967} is usually used as a scalar measure of overall multidimensional scatter and can be useful to compare the variance of each meta-learner. 

In our case, comparing the generalized variance is equivalent to comparing $\det{ \Big(\frac{1}{n} \boldsymbol{\Sigma} \Big)}$ of each pseudo-outcome meta-learner since 
\begin{equation}
    \det{\big( \mathbb{V}(\widehat{\boldsymbol{\beta}}_k) \big)} = \big(\det{ \mathbf{F}^{-1}} \big)^2 \det{ \Big( \frac{1}{n} \boldsymbol{\Sigma} \Big)} = \frac{1}{\big(\det{ \mathbf{F}} \big)^2} \det{ \Big(\frac{1}{n} \boldsymbol{\Sigma} \Big)},
\end{equation}

with, obviously, $\det{(\boldsymbol{\Sigma})} >0$ because $\boldsymbol{\Sigma}$ is symmetric positive definite. 

In some cases, the polynomials $f_j$ are chosen to be orthonormal with respect to the distribution of $\boldsymbol{X}$ (e.g. Polynomials Chaos \citep{Sudret2008}). A consequence of this choice implies that $\mathbf{F}$ is the identity matrix. Therefore, in the following, we focus on computing and bounding $\boldsymbol{\Sigma}$ terms. The assumptions (\ref{assump:unconfound}-\ref{assump:beta_Y}) and the following lemma will be used for this purpose.

\begin{lemma}
\label{AppB:lemma1}
    If $X_1, \ldots, X_m$ is a sequence of random variables and $b>1$, then
    \begin{equation}
    \begin{aligned}
        &\bigg| \mathbb{E} \Big[ \Big( \sum_{i=1}^m X_i\Big)^2 \Big] \bigg| \leq m \sum_{i=1}^m \mathbb{E}\Big[ \big| X^2_i \big| \Big], \\
        &\bigg| \mathbb{E} \Big[ \Big( \sum_{i=1}^m X_i\Big)^b \Big] \bigg| \leq m^{(b-1)} \sum_{i=1}^m \mathbb{E}\Big[ \big| X^b_i \big| \Big].
    \end{aligned}
    \end{equation}
\end{lemma}
\begin{proof}
    The first inequality is obtained by Cauchy-Schwarz, whereas the second inequality can be proved by Jensen inequality. Indeed, for $b>1$, the function $x \mapsto x^b$ is convex for $x>0$ and
    \begin{equation}
        \bigg| \frac{\sum_{i=1}^m X_i}{m} \bigg|^b \leq \frac{\sum_{i=1}^m |X_i|^b}{m}.
    \end{equation}
    Therefore,
    \begin{equation}
        \bigg| \mathbb{E} \Big[ \Big( \sum_{i=1}^m X_i\Big)^b \Big] \bigg| \leq \mathbb{E} \Big[ \Big| \sum_{i=1}^m X_i \Big|^b \Big] \leq m^{(b-1)} \sum_{i=1}^m \mathbb{E}\Big[ \big| X^b_i \big| \Big].
    \end{equation}
    
\end{proof}

In the following and by Assumption \ref{assump:model_Y}, $f_j(\boldsymbol{X}) \in L^{a}$ i.e. $f_j(\boldsymbol{X})$ has all possible finite moments for all $j \in \{0,\ldots,p-1\}$. Moreover, there exists $C>0$ such that:
\begin{equation}
    \forall t \in \mathcal{T}, \forall \boldsymbol{x} \in \mathcal{D}: \ | f(t,\boldsymbol{x}) | \leq C.
\end{equation}

\subsubsection{Error estimation of the M-learner}

Let $a, b > 1$ such that $1/a + 1/b = 1$. We denote $\delta^{(a)}_{j j'} = \left| \mathbb{E}\big( f^{a}_j(\boldsymbol{X}) f^{a}_{j'}(\boldsymbol{X}) \big) \right|^{1/{a}}$. By Hölder inequality we show that for the M-learner:
\begin{equation}
    \begin{aligned}
        \big| \mathbb{E}\big( f_j(\boldsymbol{X}) f_{j'}(\boldsymbol{X}) &\psi_k^2(T,\boldsymbol{X}) \big) \big| \leq \left| \mathbb{E}\big( f^{a}_j(\boldsymbol{X}) f^{a}_{j'}(\boldsymbol{X}) \big) \right|^{1/{a}} \cdot \left| \mathbb{E}\big( \psi_k^{2b}(T,\boldsymbol{X}) \big) \right|^{1/b} \quad \text{ (Hölder)} \\
        &\leq \delta^{(a)}_{j j'} \Big( 2^{2b-1} \ \mathbb{E}\Big[ \Big(\frac{\mathbf{1}\{T = t_k\}}{r(t_k,\boldsymbol{X})} - 1\Big)^{2b}  f^{2b}(t_k, \boldsymbol{X}) \\& \qquad\qquad +  \Big(\frac{\mathbf{1}\{T = t_k\}}{r(t_k,\boldsymbol{X})} - 1\Big)^{2b} f^{2b}(t_k, \boldsymbol{X}) \Big] \Big)^{1/b} \\
        &\qquad \qquad \qquad \text{ (Lemma \ref{AppB:lemma1} with $m=2$)}\\
        &\leq 2^{(2b-1)/b} \ \delta^{(a)}_{j j'} \Big(  \mathbb{E}\Big[ 2^{2b-1} \Big(\frac{\mathbf{1}\{T = t_k\}}{r^{2b}(t,\boldsymbol{X})} + 1\Big) f^{2b}(t_k, \boldsymbol{X}) \Big] \\& \qquad\qquad\qquad + \mathbb{E}\Big[ 2^{2b-1} \Big(\frac{\mathbf{1}\{T = t_k\}}{r^{2b}(t_k,\boldsymbol{X})} + 1\Big) f^{2b}(t_k, \boldsymbol{X}) \Big] \Big)^{1/b} \ \text{ (Lemma \ref{AppB:lemma1})}\\
        &\leq 2^{2(2b-1)/b} \ \delta^{(a)}_{j j'} \Big(  \mathbb{E}\Big[ \mathbb{E}\Big(\frac{\mathbf{1}\{T = t_k\}}{r^{2b}(t,\boldsymbol{X})} + 1\Big) \mid \boldsymbol{X} \Big) f^{2b}(t_k, \boldsymbol{X}) \Big] \\& \qquad\qquad\qquad + \mathbb{E}\Big[ \mathbb{E} \Big(\frac{\mathbf{1}\{T = t_k\}}{r^{2b}(t_k,\boldsymbol{X})} + 1\Big) \mid \boldsymbol{X} \Big) f^{2b}(t_k, \boldsymbol{X}) \Big] \Big)^{1/b} \\
        &\leq 2^{2(2b-1)/b} \ \delta^{(a)}_{j j'} \Big(  \mathbb{E}\Big[ \Big(\frac{1}{r^{2b-1}(t,\boldsymbol{X})} + 1\Big) f^{2b}(t_k, \boldsymbol{X}) \Big] \\& \qquad\qquad\qquad + \mathbb{E}\Big[ \Big(\frac{1}{r^{2b-1}(t_k,\boldsymbol{X})} + 1\Big) f^{2b}(t_k, \boldsymbol{X}) \Big] \Big)^{1/b} \\
        &\leq 2^{2(2b-1)/b} \ \delta^{(a)}_{j j'}  \Big(\frac{1}{r^{2b-1}_{\mathrm{min}}} + 1\Big)^{1/b} \Big( C^{2b} + C^{2b} \Big)^{1/b} \quad \text{ (Bounding $r$ and $f$)} \\
        &\leq 2^{2(2b-1)/b} \ \delta^{(a)}_{j j'}  \Big(\frac{1}{r^{2b-1}_{\mathrm{min}}} + \frac{1}{r^{2b-1}_{\mathrm{min}}}
        \Big)^{1/b} 2^{1/b} C^{b} \\
        &\leq 2^{2(2b-1)/b} \ \delta^{(a)}_{j j'}  \frac{2^{1/b}}{r^{(2b-1)/b}_{\mathrm{min}}} 2^{1/b} C^{b} \\
        &\leq 2^4 \ \delta^{(a)}_{j j'} \frac{1}{r^{(2b-1)/b}_{\mathrm{min}}}C^b =  \frac{16}{r^{(2b-1)/b}_{\mathrm{min}}} \delta^{(a)}_{j j'} C^b.
    \end{aligned}
\end{equation}

On the other term, one obtains similarly:
\begin{equation}
    \begin{aligned}
        \big|\mathbb{E} \big(f_j(\boldsymbol{X}) f_{j'}(\boldsymbol{X}) &A^2_{t_k}(T,\boldsymbol{X}) \big) \big| \leq \left|\mathbb{E} \big(f^a_j(\boldsymbol{X}) f^a_{j'}(\boldsymbol{X}) \right|^{1/a} \cdot\left|\mathbb{E} \big( A^{2b}_{t_k}(T,\boldsymbol{X}) \big) \right|^{1/b} \quad \text{ (Hölder)} \\
        &\leq \delta^{(a)}_{j j'} \left|\mathbb{E} \big( A^{2b}_{t_k}(T,\boldsymbol{X}) \big) \right|^{1/b}  \\
        &\leq \delta^{(a)}_{j j'} \left( 2^{2b-1} \ \mathbb{E} \Big( \frac{\mathbf{1}\{T = t_k\}}{r(t_k,\boldsymbol{X})} \Big)^{2b} +  \mathbb{E} \Big( \frac{\mathbf{1}\{T = t_k \}}{r(t_k,\boldsymbol{X})}  \Big)^{2b} \right)^{1/b} \ \text{ (Lemma \ref{AppB:lemma1})} \\
        &\leq 2^{(2b-1)/b} \sigma^2 \delta^{(a)}_{j j'}  \left(\mathbb{E} \Big( \frac{\mathbf{1}\{T = t_k\}}{r^{2b}(t_k,\boldsymbol{X})} \Big) +  \mathbb{E} \Big( \frac{\mathbf{1}\{T = t_k \}}{r^{2b}(t_k,\boldsymbol{X})}  \Big) \right)^{1/b} \\
        &\leq 2^{(2b-1)/b} \sigma^2 \delta^{(a)}_{j j'}  \Big( \frac{2}{r^{2b-1}_{\mathrm{min}}}\Big)^{1/b} = \frac{4}{r^{(2b-1)/b}_{\mathrm{min}}} \sigma^2 \delta^{(a)}_{j j'}.
    \end{aligned}
\end{equation}

Thus, by combining the two terms, one gets:
\begin{equation}
    \begin{aligned}
        \left| \boldsymbol{\Sigma}^{\mathrm{(M)}}_{j j'} \right| &\leq \left| \mathbb{E} \big(f_j(\boldsymbol{X}) f_{j'}(\boldsymbol{X}) \psi^2_k(T, \boldsymbol{X})\big) \right|+ \sigma^2 \left|\mathbb{E} \big(f_j(\boldsymbol{X}) f_{j'}(\boldsymbol{X}) A^2_{t_k}(T,\boldsymbol{X}) \big) \right|\\
        &\leq \frac{16}{r^{(2b-1)/b}_{\mathrm{min}}} \delta^{(a)}_{j j'} C^b + \frac{4}{r^{(2b-1)/b}_{\mathrm{min}}} \sigma^2 \delta^{(a)}_{j j'}  \\
        &\leq \frac{1}{r^{(2b-1)/b}_{\mathrm{min}}} \big( 16 \ C^b + 4\sigma^2 \big) \delta_*^{(b)} ,
    \end{aligned}
\end{equation}
where $\delta_*^{(b)} = \max_{j,j'} \left| \mathbb{E}\big( f^{b/(b-1)}_j(\boldsymbol{X}) f^{b/(b-1)}_{j'}(\boldsymbol{X}) \big) \right|^{(b-1)/b} =  \max_{j,j'} \delta^{(a)}_{j j'}$.

Therefore,  for all $\epsilon=b-1>0$, there exists $C_{M} = 4 C + \sigma^2$ such that 
\begin{equation}
\label{eq:M_learnerDet}
    \left| \boldsymbol{\Sigma}^{\mathrm{(M)}}_{j j'} \right| \leq 4 r^{1/(1+\epsilon)-2}_{\mathrm{min}} \delta_*^{(1+\epsilon)}  C_{M}.
\end{equation}

In particular, if $\epsilon \ll 1$ then $ 1/(1+\epsilon)-2 \approx -(1+\epsilon)$ and
\begin{equation}
   \left| \boldsymbol{\Sigma}^{\mathrm{(M)}}_{j j'} \right| \leq  \frac{4}{r^{1+\epsilon}_{\mathrm{min}}} \delta_*^{(1+\epsilon)} C_{M}
\end{equation}


\subsubsection{Error estimation of the DR-learner.}

In this case, we have
\begin{align}
    &A_{t_k}(T,\boldsymbol{X}) = \frac{\mathbf{1}\{T = t_k\}}{r(t_k,\boldsymbol{X})} - \frac{\mathbf{1}\{T = t_k \}}{r(t_k,\boldsymbol{X})}, \\
    &B_{t_k}(T,\boldsymbol{X}) = \mu_{t_k}(\boldsymbol{X})- \mu_{t_k}(\boldsymbol{X}) - \left( \frac{\mathbf{1}\{T = t_k\}}{r(t_k,\boldsymbol{X})} - \frac{\mathbf{1}\{T = t_k \}}{r(t_k,\boldsymbol{X})} \right) \mu_T(\boldsymbol{X}).
\end{align}
We need just to compute the upper bound of $\mathbb{E}\big( f_j(\boldsymbol{X}) f_{j'}(\boldsymbol{X}) \psi_k^2(T,\boldsymbol{X}) \big) $ such that 
\begin{equation}
    \begin{aligned}
        &\psi_k(T,\boldsymbol{X}) =  A_{t_k}(T,\boldsymbol{X}) f(T,\boldsymbol{X}) - \tau_k(\boldsymbol{x}) + B_{t_k}(T,\boldsymbol{X}) \\
        &= \Big(\frac{\mathbf{1}\{T = t_k\}}{r(t_k,\boldsymbol{X})} - 1\Big) f(t_k, \boldsymbol{X}) - \Big(\frac{\mathbf{1}\{T = t_k\}}{r(t_k,\boldsymbol{X})} - 1\Big) f(t_k, \boldsymbol{X}) + \mu_{t_k}(\boldsymbol{X}) \left( 1 - \frac{\mathbf{1}\{T = t_k\}}{r(t_k,\boldsymbol{X})} \right) \\&
        \qquad \qquad \qquad \qquad- \mu_{t_k}(\boldsymbol{X}) \left( 1 - \frac{\mathbf{1}\{T = t_k \}}{r(t_k,\boldsymbol{X})} \right) \\
        &= \Big(\frac{\mathbf{1}\{T = t_k\}}{r(t_k,\boldsymbol{X})} - 1\Big) \big(f(t_k, \boldsymbol{X}) -  \mu_{t_k}(\boldsymbol{X})\big) - \Big(\frac{\mathbf{1}\{T = t_k\}}{r(t_k,\boldsymbol{X})} - 1\Big) \big(f(t_k, \boldsymbol{X}) -  \mu_{t_k}(\boldsymbol{X})\big)
    \end{aligned}
\end{equation}
Similarly to the previous calculus, we show that for the DR-learner 
\begin{equation}
    \begin{aligned}
        \big| \mathbb{E}\big( f_j(\boldsymbol{X}) &f_{j'}(\boldsymbol{X}) \psi_k^2(T,\boldsymbol{X}) \big) \big| \leq \left| \mathbb{E}\big( f^a_j(\boldsymbol{X}) f^a_{j'}(\boldsymbol{X}) \big) \right|^{1/a} \cdot \left| \mathbb{E}\big( \psi_k^{2b}(T,\boldsymbol{X}) \big) \right|^{1/b} \quad \text{ (Hölder)} \\
        &\leq \delta^{(a)}_{j j'} \Big( 2^{2b-1} \ \mathbb{E}\Big[ \Big(\frac{\mathbf{1}\{T = t_k\}}{r(t_k,\boldsymbol{X})} - 1\Big)^{2b} \big( f(t_k, \boldsymbol{X}) - \mu_{t_k}(\boldsymbol{X}) \big)^{2b} \\& \qquad\qquad + \Big(\frac{\mathbf{1}\{T = t_k\}}{r(t_k,\boldsymbol{X})} - 1\Big)^{2b} \big( f(t_k, \boldsymbol{X}) - \mu_{t_k}(\boldsymbol{X})\big)^{2b} \Big] \Big)^{1/b} \quad \text{ (Lemma \ref{AppB:lemma1})} \\
        &\leq 2^{(2b-1)/b} \ \delta^{(a)}_{j j'} \Big(  \mathbb{E}\Big[ \Big(\frac{\mathbf{1}\{T = t_k\}}{r(t_k,\boldsymbol{X})} - 1\Big)^{2b} \big( f(t_k, \boldsymbol{X}) - \mu_{t_k}(\boldsymbol{X}) \big)^{2b}  \Big] \\& \qquad\qquad\qquad + \mathbb{E}\Big[\Big(\frac{\mathbf{1}\{T = t_k\}}{r(t_k,\boldsymbol{X})} - 1\Big)^{2b} \big( f(t_k, \boldsymbol{X}) - \mu_{t_k}(\boldsymbol{X})\big)^{2b} \Big] \Big)^{1/b} \\
        &\leq  2^{(2b-1)/b} \ \delta^{(a)}_{j j'} \Big(  \mathbb{E}\Big[ 2^{2b-1} \ \Big(\frac{\mathbf{1}\{T = t_k\}}{r^{2b}(t,\boldsymbol{X})} + 1\Big) \big( f(t_k, \boldsymbol{X}) - \mu_{t_k}(\boldsymbol{X}) \big)^{2b} \Big] \\& \qquad\qquad+ \mathbb{E}\Big[2^{2b-1} \ \Big(\frac{\mathbf{1}\{T = t_k\}}{r^{2b}(t_k,\boldsymbol{X})} + 1\Big) \big( f(t_k, \boldsymbol{X}) - \mu_{t_k}(\boldsymbol{X})\big)^{2b} \Big] \Big)^{1/b}  \text{ (Lemma \ref{AppB:lemma1})} \\
        &\leq 2^{2(2b-1)/b} \ \delta^{(a)}_{j j'} \Big(  \mathbb{E}\Big[ \ \Big(\frac{1}{r^{2b-1}(t,\boldsymbol{X})} + 1\Big) \big( f(t_k, \boldsymbol{X}) - \mu_{t_k}(\boldsymbol{X}) \big)^{2b}  \Big] \\& \qquad\qquad\qquad + \mathbb{E}\Big[ \Big(\frac{1}{r^{2b-1}(t_k,\boldsymbol{X})} + 1\Big) \big( f(t_k, \boldsymbol{X}) - \mu_{t_k}(\boldsymbol{X})\big)^{2b} \Big] \Big)^{1/b} \\
        &\leq 2^{2(2b-1)/b}  \ \delta^{(a)}_{j j'} \Big( \frac{1}{r^{(2b-1)/b}_{\mathrm{min}}} + 1 \Big) \Big( \mathbb{E}\Big[ \ \big( f(t_k, \boldsymbol{X}) - \mu_{t_k}(\boldsymbol{X}) \big)^{2b}  \Big] \\& \qquad\qquad\qquad + \mathbb{E}\Big[ \big( f(t_k, \boldsymbol{X}) - \mu_{t_k}(\boldsymbol{X})\big)^{2b} \Big] \Big)^{1/b}\\
        &\leq 2^{2(2b-1)/b} \ \delta^{(a)}_{j j'} \Big( \frac{1}{r^{(2b-1)/b}_{\mathrm{min}}} + 1 \Big)   \Big[  \Big( \mathbb{E}\big( f(t_k, \boldsymbol{X}) - \mu_{t_k}(\boldsymbol{X}) \big)^{2b}  \Big)^{1/b} \\& \qquad\qquad\qquad + \mathbb{E}\big( f(t_k, \boldsymbol{X}) - \mu_{t_k}(\boldsymbol{X})\big)^{2b} \Big)^{1/b}  \Big] \quad \text{ (Subadditivity of $\mid \boldsymbol{X}|^{1/b}$)}  \\
    \end{aligned}
\end{equation}

Hence, 
\begin{equation}
    \begin{aligned}
        \left| \boldsymbol{\Sigma}^{(\mathrm{DR})}_{j j'} \right| &\leq 2^{2(2b-1)/b} \ \delta^{(a)}_{j j'} \Big( \frac{1}{r^{(2b-1)/b}_{\mathrm{min}}} + 1 \Big)   \Big[  \Big( \mathbb{E}\big( f(t_k, \boldsymbol{X}) - \mu_{t_k}(\boldsymbol{X}) \big)^{2b} \Big)^{1/b} \\& \qquad\qquad\qquad  + \Big(\mathbb{E}\big( f(t_k, \boldsymbol{X}) - \mu_{t_k}(\boldsymbol{X})\big)^{2b} \Big)^{1/b}  \Big]  + \frac{4}{r^{(2b-1)/b}_{\mathrm{min}}} \sigma^2 \delta^{(a)}_{j j'}  \\
        &\leq 2^{2(2b-1)/b} \ \delta_*^{(b)} \Big( \frac{1}{r^{(2b-1)/b}_{\mathrm{min}}} + 1 \Big)   \Big[  \Big( \mathbb{E}\big( f(t_k, \boldsymbol{X}) - \mu_{t_k}(\boldsymbol{X}) \big)^{2b} \Big)^{1/b} \\& \qquad\qquad\qquad  + \Big(\mathbb{E}\big( f(t_k, \boldsymbol{X}) - \mu_{t_k}(\boldsymbol{X})\big)^{2b} \Big)^{1/b}  \Big]  + \frac{4}{r^{(2b-1)/b}_{\mathrm{min}}} \sigma^2 \delta_*^{(b)}  \\
    \end{aligned}
\end{equation}

We consider now $\epsilon=b-1>0$, and we assume that $\epsilon \ll 1$, then
\begin{equation}
    \begin{aligned}
        &2^{2(2b-1)/b} \ \delta_*^{(b)}  \Big( \frac{1}{r^{(2b-1)/b}_{\mathrm{min}}} + 1 \Big) \Big[  \Big( \mathbb{E}\big( f(t_k, \boldsymbol{X}) - \mu_{t_k}(\boldsymbol{X}) \big)^{2b} \Big)^{1/b}  + \Big(\mathbb{E}\big( f(t_k, \boldsymbol{X}) - \mu_{t_k}(\boldsymbol{X})\big)^{2b} \Big)^{1/b}  \Big] \\
        &  + \frac{4}{r^{(2b-1)/b}_{\mathrm{min}}} \sigma^2 \delta_*^{(b)} \approx 4 \ \delta_*^{(1+\epsilon)} \Big( \frac{1}{r^{1+\epsilon}_{\mathrm{min}}} + 1 \Big) \Big( \mathbb{E} \big( f(t_k, \boldsymbol{X}) - \mu_{t_k}(\boldsymbol{X}) \big)^2 + \mathbb{E} \big( f(t_k, \boldsymbol{X}) - \mu_{t_k}(\boldsymbol{X})\big)^2  \Big) \\ & \qquad\qquad\qquad\qquad\qquad + 4 \ \sigma^2 \delta_*^{(1+\epsilon)} \frac{1}{r^{1+\epsilon}_{\mathrm{min}}}.
    \end{aligned} 
\end{equation}

Consequently,
\begin{equation}
\label{eq:DR_learnerDet}
   \left| \boldsymbol{\Sigma}^{\mathrm{(DR)}}_{j j'} \right| \leq 4 \Big( \frac{C^*_{DR} + \sigma^2}{r^{1+\epsilon}_{\mathrm{min}}} + C^*_{DR} \Big) \delta_*^{(1+\epsilon)},
\end{equation}
where $ C^*_{DR}= \mathbb{E} \big( f(t_k, \boldsymbol{X}) - \mu_{t_k}(\boldsymbol{X}) \big)^2 + \mathbb{E}\big( f(t_k, \boldsymbol{X}) - \mu_{t_k}(\boldsymbol{X})\big)^2 = \mathrm{err}({\mu}_{t_k}) + \mathrm{err}({\mu}_{t_k})$.

\subsubsection{Error estimation of the X-learner.} 

In this case, we have
\begin{align}
    &A_{t_k}(T,\boldsymbol{X}) = 2 \times \mathbf{1}\{T = t_k\} - 1, \\
    &B_{t_k}(T,\boldsymbol{X}) = (1-\mathbf{1}\{T = t_k\})\mu_{t_k}(\boldsymbol{X})- \mu_{t_k}(\boldsymbol{X}) + \sum_{l \neq k} \mathbf{1}\{T = t_l\}\mu_{t_l}(\boldsymbol{X}).
\end{align}
 One can write $\psi_k$ as
\begin{equation}
    \begin{aligned}
        \psi_k(T,\boldsymbol{X}) &=  A_{t_k}(T,\boldsymbol{X}) f(T,\boldsymbol{X}) - \tau_k(\boldsymbol{x}) + B_{t_k}(T,\boldsymbol{X}) \\
        &= \big(  2 \ \mathbf{1}\{T = t_k\} - 1 \big)  f(T,\boldsymbol{X}) - (f(t_k,\boldsymbol{X}) - f(t_k,\boldsymbol{X})) + \big(1-\mathbf{1}\{T = t_k\}\big)\\& \qquad \qquad \qquad \qquad \mu_{t_k}(\boldsymbol{X}) - \mu_{t_k}(\boldsymbol{X}) + \sum_{l \neq k} \mathbf{1}\{T = t_l\}\mu_{t_l}(\boldsymbol{X}) \\
        &= \big( 1 - \mathbf{1}\{T = t_k\}   \big) ( \mu_{t_k}(\boldsymbol{X}) - f(t_k,\boldsymbol{X}) ) - (\mu_{t_k}(\boldsymbol{X}) - f(t_k,\boldsymbol{X})) \\&  \qquad\qquad\qquad+ \sum_{l \neq k} \mathbf{1}\{T = t_l\} \big( \mu_{t_l}(\boldsymbol{X}) - f(t_l,\boldsymbol{X}) \big) = a_k + \sum_{l \neq k} b_{l}.
    \end{aligned}
\end{equation}
where 
\begin{gather}
    a_k = \big( 1 - \mathbf{1}\{T = t_k\}   \big) ( \mu_{t_k}(\boldsymbol{X}) - f(t_k,\boldsymbol{X}) ) - (\mu_{t_k}(\boldsymbol{X}) - f(t_k,\boldsymbol{X})), \\ b_l =  \mathbf{1}\{T = t_l\} \big( \mu_{t_l}(\boldsymbol{X}) - f(t_l,\boldsymbol{X}) \big).
\end{gather}

Similarly to the M- and DR-learners calculus, and using lemma \ref{AppB:lemma1}:
\begin{equation}
    \begin{aligned}
        \big| \mathbb{E}\big( f_j(\boldsymbol{X}) &f_{j'}(\boldsymbol{X}) \psi_k^2(T,\boldsymbol{X}) \big) \big| \leq \left| \mathbb{E}\big( f^a_j(\boldsymbol{X}) f^a_{j'}(\boldsymbol{X}) \big) \right|^{1/a} \cdot \left| \mathbb{E}\big( \psi_k^{2b}(T,\boldsymbol{X}) \big) \right|^{1/b} \\ 
        &\leq \delta^{(a)}_{j j'} \Big| \mathbb{E}\big( a_t + \sum_{l \neq k} b_{l} \big)^{2b} \Big|^{1/b} \quad \text{ (Hölder)} \\
        &\leq \delta^{(a)}_{j j'} \Big( 2^{2b-1} \Big( \mathbb{E}\big( a^{2b}_t\big) + \mathbb{E}\big(\sum_{l \neq k} b_{l} \big)^{2b} \Big) \Big)^{1/b} \quad \text{ (Lemma \ref{AppB:lemma1} with $m=2$)} \\
        &\leq 2^{(2b-1)/b} \ \delta^{(a)}_{j j'}  \Big( \mathbb{E}\big( a^{2b}_t\big) + \mathbb{E}\big(\sum_{l \neq k} b_{l} \big)^{2b}  \Big)^{1/b} \\
        &\leq 2^{(2b-1)/b} \ \delta^{(a)}_{j j'}  \Big[ 2^{2b-1} \Big( \mathbb{E}\big(\big( 1 - \mathbf{1}\{T = t_k\}  \big)^{2b} ( \mu_{t_k}(\boldsymbol{X}) - f(t_k,\boldsymbol{X}) \big)^{2b} \big) \\ & \qquad\qquad + \mathbb{E}\big(\mu_{t_k}(\boldsymbol{X})  - f(t_k,\boldsymbol{X})\big)^{2b} \Big)  + (K-1)^{2b-1} \\ & \qquad\qquad \times \sum_{l \neq k} \mathbb{E}\big( \mathbf{1}\{T = t_l\} \big( \mu_{t_l}(\boldsymbol{X}) - f(t_l,\boldsymbol{X}) \big)^{2b} \Big]^{1/b}  \\ & \text{ (Lemma \ref{AppB:lemma1} with $m=2$ on the $1^{\text{st}}$ term, and $m=(K-1)$ on the $2^{\text{nd}}$ term)} \\
        &\leq 2^{(2b-1)/b} \ \delta^{(a)}_{j j'}  \Big[ 2^{2b-1} \Big( \mathbb{E}\big( \mu_{t_k}(\boldsymbol{X}) - f(t_k,\boldsymbol{X}) \big)^{2b} + \mathbb{E}\big(\mu_{t_k}(\boldsymbol{X}) - f(t_k,\boldsymbol{X})\big)^{2b} \Big) \\& \qquad \qquad \qquad + (K-1)^{2b-1} \sum_{l \neq k} \mathbb{E}\big( \mu_{t_l}(\boldsymbol{X}) - f(t_l,\boldsymbol{X}) \big)^{2b} \Big]^{1/b} \\ 
        &\leq 2^{(2b-1)/b} \ \delta^{(a)}_{j j'}  \Big[ 2^{(2b-1)/b} \Big( \mathbb{E}\big( \mu_{t_k}(\boldsymbol{X}) - f(t_k,\boldsymbol{X}) \big)^{2b}\Big)^{1/b} + 2^{(2b-1)/b} \Big( \mathbb{E}\big(\mu_{t_k}(\boldsymbol{X})\\& \qquad - f(t_k,\boldsymbol{X})\big)^{2b} \Big)^{1/b} + (K-1)^{(2b-1)/b} \sum_{l \neq k} \Big(\mathbb{E}\big( \mu_{t_l}(\boldsymbol{X}) - f(t_l,\boldsymbol{X}) \big)^{2b} \Big)^{1/b} \Big] \\
        &\leq 2^{2(2b-1)/b} \ \delta^{(a)}_{j j'}  \Big[ \Big( \mathbb{E}\big( \mu_{t_k}(\boldsymbol{X}) - f(t_k,\boldsymbol{X}) \big)^{2b}\Big)^{1/b} + \Big( \mathbb{E}\big(\mu_{t_k}(\boldsymbol{X}) - f(t_k,\boldsymbol{X})\big)^{2b} \Big)^{1/b} \\& \qquad + \big(\frac{K-1}{2}\big)^{(2b-1)/b} \sum_{l \neq k} \Big(\mathbb{E}\big( \mu_{t_l}(\boldsymbol{X}) - f(t_l,\boldsymbol{X}) \big)^{2b} \Big)^{1/b} \Big]. 
    \end{aligned}
\end{equation}

Given that $\mathbb{E}\big( f_j(\boldsymbol{X}) f_{j'}(\boldsymbol{X}) A^2_{t_k}(T,\boldsymbol{X}) \big) =  \mathbb{E}\big( f_j(\boldsymbol{X}) f_{j'}(\boldsymbol{X}) \big) = \delta^{(1)}_{j j'}$, we deduce finally
\begin{equation}
    \begin{aligned}
    \left| \boldsymbol{\Sigma}^{\mathrm{(X)}}_{j j'} \right| &\leq \left| \mathbb{E} \big(f_j(\boldsymbol{X}) f_{j'}(\boldsymbol{X}) \psi^2_k(T, \boldsymbol{X})\big) \right|+ \sigma^2 \left|\mathbb{E} \big(f_j(\boldsymbol{X}) f_{j'}(\boldsymbol{X}) A^2_{t_k}(T,\boldsymbol{X}) \big) \right| \\
    &\leq 2^{2(2b-1)/b} \ \delta^{(a)}_{j j'}  \Big[ \Big( \mathbb{E}\big( \mu_{t_k}(\boldsymbol{X}) - f(t_k,\boldsymbol{X}) \big)^{2b}\Big)^{1/b} + \Big( \mathbb{E}\big(\mu_{t_k}(\boldsymbol{X}) - f(t_k,\boldsymbol{X})\big)^{2b} \Big)^{1/b} \\& \qquad + \big(\frac{K-1}{2}\big)^{(2b-1)/b} \sum_{l \neq k} \Big(\mathbb{E}\big( \mu_{t_l}(\boldsymbol{X}) - f(t_l,\boldsymbol{X}) \big)^{2b} \Big)^{1/b} \Big] + \sigma^2 \delta^{(1)}_{j j'} \\
    &\leq 2^{2(2b-1)/b} \ \delta^{(b)}_{*}  \Big[ \Big( \mathbb{E}\big( \mu_{t_k}(\boldsymbol{X}) - f(t_k,\boldsymbol{X}) \big)^{2b}\Big)^{1/b} + \Big( \mathbb{E}\big(\mu_{t_k}(\boldsymbol{X}) - f(t_k,\boldsymbol{X})\big)^{2b} \Big)^{1/b} \\& \qquad + \big(\frac{K-1}{2}\big)^{(2b-1)/b} \sum_{l \neq k} \Big(\mathbb{E}\big( \mu_{t_l}(\boldsymbol{X}) - f(t_l,\boldsymbol{X}) \big)^{2b} \Big)^{1/b} \Big] + \sigma^2 \delta^{(1)}_{*} \\
    \end{aligned}
\end{equation}
where $\delta^{(1)}_{*} = \max_{j, j'} \mathbb{E}\big( f_j(\boldsymbol{X}) f_{j'}(\boldsymbol{X}) \big)$.

As in the previous cases, we consider now $\epsilon=b-1>0$ with $\epsilon \ll 1$, then
\begin{equation}
    \begin{aligned}
        &2^{2(2b-1)/b} \ \delta^{(b)}_{*}  \Big[ \Big( \mathbb{E}\big( \mu_{t_k}(\boldsymbol{X}) - f(t_k,\boldsymbol{X}) \big)^{2b}\Big)^{1/b} + \Big( \mathbb{E}\big(\mu_{t_k}(\boldsymbol{X}) - f(t_k,\boldsymbol{X})\big)^{2b} \Big)^{1/b} \\ & \qquad\qquad\qquad + \big(\frac{K-1}{2}\big)^{(2b-1)/b} \sum_{l \neq k} \Big(\mathbb{E}\big( \mu_{t_l}(\boldsymbol{X}) - f(t_l,\boldsymbol{X}) \big)^{2b} \Big)^{1/b} \Big] + \sigma^2 \delta^{(1)}_{*}  \\
        & \qquad\quad \approx 4 \ \delta_*^{(1+\epsilon)} \Big( \mathbb{E}  \big( f(t_k, \boldsymbol{X}) - \mu_{t_k}(\boldsymbol{X}) \big)^2  + \mathbb{E} \big( f(t_k, \boldsymbol{X}) - \mu_{t_k}(\boldsymbol{X})\big)^2   \\ & \qquad\qquad\qquad\quad + \frac{(K-1)^2}{4} \sum_{l \neq k} \mathbb{E} \big( \mu_{t_l}(\boldsymbol{X}) - f(t_l,\boldsymbol{X}) \big)^{2}  + \sigma^2 \delta^{(1)}_{*}.
    \end{aligned} 
\end{equation}

Therefore,
\begin{equation}
\label{eq:X_learnerDet}
   \left| \boldsymbol{\Sigma}^{\mathrm{(X)}}_{j j'} \right| \leq  4 \delta_*^{(1+\epsilon)} C_{X} + \sigma^2 \delta^{(1)}_{*}.
\end{equation}
where $ C_{X} = \mathrm{err}({\mu}_{t_k}) + \mathrm{err}({\mu}_{t_k}) + \frac{(K-1)^2}{4}\sum_{l \neq k}\mathrm{err}({\mu}_{t_l}) .$

\subsection{Error estimation of the T- and naive X-learners.}
\label{App:T_Xnv_learners}

\vspace{0.1in}

In this subsection, we propose to conduct the bias-variance analysis of the T-learner and the naive extension of the X-learner. Some steps of this proof are quite similar to the proof of Appendix \ref{App:pseudo_outcome}.

\subsubsection{Error estimation of the T-learner.}

\vspace{0.1in}

\subsubsection*{Step 0. Set-up}

For all $t \in \mathcal{T}$, we define the set $\mathbf{S}_t=\{i, T_i=t \}$ with $n_t = \big| \mathbf{S}_t \big|$. Under Assumptions (\ref{assump:unconfound}-\ref{assump:beta_Y}), the T-learner of the CATE can be defined as 
\begin{equation}
    {\widehat{\tau}}^{\mathrm{(T)}}_{k}(\boldsymbol{x}) = \boldsymbol{f}(\boldsymbol{x})^{\top} {\widehat{\boldsymbol{\beta}}}_{t_k} - \boldsymbol{f}(\boldsymbol{x})^{\top} {\widehat{\boldsymbol{\beta}}}_{t_0} =  \boldsymbol{f}(\boldsymbol{x})^{\top} ({\widehat{\boldsymbol{\beta}}}_{t_k} - {\widehat{\boldsymbol{\beta}}}_{t_0}),
\end{equation}
where ${\widehat{\boldsymbol{\beta}}}_{t_k}$ and ${\widehat{\boldsymbol{\beta}}}_{t_0}$ are the OLS estimators of ${\boldsymbol{\beta}}_{t_k}$ and ${\boldsymbol{\beta}}_{t_0}$ such that:
\begin{align}
    &\widehat{\boldsymbol{\beta}}_{t_k}= \big(\mathbf{H}_k^{\top}\mathbf{H}_k\big)^{-1} \mathbf{H}_k^{\top} \boldsymbol{y}_k, \\
    &\widehat{\boldsymbol{\beta}}_{t_0}= \big(\mathbf{H}_0^{\top}\mathbf{H}_0\big)^{-1} \mathbf{H}_0^{\top} \boldsymbol{y}_0,
\end{align}
where $\mathbf{H}_k = (f_j(\boldsymbol{X}_i))_{i \in \mathbf{S}_{t_k}, j} \in \mathbb{R}^{n_{k} \times p}$ (respectively, $\mathbf{H}_0 = (f_j(\boldsymbol{X}_i))_{i \in \mathbf{S}_{t_0}, j} \in \mathbb{R}^{n_{0} \times p}$) is the regression matrix and $\boldsymbol{y}_k = ( Y_{\mathrm{obs},i})_{i \in \mathbf{S}_{t_k}}$ (respectively, $\boldsymbol{y}_0 = ( Y_{\mathrm{obs},i})_{i \in \mathbf{S}_{t_0}}$).

\subsubsection*{Step 1. Identification of $\widehat{\boldsymbol{\beta}}_k$}

By similar calculus, we show that:
\begin{equation}
    \begin{aligned}
        \widehat{\boldsymbol{\beta}}_{t_k} &= \big(\mathbf{H}_k^{\top}\mathbf{H}_k\big)^{-1} \mathbf{H}_k^{\top} \boldsymbol{y}_k \\
        &= \big(\mathbf{H}_k^{\top}\mathbf{H}_k\big)^{-1} \mathbf{H}_k^{\top} \big( f(t_k, \boldsymbol{X}_i) + \varepsilon_i(t_k) \big)_{i \in \mathbf{S}_{t_k} } \\
        &= \boldsymbol{\beta}_{t_k} + \big(\mathbf{H}_k^{\top}\mathbf{H}_k\big)^{-1} \mathbf{H}_k^{\top}  \boldsymbol{\epsilon}_k \\
        &= \boldsymbol{\beta}_{t_k} + \frac{1}{\sqrt{n_k}} \Big(\frac{1}{n_k} \mathbf{H}_k^{\top}\mathbf{H}_k \Big)^{-1} \Big( \frac{1}{\sqrt{n_k}} \mathbf{H}_k^{\top} \boldsymbol{\epsilon}_k \Big).
    \end{aligned}
\end{equation}
where $\boldsymbol{\epsilon}_k = (\varepsilon_i(t_k))_{i=1}^n$ are i.i.d. Gaussian $\mathcal{N}(0, \sigma^2)$ and independent of $(T_i,\boldsymbol{X}_i)_{i=1}^n$.

Thus,
\begin{equation}
    \begin{aligned}
        \sqrt{n} \big( \widehat{\boldsymbol{\beta}}_{t_k} - \boldsymbol{\beta}_{t_k} \big) &= \sqrt{ \frac{n}{n_k} } \Big(\frac{1}{n_k} \mathbf{H}_k^{\top}\mathbf{H}_k \Big)^{-1} \Big( \frac{1}{\sqrt{n_k}} \mathbf{H}_k^{\top} \boldsymbol{\epsilon}_k \Big).
    \end{aligned}
\end{equation}

By similar calculus
\begin{equation}
    \begin{aligned}
        \sqrt{n} \big( \widehat{\boldsymbol{\beta}}_{t_0} - \boldsymbol{\beta}_{t_0} \big) &= \sqrt{ \frac{n}{n_0} } \Big(\frac{1}{n_0} \mathbf{H}_0^{\top}\mathbf{H}_0 \Big)^{-1} \Big( \frac{1}{\sqrt{n_0}} \mathbf{H}_0^{\top} \boldsymbol{\epsilon}_0 \Big).
    \end{aligned}
\end{equation}

Therefore, by considering ${\widehat{\boldsymbol{\beta}}}_k = \widehat{\boldsymbol{\beta}}_{t_k} - \widehat{\boldsymbol{\beta}}_{t_0}$,
\begin{equation}
\label{eq:T_learn_iden}
    \begin{aligned}
        \sqrt{n} \big( {\widehat{\boldsymbol{\beta}}}_k - {\boldsymbol{\beta}}^*_k \big) &= \sqrt{n} \big( \widehat{\boldsymbol{\beta}}_{t_k} - \boldsymbol{\beta}_{t_k} \big) + \sqrt{n} \big( \widehat{\boldsymbol{\beta}}_{t_0} - \boldsymbol{\beta}_{t_0} \big) \\
        &= \sqrt{ \frac{n}{n_k} } \Big(\frac{1}{n_k} \mathbf{H}_k^{\top}\mathbf{H}_k \Big)^{-1} \Big( \frac{1}{\sqrt{n_k}} \mathbf{H}_k^{\top} \boldsymbol{\epsilon}_k \Big) + \sqrt{ \frac{n}{n_0} } \Big(\frac{1}{n_0} \mathbf{H}_0^{\top}\mathbf{H}_0 \Big)^{-1} \Big( \frac{1}{\sqrt{n_0}} \mathbf{H}_0^{\top} \boldsymbol{\epsilon}_0 \Big).
    \end{aligned}
\end{equation}

\subsubsection*{Step 2. The asymptotic behaviour of the OLS estimator's mean and covariance}

Let $\boldsymbol{a}=(\boldsymbol{a}_k, \boldsymbol{a}_0) \in \mathbb{R}^{2p}$ and let $\phi_n$ denote the characteristic function of the vector $\big( \frac{1}{\sqrt{n_k}} \mathbf{H}_k^{\top} \boldsymbol{\epsilon}_k, \frac{1}{\sqrt{n_0}} \mathbf{H}_0^{\top} \boldsymbol{\epsilon}_0 \big)$. We have
\begin{equation}
    \begin{aligned}
        \phi_n( \boldsymbol{a} ) &= \mathbb{E} \Big[ \exp{ i \boldsymbol{a}^{\top} \big( \frac{1}{\sqrt{n_k}} \mathbf{H}_k^{\top} \boldsymbol{\epsilon}_k, \frac{1}{\sqrt{n_0}} \mathbf{H}_0^{\top} \boldsymbol{\epsilon}_0 \big) } \Big] \\
        &= \mathbb{E} \Big[ \exp{ i
        \Big( \boldsymbol{a}_k^{\top} \frac{1}{\sqrt{n_k}} \mathbf{H}_k^{\top} \boldsymbol{\epsilon}_k + \boldsymbol{a}_0^{\top}  \frac{1}{\sqrt{n_0}} \mathbf{H}_0^{\top} \boldsymbol{\epsilon}_0 \Big) } \Big] \\
        &= \mathbb{E} \Big[ \exp{ i \Big( \frac{1}{\sqrt{n_k}} \sum_{m=1}^{n_k} \boldsymbol{a}_k^{\top}  \big( H_{m j} \varepsilon_m(t_k) \big)_{j=0}^{p-1} + \frac{1}{\sqrt{n_0}} \sum_{m=1}^{n_0} \boldsymbol{a}_0^{\top} \big( H_{m j} \varepsilon_m(t_0)\big)_{j=0}^{p-1} \Big)  } \Big] \\
        &= \mathbb{E} \Big[ \exp{ i \Big( \frac{1}{\sqrt{n}} \sum_{m=1}^{n} \boldsymbol{a}_k^{\top}  \big( H_{m j} \big)_{j=0}^{p-1} \varepsilon_m(t_k) \mathbf{1} \{T_m = t_k \}  \times \frac{\sqrt{n}}{\sqrt{ n_k } } \Big) }  \\  & \qquad\qquad \ + \frac{1}{\sqrt{n}} \sum_{m=1}^{n} \boldsymbol{a}_0^{\top}  \big( H_{m j} \big)_{j=0}^{p-1}  \varepsilon_m(t_0) \mathbf{1} \{T_m = t_0 \}  \times \frac{\sqrt{n}}{\sqrt{ n_0 } } \Big)  \Big].
    \end{aligned}
\end{equation}

Now, let us consider the vector $\boldsymbol{Z}^{(n)} = \big(\boldsymbol{Z}_k^{(n)}, \boldsymbol{Z}_0^{(n)}\big)   \in \mathbb{R}^{2p}$ such that
\begin{equation}
    \begin{aligned}
        \boldsymbol{Z}^{(n)} 
        &= \Big( \frac{1}{n} \sum_{m=1}^n H_{m 1} \varepsilon_m(t_k) \mathbf{1} \{T_m = t_k \}, \ldots, \frac{1}{n} \sum_{m=1}^n  H_{m p} \varepsilon_m(t_k) \mathbf{1} \{T_m = t_k \},  \\
        & \qquad \ \frac{1}{n} \sum_{m=1}^n \big( H_{m 1} \varepsilon_m(t_0) \mathbf{1} \{T_m = t_0 \}, \ldots, \frac{1}{n} \sum_{m=1}^n H_{m p} \varepsilon_m(t_0) \mathbf{1} \{T_m = t_0 \} \Big) \\
        &= \frac{1}{n} \sum_{m=1}^n \Big( H_{m 1} \varepsilon_m(t_k) \mathbf{1} \{T_m = t_k \}, \ldots, \ H_{m p} \varepsilon_m(t_k) \mathbf{1} \{T_m = t_k \}, \\ & \qquad \qquad \ \ \ H_{m 1} \varepsilon_m(t_0) \mathbf{1} \{T_m = t_0 \}, \ldots, \ H_{m p} \varepsilon_m(t_0) \mathbf{1} \{T_m = t_0 \} \Big) \\ &= \frac{1}{n} \sum_{m=1}^n \boldsymbol{Z}_{m}.
    \end{aligned}
\end{equation}

The mean $\boldsymbol{m}= (\boldsymbol{m}_k, \boldsymbol{m}_0)$ of the vector $\boldsymbol{Z}_{m}$ satisfies, for $j=0,\ldots,p-1$,
\begin{gather}
        {m}_{k,j} = \mathbb{E}\big[f_j(\boldsymbol{X}) \varepsilon(t_k) \mathbf{1} \{T = t_k \}\big] = 0, \\
        m_{0,j} = \mathbb{E}\big[f_j(\boldsymbol{X}) \varepsilon(t_0) \mathbf{1} \{T = t_0 \} \big] = 0,
\end{gather} 
and its covariance matrix that satisfies, for $j,j'=1,\ldots,2p$,
\begin{equation}
    \begin{aligned}
    \operatorname{Cov} \Big( \boldsymbol{Z}_{m,j}, \boldsymbol{Z}_{m,j'}\big)&= \left\{ \begin{array}{l l}
        \mathbb{E} \big[f_{j-1}(\boldsymbol{X}) f_{j'-1}(\boldsymbol{X}) \varepsilon^2(t_k) \mathbf{1} \{T = t_k \} \big] & \quad \text{if $j,j' \in \{1,\ldots,p\}$}\\[.15cm]
        \mathbb{E} \big[f_{j-1}(\boldsymbol{X}) f_{j'-1}(\boldsymbol{X}) \varepsilon^2(t_0) \mathbf{1} \{T = t_0 \} \big] & \quad \text{if $j,j' \in \{p+1,\ldots,2p\}$}\\[.15cm]
        \mathbb{E} \big[f_{j-1}(\boldsymbol{X}) f_{j'-1}(\boldsymbol{X}) \varepsilon(t_k) \varepsilon(t_0) \mathbf{1} \{T = t_k \} \mathbf{1} \{T = t_0 \} \big]  & \quad \text{otherwise.}\\[.15cm]
        \end{array} \right. \\
    &= \left\{ \begin{array}{l l}
        \sigma^2\mathbb{E} \big[ f_{j-1}(\boldsymbol{X}) f_{j'-1}(\boldsymbol{X}) \mathbf{1} \{T = t_k \}  \big] & \quad \text{if $j,j' \in \{1,\ldots,p\}$ }\\[.15cm]
        \sigma^2\mathbb{E} \big[  f_{j-1}(\boldsymbol{X}) f_{j'-1}(\boldsymbol{X}) \mathbf{1} \{T = t_0 \}  \big]  & \quad \text{if $j,j' \in \{p+1,\ldots,2p\}$}\\[.15cm]
        0  & \quad \text{otherwise,}\\
        \end{array} \right. \\
    &= \left\{ \begin{array}{l l}
        \sigma^2 \rho(t_k) F_{k, j j'} \quad \text{if $j,j' \in \{1,\ldots,p\},$}\\[.15cm]
        \sigma^2 \rho(t_0) F_{0, j j'}  \quad \text{if $j,j' \in \{p+1,\ldots,2p\}$,}\\[.15cm]
        0  \qquad\qquad\qquad \text{otherwise,}\\
         \end{array} \right.
    \end{aligned} 
\end{equation}
where the matrices $\mathbf{F}_k = \big( \mathbb{E} \big[f_{j-1}(\boldsymbol{X}) f_{j'-1}(\boldsymbol{X}) \mid T=t_k \big] \big)_{j,j'} \in \mathbb{R}^p$ and $\mathbf{F}_0 = \big( \mathbb{E} \big[f_{j-1}(\boldsymbol{X}) f_{j'-1}(\boldsymbol{X}) \mid T=t_0 \big] \big)_{j,j'} \in \mathbb{R}^p$ are supposed to be invertible. Note that, for a integrable function $h$, $\mathbb{E} \big[h(\boldsymbol{X}) \mathbf{1} \{T = t \} \big] = \mathbb{P}(T=t) \mathbb{E} \big[h(\boldsymbol{X}) \mid T = t \big]$.

Therefore, using the multivariate CLT on $ \boldsymbol{Z}^{(n)}$, we get
\begin{equation}
    \begin{aligned}
        \left(\begin{array}{c}{ \sqrt{n} \ \boldsymbol{a}_k^{\top} \boldsymbol{Z}_k^{(n)} } \\ { \sqrt{n} \ \boldsymbol{a}_0^{\top} \boldsymbol{Z}_0^{(n)}  } \end{array}\right) 
        &= \left(\begin{array}{c}{ \frac{1}{\sqrt{n}} \sum_{m=1}^{n} \boldsymbol{a}_k^{\top}  \big( H_{m j} \big)_{j=0}^{p-1} \varepsilon_m(t_k) \mathbf{1} \{T_m = t_k \} } \\ { \frac{1}{\sqrt{n}} \sum_{m=1}^{n} \boldsymbol{a}_0^{\top}  \big( H_{m j} \big)_{j=0}^{p-1} \varepsilon_m(t_0) \mathbf{1} \{T_m = t_0 \} } \end{array}\right) \\[.15cm]
        &\overset{\mathcal{L}}{\longrightarrow} \mathcal{N} \left( \left(\begin{array}{c}{ 0 } \\ { 0} \end{array}\right) , 
        \begin{pmatrix}
            \sigma^2 \rho(t_k) \boldsymbol{a}_k^{\top} \mathbf{F}_k  \boldsymbol{a}_k & 0\\
            0 & \sigma^2 \rho(t_0) \boldsymbol{a}_0^{\top} \mathbf{F}_0  \boldsymbol{a}_0 
        \end{pmatrix} \right).
    \end{aligned}
\end{equation}

On the other hand,
\begin{equation}
     \left(\frac{\sqrt{n}}{\sqrt{ n_k } }, \frac{\sqrt{n}}{\sqrt{ n_0 } } \right) = \left(\frac{\sqrt{n}}{\sqrt{\sum_{m=1}^{n} \mathbf{1} \{T_m = t_k \} } }, \frac{\sqrt{n}}{\sqrt{\sum_{m=1}^{n} \mathbf{1} \{T_m = t_0 \} } } \right) \overset{a.s}{\longrightarrow} \left(\frac{1}{\sqrt{\rho(t_k)}}, \frac{1}{\sqrt{\rho(t_0)}} \right),
\end{equation}
where $\rho(t) = \mathbb{P}(T=t)$.

Thus, by the Slutsky theorem:
\begin{equation}
    \begin{aligned}
        \frac{1}{\sqrt{n}} \sum_{m=1}^{n} \big( \boldsymbol{a}_k^{\top} \big( H_{m j} \big)_{j=0}^{p-1} \varepsilon_m(t_k) \mathbf{1}& \{T_m = t_k \} \frac{\sqrt{n}}{\sqrt{ n_k } } + \boldsymbol{a}_0^{\top} \big( H_{m j} \big)_{j=0}^{p-1} \varepsilon_m(t_0) \mathbf{1} \{T_m = t_0 \} \frac{\sqrt{n}}{\sqrt{ n_0 } } \big) \\ &\overset{\mathcal{L}}{\longrightarrow}  \mathcal{N}(\boldsymbol{0}, \sigma^2 \boldsymbol{a}_k^{\top} \mathbf{F}_k \boldsymbol{a}_k + \sigma^2  \boldsymbol{a}_0^{\top} \mathbf{F}_0 \boldsymbol{a}_0). 
     \end{aligned}
\end{equation}

Therefore,
\begin{equation}
    \begin{aligned}
        \phi_n( \boldsymbol{a} ) \overset{n \rightarrow +\infty}{\longrightarrow} \mathbb{E} \Big[ \exp{ i \Big(  \boldsymbol{a}_k^{\top} \sigma^2 \rho(t_k) \mathbf{F}_k \boldsymbol{a}_k + \boldsymbol{a}_0^{\top} \sigma^2 \rho(t_0) \mathbf{F}_0 \boldsymbol{a}_0 } \Big)  \Big] =  \phi_{(\boldsymbol{Z}_k, \boldsymbol{Z}_0)}( \boldsymbol{a} ),
    \end{aligned}
\end{equation}
where $\boldsymbol{Z}_k$ and $\boldsymbol{Z}_0$ are two independent zero-mean random vectors with covariance matrices $\sigma^2 \mathbf{F}_k$ and $\sigma^2 \mathbf{F}_0$ respectively. 

As shown previously in Appendix \ref{App:pseudo_outcome}, we can prove immediately that $\big( 1/n_k \ \mathbf{H}_k^{\top}\mathbf{H}_k \big)^{-1} \overset{P}{\longrightarrow} \mathbf{F}_k^{-1}$. Moreover, $n_{k}/n \overset{a.s}{\longrightarrow} \rho(t_k)$ so  $n_{k}/n \overset{P}{\longrightarrow} \rho(t_k)$. Thus
\begin{equation}
    \begin{aligned}
        \sqrt{ \frac{n}{n_k} } \Big(\frac{1}{n_k} \mathbf{H}_k^{\top}\mathbf{H}_k \Big)^{-1}
        &\overset{P}{\longrightarrow} \frac{1}{\sqrt{\rho(t_k)}} \mathbf{F}_k^{-1}.
    \end{aligned}
\end{equation}

Finally, given Equation (\ref{eq:T_learn_iden}) and using the Slutsky theorem, we get
\begin{equation}
    \begin{aligned}
        \label{eq:T_learner_dist}
            \sqrt{n} \big( {\widehat{\boldsymbol{\beta}}}_k - {\boldsymbol{\beta}}^*_k \big) &\overset{\mathcal{L}}{\longrightarrow}\mathcal{N} \big( \boldsymbol{0}, \frac{1}{\rho(t_k)} \mathbf{F}_k^{-1} \sigma^2 \mathbf{F}_k \mathbf{F}_k^{-1} + \frac{1}{\rho(t_0)} \mathbf{F}_0^{-1} \sigma^2 \mathbf{F}_0 \mathbf{F}_0^{-1} \big) \\
            &= \mathcal{N} \Big( \boldsymbol{0}, \frac{\sigma^2}{\rho(t_k)}  \mathbf{F}_k^{-1}  + \frac{\sigma^2}{\rho(t_0)} \mathbf{F}_0^{-1} \Big).
    \end{aligned}
\end{equation}

Here also, we can deduce that the asymptotic mean and  covariance matrix are of the form
\begin{equation}
    \begin{aligned}
        \mathbb{E}(\widehat{\boldsymbol{\beta}}_{k}) &=  \boldsymbol{\beta}_{t_k} - \boldsymbol{\beta}_{t_0} = {\boldsymbol{\beta}}^*_k, \\ \mathbb{V}(\widehat{\boldsymbol{\beta}}_k) &\approx \frac{1}{n}  \Big( \frac{1}{\rho(t_k)}  \mathbf{F}_k^{-1}  + \frac{1}{\rho(t_0)} \mathbf{F}_0^{-1} \Big) \sigma^2.
    \end{aligned}
\end{equation}

\subsubsection*{Step 3. Obtaining the error upper bound}

The asymptotic covariance matrix is given by the matrices $\mathbf{F}_k^{-1}$ and $\mathbf{F}_0^{-1}$. We assume that the polynomials $f_j$ are chosen to be orthonormal, and that, conditionally to $T$, their distribution is not significantly different. One can anticipate, therefore, that $\mathbf{F}_k, \mathbf{F}_0 \approx \mathbf{F}$ and easily identify the error's upper bound of the T-learner as:
\begin{equation}
\label{eq:T_upperbound}
          \frac{1}{\rho(t_k)} +  \frac{1}{\rho(t_0)}.
\end{equation}

\subsection{Error estimation of the naive X-learner.}

\vspace{0.1in}

Let $\overline{r}$ denote a fixed arbitrary estimator of the GPS (see remark 3.3) and respecting the assumption \ref{assump:overlap}, that is, $r_{\mathrm{min}} \leq \overline{r}(t,\boldsymbol{x})$. Let $\widehat{\mu}_{t_k}$ denote the estimator of ${\mu}_{t_k}$ . The model $\widehat{\mu}_{t_k}$ is trained using the sample $\mathbf{S}_{t_k}$, the OLS estimator $\widehat{\boldsymbol{\beta}}_{t_k}$ satisfies
\begin{equation}
    \begin{aligned}
        \widehat{\boldsymbol{\beta}}_{t_k} &=  \big(\mathbf{H}_k^{\top}\mathbf{H}_k\big)^{-1} \mathbf{H}_k^{\top} \boldsymbol{y}_k \\
        &= \boldsymbol{\beta}_{t_k} + \big(\mathbf{H}_k^{\top}\mathbf{H}_k \big)^{-1} \mathbf{H}_k^{\top} \boldsymbol{\epsilon}_k
    \end{aligned}
\end{equation}
where $\boldsymbol{y}_k = ( Y_{\mathrm{obs},i})_{i \in \mathbf{S}_{t_k}}$ and $\boldsymbol{\epsilon}_k = ( \varepsilon_i(t_k) )_{i \in \mathbf{S}_{t_k}}$.

Similarly, the OLS estimator of $\mu_{t_0}$ satisfies also
\begin{equation}
        \widehat{\boldsymbol{\beta}}_{t_0} 
        = \boldsymbol{\beta}_{t_0} + \big(\mathbf{H}_0^{\top}\mathbf{H}_0\big)^{-1} \mathbf{H}_0^{\top}  \boldsymbol{\epsilon}_0,
\end{equation}
where $\boldsymbol{y}_0 = ( Y_{\mathrm{obs},i})_{i \in \mathbf{S}_{t_0}}$ and $\boldsymbol{\epsilon}_0 = ( \varepsilon_i(t_0) )_{i \in \mathbf{S}_{t_0}}$.

We recall now the definition of the naive extension of the X-learner:
\begin{equation}
        {\widehat{\tau}}^{\mathrm{(X,nv)}}_{k}(\boldsymbol{x}) = \frac{\overline{r}(t_k,\boldsymbol{x})}{\overline{r}(t_k,\boldsymbol{x}) + \overline{r}(t_0,\boldsymbol{x})} {\widehat{\tau}}^{(k)}(\boldsymbol{x}) + \frac{\overline{r}(t_0,\boldsymbol{x})}{\overline{r}(t_k,\boldsymbol{x}) + \overline{r}(t_0,\boldsymbol{x})} {\widehat{\tau}}^{(0)}(\boldsymbol{x}).
\end{equation}

where the estimators ${\widehat{\tau}}^{(k)}$ and ${\widehat{\tau}}^{(0)}$ are built respectively on $\mathbf{S}_{t_k}$ and $\mathbf{S}_{t_0}$ by regressing $({D}_i^{(k)})_{i \in \mathbf{S}_{t_k}} = (Y_i(t_k) - \widehat{\mu}_{t_0}(\boldsymbol{X}_i))_{i \in \mathbf{S}_{t_k}}$ and $({D}_i^{(0)})_{i \in \mathbf{S}_{t_0}} = (\widehat{\mu}_{t_k}(\boldsymbol{X}_i) - Y_i(t_0))_{i \in \mathbf{S}_{t_0}}$ on $\boldsymbol{X}$. In the following, we denote ${\widehat{\tau}}^{(k)}(\boldsymbol{x}) = \boldsymbol{f}(\boldsymbol{x})^{\top}  \widehat{\boldsymbol{\beta}}^{(k)}$ and ${\widehat{\tau}}^{(0)}(\boldsymbol{x}) = \boldsymbol{f}(\boldsymbol{x})^{\top}  \widehat{\boldsymbol{\beta}}^{(0)}$. Here, $\widehat{\boldsymbol{\beta}}^{(k)}$ denotes the OLS estimator of ${\widehat{\tau}}^{(k)}$ and is given by:
\begin{equation}
    \begin{aligned}
        \widehat{\boldsymbol{\beta}}^{(k)} &= \big(\mathbf{H}_k^{\top}\mathbf{H}_k\big)^{-1} \mathbf{H}_k^{\top} \big (Y_{\mathrm{obs,i}} - \widehat{\mu}_{t_0}(\boldsymbol{X}_i)\big)_{i \in \mathbf{S}_{t_k}} \\
        &= \big(\mathbf{H}_k^{\top}\mathbf{H}_k\big)^{-1} \mathbf{H}_k^{\top} \big (Y_{\mathrm{obs,i}} - \boldsymbol{f}(\boldsymbol{X}_i)^{\top} \widehat{\boldsymbol{\beta}}_{t_0} \big)_{i \in \mathbf{S}_{t_k}} \\
        &= \big(\mathbf{H}_k^{\top}\mathbf{H}_k\big)^{-1} \mathbf{H}_k^{\top} \boldsymbol{y}_k - \big(\mathbf{H}_k^{\top}\mathbf{H}_k\big)^{-1} \mathbf{H}_k^{\top} \mathbf{H}_k \widehat{\boldsymbol{\beta}}_{t_0} \\
        &= \widehat{\boldsymbol{\beta}}_{t_k} -\widehat{\boldsymbol{\beta}}_{t_0} = {\widehat{\boldsymbol{\beta}}}_k,
    \end{aligned}
\end{equation}

where ${\widehat{\boldsymbol{\beta}}}_k = {\widehat{\boldsymbol{\beta}}}_{t_k} - {\widehat{\boldsymbol{\beta}}}_{t_0}$ is the T-learner OLS estimator as given in (\ref{eq:T_learner_dist}).

By similar calculus, we show that
\begin{equation}
    \begin{aligned}
        \widehat{\boldsymbol{\beta}}^{(0)} &= \big(\mathbf{H}_0^{\top}\mathbf{H}_0\big)^{-1} \mathbf{H}_0^{\top} \big (\widehat{\mu}_{t_k}(\boldsymbol{X}_i) - Y_{\mathrm{obs,i}} \big)_{i \in \mathbf{S}_{t_0}} \\
        &= \big(\mathbf{H}_0^{\top}\mathbf{H}_0\big)^{-1} \mathbf{H}_0^{\top} \big (\boldsymbol{f}(\boldsymbol{X}_i)^{\top} \widehat{\boldsymbol{\beta}}_{t_k} - Y_{\mathrm{obs,i}} \big)_{i \in \mathbf{S}_{t_0}} \\
        &=  \big(\mathbf{H}_0^{\top}\mathbf{H}_0\big)^{-1} \mathbf{H}_0^{\top} \mathbf{H}_0 \widehat{\boldsymbol{\beta}}_{t_k} - \big(\mathbf{H}_0^{\top}\mathbf{H}_0\big)^{-1} \mathbf{H}_0^{\top} \boldsymbol{y}_0 \\
        &= \widehat{\boldsymbol{\beta}}_{t_k} - \widehat{\boldsymbol{\beta}}_{t_0} = {\widehat{\boldsymbol{\beta}}}_k.
    \end{aligned}
\end{equation}

It results that
\begin{equation}
    \begin{aligned}
        {\widehat{\tau}}^{\mathrm{(X,nv)}}_{k}(\boldsymbol{x}) &= \frac{\overline{r}(t_k,\boldsymbol{x})}{\overline{r}(t_k,\boldsymbol{x}) + \overline{r}(t_0,\boldsymbol{x})} \boldsymbol{f}(\boldsymbol{x})^{\top}  \widehat{\boldsymbol{\beta}}^{(k)} + \frac{\overline{r}(t_0,\boldsymbol{x})}{\overline{r}(t_k,\boldsymbol{x}) + \overline{r}(t_0,\boldsymbol{x})} \boldsymbol{f}(\boldsymbol{x})^{\top}  \widehat{\boldsymbol{\beta}}^{(0)} \\
        &= \Big( \frac{\overline{r}(t_k,\boldsymbol{x})}{\overline{r}(t_k,\boldsymbol{x}) + \overline{r}(t_0,\boldsymbol{x})} + \frac{\overline{r}(t_0,\boldsymbol{x})}{\overline{r}(t_k,\boldsymbol{x}) + \overline{r}(t_0,\boldsymbol{x})} \Big) \boldsymbol{f}(\boldsymbol{x})^{\top}  {\widehat{\boldsymbol{\beta}}}_k = \boldsymbol{f}(\boldsymbol{x})^{\top}  {\widehat{\boldsymbol{\beta}}}_k
    \end{aligned}
\end{equation}

In the end, the naive X-learner is no more than a simple T-learner, the error's upper bound of the naive X-learner is given therefore by:
\begin{equation}
        \sigma^2 \Big( \frac{ 1 }{\rho(t_k)} + \frac{1 }{\rho(t_0)} \Big).
\end{equation}

\section{Discussion about the binarized R-learner.}
\label{App:bin_R_learner}

\vspace{0.1in}

Another alternative to R-learning to continuous treatments is proposed by \cite{Kaddour2021}. The approach considers both Assumptions \ref{assump:model_Y} and \ref{assump:beta_Y} on the outcome $Y(t)= \boldsymbol{f}(\boldsymbol{X})^{\top} \boldsymbol{\beta}_t + \varepsilon(t)$, then establishes the binarized \cite{Robinson1988} decomposition such that
\begin{equation}
\label{eq:bin_Rloss}
    Y_{\mathrm{obs}}-m(\boldsymbol{X}) = \boldsymbol{f}(\boldsymbol{X})^{\top}( \boldsymbol{\beta}_T-e^{\boldsymbol{\beta}}(\boldsymbol{X})) + \epsilon,
\end{equation}
where $\epsilon = \varepsilon(T)$, $m(\boldsymbol{x})=\mathbb{E}(Y_{\mathrm{obs}} \mid \boldsymbol{X} = \boldsymbol{x})$ and $e^{\boldsymbol{\beta}}(\boldsymbol{x}) = \mathbb{E}(\boldsymbol{\beta}_T \mid \boldsymbol{X} = \boldsymbol{x})$.

Considering the mean squared error of $\epsilon$ as a loss function and minimizing it allows us to identify the optimal $\widehat{\boldsymbol{\beta}}$ and therefore CATEs. Given two nuisance estimators $\widehat{m}$ and $\widehat{e}^{\boldsymbol{\beta}}$ of $m$ and ${e}^{\boldsymbol{\beta}}$, one needs to solve the following problem:
\begin{equation}
\label{eq:RLearnerCont}
    \widehat{\boldsymbol{\beta}} = \operatorname{argmin}_{ \boldsymbol{\beta} \in \mathcal{F}} \ \frac{1}{n} \sum_{i=1}^n \Big[  \left(Y_{\mathrm{obs},i} - \widehat{m}(\boldsymbol{X}_i) \right) - \boldsymbol{f}(\boldsymbol{X}_i)^{\top}\big(\boldsymbol{\beta}_{T_i} -\widehat{e}^{\boldsymbol{\beta}}(\boldsymbol{X}_i) \big) \Big]^2,
\end{equation}
where $\mathcal{F}$ is the space of candidate models $\boldsymbol{\beta}$. The previous problem corresponds to a classical OLS estimator and has, therefore, a unique solution.

If the space of candidate models $\mathcal{F}$ is separable, then the optimization problem can be divided into the following sub-problems:
\begin{gather*}
    \widehat{\boldsymbol{\beta}}_{t_0} = \operatorname{argmin} \ \frac{1}{n_k} \sum_{i \in \mathbf{S}_{t_0}} \Big[  \left(Y_{\mathrm{obs},i} - \widehat{m}(\boldsymbol{X}_i) \right) - \boldsymbol{f}(\boldsymbol{X}_i)^{\top}\big(\boldsymbol{\beta}_{t_0} -\widehat{e}^{\boldsymbol{\beta}}(\boldsymbol{X}_i) \big) \Big]^2 \\
    \vdots \\
    \widehat{\boldsymbol{\beta}}_{t_K} = \operatorname{argmin} \ \frac{1}{n_K} \sum_{i \in \mathbf{S}_{t_K}} \Big[  \left(Y_{\mathrm{obs},i} - \widehat{m}(\boldsymbol{X}_i) \right) - \boldsymbol{f}(\boldsymbol{X}_i)^{\top}\big(\boldsymbol{\beta}_{t_K} -\widehat{e}^{\boldsymbol{\beta}}(\boldsymbol{X}_i) \big) \Big]^2.
\end{gather*}

However, this approach does not consider the interactions between different $\widehat{\boldsymbol{\beta}}_{t}$ and is computationally heavy when the number of possible treatments $K$ becomes larger. It also requires specifying the family of models $\mathcal{F}$ and precise the dimension $p$ for Assumption \ref{assump:beta_Y}. 

There are two main differences between the generalized R-learner and the binarized: 1) In the binarized R-learner, $(\widehat{\boldsymbol{\beta}}_{t_k})_{k=1}^K$ may be solved separately but using a small sample ($\mathbf{S}_{t_k}$ instead of $ \mathbf{D}_{\mathrm{obs}}$); 2) The solution $(\widehat{\boldsymbol{\beta}}_{t_k})_{k=1}^K$ of the binarized R-learner is unique and is given by the OLS estimator of the binarized R-loss function.

\newpage

\section{Additional details about simulated analytical functions in section \ref{sec:6.1}}
\label{App:details_analytics}

\vspace{0.1in}

In this section, we consider a treatment $T$ with $K+1=10$ possible values in $\mathcal{T}=\{ t_k := \frac{k}{K}, k \in \{0,\ldots,K\} \}$, drawn from a uniform distribution, and the following outcome functions.

The linear model outcome for $X \in \mathbb{R}$:
\begin{equation}
\label{def:lin_model}
    Y(t) \mid X \sim \mathcal{N} \big( (1+t) X, \sigma^2\big).
\end{equation}

The multivariate hazard rate \citep{Imbens2000} outcome satisfies for $\boldsymbol{X} \in \mathbb{R}^5$:
\begin{equation}
\label{def:HR_model}
    Y(t) \mid \boldsymbol{X} \sim \mathcal{N} \big( t+\|\boldsymbol{X}\| \exp{(-t \|\boldsymbol{X}\|)} \ , \sigma^2\big).
\end{equation}

We compute the exact components of each model in the following subsections: the GPS $r$, the potential outcome models $\mu_t$ and the observed outcome model $m$.

\subsection{Computing nuisance components}
\label{App:D_1}

\vspace{0.08in}

\subsubsection*{The Generalized Propensity Score (GPS).}

In the first design (RCT), we sample $n$ units such that $T$ and $\boldsymbol{X}$ are independent. The true propensity score is known 
\begin{equation}
\label{eq:GPS_RCT}
    r(t,\boldsymbol{X})=\mathbb{P}(T=t) = 1/(K+1) \text{ for } t \in \mathcal{T}.
\end{equation}

In the second design (observational studies), we combine $K+2$ samples in a single sample of $n$ units. The first sample $\mathbf{D}_{K+1}$ contains $n_{K+1}=n/2$ units where the treatment is assigned randomly: $\boldsymbol{X}$ and $T$ are independent, $\mathbb{P}(T=t)=1/(K+1)$, $\boldsymbol{X} \sim \mathcal{N}(\boldsymbol{0},\mathbf{I}_5)$ when the hazard rate model is applied and $X \sim \mathcal{U}(0,1)$ when the linear model is applied. For $k=0,\ldots,K$, the sample $\mathbf{D}_k$ contains $n_k=n/(2(K+1))$ units and the distribution of $(\boldsymbol{X},T)$ does not respect a RCT setting. For the linear model, the joint distribution of $({X},T)$ is given by:
\begin{equation}
\label{eq:joint_lin}
    T=\frac{k}{K} \text{ and $X$ follows a uniform distribution $\mathcal{U}(I_k)$ with} \ I_k = \Big[\frac{k}{K+1},\frac{k+1}{K+1}\Big).
\end{equation}

For the hazard rate model, the joint distribution of $(\boldsymbol{X},T)$ is given by:
\begin{equation}
\label{eq:joint_HR}
\begin{aligned}
    &T=\frac{k}{K} \text{, $X_{1}$ follows a truncated standardized normal distribution on} \ I_k = \big[q_\frac{k}{K+1},q_\frac{k+1}{K+1}\big) \\
    &\qquad\qquad \text{ and $X_j$ follow a standardized normal distribution $\mathcal{N}(0,1)$ for $2 \leq j \leq 5$,}
\end{aligned}
\end{equation}
where $q_{\alpha}$ is the ${\alpha}$-quantile of the standardized normal distribution with $q_0=-\infty$ and $q_1=+\infty$ . This strategy of selecting preferentially only observations with certain characteristics is called \textit{preferential selection} sampling and creates thus a selection bias on observed data.

For all $k \in \{0,\ldots,K\}$, the true propensity score satisfies for the linear model:
\begin{equation}
\label{eq:true_GPS_lin}
    \begin{aligned}
        r(t_k, x) &= \left\{  \begin{array}{l l}
        \frac{K+2}{ 2(K+1)} & \quad \text{if $x \in I_k$,}\\[.15cm]
        \frac{1}{2(K+1)} & \quad \text{otherwise.}\\[.15cm] \end{array} \right. \\
    \end{aligned}
\end{equation}

and, for the hazard rate model, it satisfies for $\boldsymbol{x} \in \mathbb{R}^5$:
\begin{equation}
\label{eq:true_GPS_HR}
    \begin{aligned}
        r(t_k, \boldsymbol{x}) &= \left\{  \begin{array}{l l}
        \frac{K+2}{ 2(K+1)} & \quad \text{if $x_1 \in I_k$,}\\[.15cm]
        \frac{1}{2(K+1)} & \quad \text{otherwise.}\\[.15cm] \end{array} \right. \\
    \end{aligned}
\end{equation}

\begin{proof}
    We show proof for the hazard rate model with normal distribution. The proof remains the same for the linear model in a non-randomized setting.
    
    Let $A$ be a given event, and then 
    \begin{equation}
        \mathbb{P}(A) = \sum_{k=0}^{K+1} \frac{n_k}{n} \mathbb{P}_{k}(A),
    \end{equation}
    where $\mathbb{P}$ is the observed probability distribution of the combined sample and $\mathbb{P}_{k}$ denotes the probability measure induced by (\ref{eq:GPS_RCT}), (\ref{eq:joint_HR}) and the unconfoundedness assumption \ref{assump:unconfound}.
    
    Given the treatment $T=t_j$ and covariate vector $\boldsymbol{x}=(x,x_2,\ldots,x_5)$, we have
    \begin{equation}
        \begin{aligned}
        r(t_j,\boldsymbol{x}) &= \mathbb{P}(T = t_j \mid X_1 = x) \\
        &= \lim_{\delta \rightarrow 0} \mathbb{P}(T = t_j \mid X_1 \in [x,  x + \delta]) \\
        &= \lim_{\delta \rightarrow 0} \frac{\mathbb{P}\left(T=t_j, X_1 \in [x,  x + \delta] \right)}{\mathbb{P}\left( X_1 \in [x,  x + \delta] \right)}.
        \end{aligned}
    \end{equation}
    
    On the one hand,
    \begin{equation}
        \begin{aligned}
        \mathbb{P}(T=t_j, X_1 \in &[x,  x + \delta]  ) = \sum_{k=0}^{K+1} \frac{n_k}{n} \mathbb{P}_{k}(T=t_j, X_1 \in [x,  x + \delta] ) \\
        &= \frac{n_j}{n} \ \mathbb{P}_{j}(T=t_j, X_1 \in [x,  x + \delta] )
        + \frac{n_{K+1}}{n} \ \mathbb{P}_{K+1}(T=t_j, X_1 \in [x,  x + \delta] ) \\
        &= \frac{n_j}{n} \ \mathbb{P}_{j}(X_1 \in [x,  x + \delta] )
        + \frac{n_{K+1}}{n} \  \mathbb{P}_{K+1}(T=t_j)  \mathbb{P}_{K+1}(X_1 \in [x,  x + \delta] ) \\
        &= \frac{1}{2(K+1)} \ \mathbb{P}_{j}(X_1 \in [x,  x + \delta] ) + \frac{1}{2(K+1)}  \ \mathbb{P}_{K+1}(X_1 \in [x,  x + \delta] ).
        \end{aligned}
    \end{equation}
    
    For $x \in \mathbb{R}$, there exists a unique $j_0$ such that $x \in I_{j_0}$. For $\delta$ small enough, we have $[x,x+\delta] \subset I_{j_0}$ and, consequently, $[x,x+\delta]\cap I_j=\emptyset$ for all $j \neq j_0$. This implies:
    \begin{equation}
         \mathbb{P}_{j}(X_1 \in [x,  x + \delta] ) = \frac{\mathbb{P}_{K+1}\left( X_1 \in [x,  x + \delta] , X_1 \in I_{j}  \right)}{\mathbb{P}_{K+1}\left(X_1 \in I_{j}  \right)} = \frac{\mathbb{P}_{K+1}\left( X_1 \in [x,  x + \delta] \right)}{\mathbb{P}_{K+1}\left(X_1 \in I_{j}   \right)} \mathbf{1}\{ j = j_0 \}.
    \end{equation}
    
    Therefore,
    \begin{equation}
        \begin{aligned}
        \mathbb{P}(T=t_j, X_1 \in [x,  x + \delta]  ) &= \frac{1}{2(K+1)} \mathbb{P}_{K+1}(X_1 \in [x,x+\delta]) \ (  \frac{\mathbf{1}\{ j = j_0 \}}{\mathbb{P}_{K+1}(X_1 \in I_{j_0})}+ 1 ) \\ 
        &= \Big(  \frac{1}{2} \mathbf{1}\{ j = j_0 \}  + \frac{1}{2(K+1)} \Big) \mathbb{P}_{K+1}(X_1 \in [x,x+\delta]).
    \end{aligned}
    \end{equation}
    
    On the other hand,
    \begin{equation}
    \begin{aligned}
        \mathbb{P} (X_1 \in [x,  x + \delta]) &= \sum_{k=0}^{K+1} \frac{n_{k}}{n} \mathbb{P}_{k}(X_1 \in [x,  x + \delta] ) \\
        &= \frac{1}{2(K+1)} \sum_{k=0}^{K}  \frac{\mathbb{P}_{K+1}\left( X_1 \in [x,  x + \delta] , X_1 \in I_k  \right)}{\mathbb{P}_{K+1}\left(X_1 \in I_k  \right)}  + \frac{1}{2} \mathbb{P}_{K+1}( X_1 \in [x,  x + \delta] )  \\
        &= \frac{1}{2(K+1)} \ \frac{ \mathbb{P}_{K+1} (X_1 \in [x,x+\delta]) }{ \mathbb{P}_{K+1}(X_1\in I_{j_0}) } +  \frac{1}{2} \mathbb{P}_{K+1}(X_1 \in [x,x+\delta]) \\
        &= \frac{1}{2} \mathbb{P}_{K+1}(X_1 \in [x,x+\delta]) + \frac{1}{2} \mathbb{P}_{K+1}(X_1 \in [x,x+\delta]) \\
        &= \mathbb{P}_{K+1}(X_1 \in [x,x+\delta])
    \end{aligned}
    \end{equation}
    
    Finally,
    \begin{equation}
    \begin{aligned}
        r(t_j, \boldsymbol{x}) &= \lim_{\delta \rightarrow 0} \frac{\mathbb{P}\left(T=t_j, X_1 \in [x,  x + \delta] \right)}{\mathbb{P}\left( X_1 \in [x,  x + \delta] \right)} \\
        &=  \lim_{\delta \rightarrow 0} \frac{ \Big(  \frac{1}{2} \mathbf{1}\{ j = j_0 \}  + \frac{1}{2(K+1)} \Big) \mathbb{P}_{K+1}(X_1 \in [x,  x + \delta]) }{ \mathbb{P}_{K+1}(X_1 \in [x,  x + \delta]) } \\
        &= \frac{1}{2} \mathbf{1}\{ j = j_0 \}  + \frac{1}{2(K+1)}  \\ 
        &= \frac{ (K+1) \mathbf{1}\{ j = j_0 \} + 1 }{2(K+1)}
        \\
        &= \left\{  \begin{array}{l l}
        \frac{K+2}{ 2(K+1)} & \quad \text{if $x \in I_j$,}\\[.15cm]
        \frac{1}{2(K+1)} & \quad \text{otherwise.}\\[.15cm] \end{array} \right. \\
    \end{aligned}
    \end{equation}
    
\end{proof}

\subsubsection*{The potential outcome models.}

The potential outcome models are given directly by the conditional mean.
For the linear model, $\mu_t$ satisfies for all $t \in \mathcal{T}$ and $x \in [0,1]$:
\begin{equation}
    \mu_t({x}) = (1+t){x}.
\end{equation}

For the hazard rate model, $\mu_t$ is given by:
\begin{equation}
    \mu_t(\boldsymbol{x}) = t+\|\boldsymbol{x}\| \exp{(-t \|\boldsymbol{x}\|)}.
\end{equation}

\subsubsection*{The observed outcome models.}

For the linear model, the observed outcome model $m$ can be computed as:
\begin{equation}
    \begin{aligned}
        m({x}) &= \mathbb{E}( Y_{\mathrm{obs}} \mid {X} = {x}) \\
        &= \mathbb{E}( (1+T)\boldsymbol{X} \mid {X} = {x}) \\
        &= (1+\mathbb{E}( T\mid {X} = {x}) ) {x} \\
        &= \big(1+ \sum_{k=0}^{K} {r}(t_k,{x}) t_k  \big){x},
    \end{aligned}
\end{equation}
where $r$ is given by (\ref{eq:true_GPS_lin}).

and, for the hazard rate model, $m$ can be computed as:
\begin{equation}
    \begin{aligned}
        m(\boldsymbol{x}) &= \mathbb{E}( \mathbb{E} (Y_{\mathrm{obs}} \mid \boldsymbol{X}, T) \mid \boldsymbol{X} = \boldsymbol{x}) \\
        &= \mathbb{E}( T+\|\boldsymbol{X}\| \exp{(-T \|\boldsymbol{X}\|)}  \mid \boldsymbol{X} = \boldsymbol{x}) \\
        &= \mathbb{E}( T\mid \boldsymbol{X} = \boldsymbol{x}) + \|\boldsymbol{x}\| \ \mathbb{E}( \exp{(-T \|\boldsymbol{X}\|)} \mid \boldsymbol{X} = \boldsymbol{x}) \\
        &= \sum_{k=0}^{K} {r}(t_k,\boldsymbol{x}) t_k  + \sum_{k=0}^{K}  \|\boldsymbol{x}\| \ r(t_k,\boldsymbol{x}) \exp{(-t_k \|\boldsymbol{x}\|)},
    \end{aligned}
\end{equation}
where $r$ is given by (\ref{eq:true_GPS_HR}). 

\pagebreak

\subsection{Additional numerical results and plots.}
\label{App:D_2}

\vspace{0.08in}

In this section, we present the results of different simulations and scenarios for linear and hazard rate models with $K+1=10$, $n=2000$ for the linear model, and $n=10000$ for the Hazard rate model. In the randomized setting, the sample $\mathbf{D}_{\mathrm{obs}}$ is sampled randomly, and the propensity score is given by (\ref{eq:GPS_RCT}). In a non-randomized setting, the sample $\mathbf{D}_{\mathrm{obs}}$ is given by preferential selection as described in Section \ref{App:details_analytics} and the GPS is given by (\ref{eq:true_GPS_lin}). When we say that \textit{"the models' nuisance components are exact"}, we replace the expression of the estimators $\widehat{\mu}_t, \widehat{m}$ or $\widehat{r}$ with the expressions obtained in Section \ref{App:details_analytics}.

\subsection*{Linear model in a randomized setting.}
    
\begin{table*}[h!] 
\caption{\textbf{mPEHE} and \textbf{sdPEHE} for three different ML base-learners; Case where nuisance components are exact.}
\centering
\label{tab:Table2}
\begin{tabular}{ c  c  c  c }
  \toprule
  Meta-learner & XGBoost & RandomForest & Linear Model \\
  \midrule
  M-Learner & 2.23 (1.20) & 2.09 (1.08) & 0.087 (0.096) \\
  DR-Learner & 0.165 (0.034) & 0.140 (0.033) & 9.65 (7.84) $10^{-3}$\\
  X-Learner & \textbf{0.022 (0.004)} & \textbf{0.029 (0.004)} & \textbf{1.42 (1.46)} $\mathbf{10^{-3}}$\\
  \midrule
  RLin-Learner & \multicolumn{3}{c}{10.2 (8.42) $10^{-3}$ }\\
  \bottomrule
\end{tabular}
\end{table*}

\begin{table*}[h!] 
    \begin{center}
    \caption{\textbf{mPEHE} and \textbf{sdPEHE} for three different ML base-learners; Case where nuisance components are well-specified.}
    \centering
    \label{tab:Table3}
    \begin{tabular}{ c  c  c  c  c }
      \toprule
      Meta-learner & XGBoost & RandomForest & Linear Model \\
      \midrule
      T-Learner & 0.065 (0.019) & 0.041 (0.016) & 10.0 (8.37) $10^{-3}$\\
      S-Learner & \textbf{0.033 (0.018) } & 0.032 (0.028) & \textbf{3.03} (2.42) $\mathbf{10^{-3}}$\\
      \textit{Nv}X-Learner & 0.060 (0.019) & \textbf{0.037 (0.016)} &  10.0 (8.37) $10^{-3}$\\
      \midrule
      M-Learner & 1.25 (0.610) & 1.22 (0.621) & 0.201 (0.191) \\
      DR-Learner & 0.068 (0.019)  -- 0.063 (0.020) & 0.068 (0.018) -- 0.068 (0.018) & 10.0 (9.14) -- \textbf{5.27 (4.36)} $\mathbf{10^{-3}}$ \\
      X-Learner & 0.063 (0.020)  -- \textbf{0.033 (0.017)} & 0.045 (0.016) -- 0.061 (0.040) & 10.0 (8.37) -- {3.28 (2.98)} $10^{-3}$ \\
      \midrule
      RLin-Learner & 0.135 (0.130) & 0.137 (0.128) & 0.073 (0.063)\\
      \bottomrule
    \end{tabular}
    \end{center}
    \begin{tablenotes}[flushleft]
	    \scriptsize
	    \item For the DR- and  X-learners: $\mu_t$ are estimated by T-learning (left value) or S-learning (right value). 
    \end{tablenotes}
\end{table*}

\begin{table*}[h!] 
\caption{\textbf{mPEHE} and \textbf{sdPEHE} for three different ML base-learners; Case where the propensity score is misspecified.}
\centering
\label{tab:Table4}
\begin{tabular}{ c  c  c  c }
  \toprule
  Meta-learner & XGBoost & RandomForest & Linear Model \\
  \midrule
  M-Learner & 3.86 (2.95) & 3.68 (2.80) & 1.45 (0.99)\\
  DR-Learner & \textbf{0.145 (0.108)} & \textbf{0.245 (0.179)} & \textbf{0.014 (0.015)}\\
  X-Learner & \textcolor{lightgray}{0.033 (0.017)} & \textcolor{lightgray}{0.061 (0.040)} & \textcolor{lightgray}{3.28 (2.98) $10^{-3}$}\\
  \midrule
  RLin-Learner & 0.336 (0.272) & 0.338 (0.277) & 0.338 (0.215)\\
  \bottomrule
\end{tabular}
\end{table*}

\begin{table*}[h!] 
\caption{\textbf{mPEHE} and \textbf{sdPEHE} for three different ML base-learners; Case where the outcome models are misspecified.}
\centering
\label{tab:Table5}
\begin{tabular}{ c  c  c  c }
  \toprule
  Meta-learner & XGBoost & RandomForest & Linear Model \\
  \midrule
  M-Learner & \textcolor{lightgray}{1.25 (0.610)} & \textcolor{lightgray}{1.22 (0.621)} & \textcolor{lightgray}{0.201 (0.191)}\\
  DR-Learner & 0.811 (0.386) & 0.888 (0.378) & 0.308 (0.362)\\
  X-Learner & 0.304 (0.330) & 0.303 (0.330) & 0.275 (0.328)\\
  \midrule
  RLin-Learner & \multicolumn{3}{c}{\textbf{ 0.073 (0.062)}} \\
  \bottomrule
\end{tabular}
\end{table*}

\begin{table*}[h!] 
\caption{\textbf{mPEHE} and \textbf{sdPEHE} for three different ML base-learners; Case where nuisance components are misspecified.}
\centering
\label{tab:Table6}
\begin{tabular}{ c  c  c  c }
  \toprule
  Meta-learner & XGBoost & RandomForest & Linear Model \\
  \midrule
  M-Learner & 3.86 (2.95) & 3.68 (2.80) & 1.45 (0.99) \\
  DR-Learner & 1.87 (1.31) & 2.09 (1.48) & 0.828 (0.496) \\
  X-Learner & 0.304 (0.330) & 0.303 (0.330) & 0.275 (0.328) \\
  \midrule
  RLin-Learner & \multicolumn{3}{c}{\textbf{ 0.277 (0.178)}} \\
  \bottomrule
\end{tabular}
\end{table*}

\newpage

\subsection*{Linear model in non-randomized setting}

\begin{table*}[h!] 
\caption{\textbf{mPEHE} and \textbf{sdPEHE} for three different ML base-learners; Case where nuisance components are exact.}
\centering
\label{tab:Table7}
\begin{tabular}{ c  c  c  c }
  \toprule
  Meta-learner & XGBoost & RandomForest & Linear Model \\
  \midrule
  M-Learner & 3.69 (1.80) & 2.96 (1.47) & 0.153 (0.177) \\
  DR-Learner & 0.276 (0.081) & 0.206 (0.056) &  10.9 (8.47) $10^{-3}$\\
  X-Learner & \textbf{0.022 (0.004)} & \textbf{0.028 (0.004)} & \textbf{1.69 (1.11)}  $\mathbf{10^{-3}}$\\
  \midrule
  RLin-Learner & \multicolumn{3}{c}{11.0 (11.1) ${10^{-3}}$}\\
  \bottomrule
\end{tabular}
\end{table*}

\begin{table*}[h!] 
    \begin{center}
    \caption{\textbf{mPEHE} and \textbf{sdPEHE} for three different ML base-learners; Case where nuisance components are well-specified.}
    \centering
    \label{tab:Table8}
    \begin{tabular}{ c  c  c  c }
      \toprule
      Meta-learner & XGBoost & RandomForest & Linear Model \\
      \midrule
      T-Learner & 0.067 (0.023) & 0.043 (0.016) & 10.5 (10.0) $10^{-3}$\\
      \textit{Reg}T-Learner & 0.059 (0.021) & 0.042 (0.016) & 13.0 (11.1) $10^{-3}$\\
      S-Learner & \textbf{0.033 (0.018)} & 0.060 (0.055) & \textbf{6.46 (5.11)} $\mathbf{10^{-3}}$\\
      \textit{Nv}X-Learner & 0.062 (0.023) & \textbf{0.039 (0.017)} & 9.56 (10.0) $10^{-3}$\\
      \midrule
      M-Learner & 1.35 (0.82) & 1.14 (0.72) & 0.196 (0.153) \\
      DR-Learner & 0.065 (0.022) -- 0.065 (0.027) & 0.069 (0.026) -- 0.096 (0.056) & 13.0 (11.1) -- 8.17 (6.30) $10^{-3}$\\
      X-Learner & 0.059 (0.021) -- \textbf{0.034 (0.017)} & 0.046 (0.016) -- 0.084 (0.058) & 14.7 (11.6) -- 6.49 (5.23) $10^{-3}$\\
      \midrule
      RLin-Learner & 0.155 (0.137) & 0.124 (0.114) & 0.108 (0.097)\\
      \bottomrule
    \end{tabular}
    \end{center}
    \begin{tablenotes}[flushleft]
	    \scriptsize
	    \item For the DR- and  X-learners: $\mu_t$ are estimated by T-learning (left value) or S-learning (right value). 
    \end{tablenotes}
\end{table*}

\subsection*{Hazard rate model in randomized setting}

\begin{table*}[h!] 
\caption{\textbf{mPEHE} and \textbf{sdPEHE} for three different ML base-learners; Case where nuisance components are exact.}
\centering
\label{tab:Table9}
\begin{tabular}{ c  c  c  c }
  \toprule
  Meta-learner & XGBoost & RandomForest & Linear Model \\
  \midrule
  M-Learner & 4.27 (1.45) & 4.21 (1.28) & 0.529 (0.188)\\
  DR-Learner & 0.127 (0.022) & 0.144 (0.044) &  0.106 (0.094)\\
  X-Learner & \textbf{0.045 (0.025)} & \textbf{0.087 (0.049)} & \textbf{0.106 (0.094)}\\
  \midrule
  RLin-Learner & \multicolumn{3}{c}{0.107 (0.094)}\\
  \bottomrule
\end{tabular}
\end{table*}

\begin{table*}[h!]
    \begin{center}
    \caption{\textbf{mPEHE} and \textbf{sdPEHE} for three different ML base-learners; Case where nuisance components are well-specified.}
    \centering
    \label{tab:Table10}
    \begin{tabular}{ c  c  c  c }
      \toprule
      Meta-learner & XGBoost & RandomForest & Linear Model \\
      \midrule
      T-Learner & 0.175 (0.046) & 0.263 (0.144) & \textbf{0.113 (0.091)} \\
      S-Learner & 0.159 (0.048) & 0.260 (0.130) & 0.662 (0.421) \\
      \textit{Nv}X-Learner & 0.176 (0.091) & 0.313 (0.188) & \textbf{0.113 (0.092)}  \\
      \midrule
      M-Learner & 1.57 (0.471) & 1.79 (0.453) & 0.824 (0.522)\\
      DR-Learner & 0.165 (0.049) -- 0.159 (0.047) & 0.281 (0.144) -- 0.275 (0.137) & 0.114 (0.094) -- 0.464 (0.286) \\
      X-Learner & 0.163 (0.057) -- \textbf{0.154 (0.051)} & 0.279 (0.157) -- 0.279 (0.146) & \textbf{0.113 (0.092)} -- 0.644 (0.380)\\
      \midrule
      RLin-Learner & 0.245 (0.136) & 0.241 (0.136) & 0.717 (0.450)\\
      \bottomrule
    \end{tabular}
    \end{center}
    \begin{tablenotes}[flushleft]
	    \scriptsize
	    \item For the DR- and  X-learners: $\mu_t$ are estimated by T-learning (left value) or S-learning (right value). 
    \end{tablenotes}
\end{table*}

\newpage

\subsection*{Hazard rate model in non-randomized setting}

\begin{table*}[h!] 
\caption{\textbf{mPEHE} and \textbf{sdPEHE} for three different ML base-learners; Case where nuisance components are exact.}
\centering
\label{tab:Table11}
\begin{tabular}{ c  c  c  c }
  \toprule
  Meta-learner & XGBoost & RandomForest & Linear Model \\
  \midrule
  M-Learner & 6.28 (1.88) & 5.74 (1.60) & 3.72 (1.42) \\
  DR-Learner & 0.138 (0.029) & 0.139 (0.044) &  \textbf{0.110 (0.097)}\\
  X-Learner &  \textbf{0.044 (0.025)} & \textbf{0.087 (0.050)} & \textbf{0.110 (0.097)}\\
  \midrule
  RLin-Learner & \multicolumn{3}{c}{0.299 (0.176)}\\
  \bottomrule
\end{tabular}
\end{table*}

\begin{table*}[h!]
    \begin{center}
    \caption{\textbf{mPEHE} and \textbf{sdPEHE} for three different ML base-learners; Case where nuisance components are well-specified.}
    \centering
    \label{tab:Table12} 
    \begin{tabular}{ c  c  c  c }
      \toprule
      Meta-learner & XGboost & RandomForest & Linear Model \\
      \midrule
      T-Learner & 0.183 (0.039) & 0.286 (0.155) & 0.129 (0.094) \\
      \textit{Reg}T-Learner & 0.176 (0.044) & 0.286 (0.155)  & \textbf{0.121 (0.098)}  \\
      S-Learner & 0.176 (0.056) & 0.306 (0.153) & 0.671 (0.428) \\
      \textit{Nv}X-Learner & 0.190 (0.096) & 0.336 (0.200) & 0.129 (0.094)  \\
      \midrule
      M-Learner & 1.61 (0.505) & 1.58 (0.472) & 0.906 (0.557)\\
      DR-Learner & 0.168 (0.045) - 0.178 (0.048) &  0.304 (0.158) -- 0.322 (0.162) & \textbf{0.121 (0.098)} -- 0.518 (0.327) \\
      X-Learner & \textbf{0.167 (0.053)} -- 0.172 (0.057)  &  0.302 (0.169) -- 0.332 (0.167)  & 0.120 (0.094) -- 0.652 (0.388)\\
      \midrule
      RLin-Learner & 0.231 (0.081) & \textbf{0.186 (0.123)} &  1.05 (0.651) \\
      \bottomrule
    \end{tabular}
    \end{center}
    \begin{tablenotes}[flushleft]
	\scriptsize
	\item For the DR- and X-learners: $\mu_t$ are estimated by T-learning (left value) or S-learning (right value). 
    \end{tablenotes}
\end{table*}

\newpage

\subsection{Asymptotic performances when $K$ increases.}
\label{App:D_3}

\vspace{0.08in}

In this subsection, we consider the effect of increasing $K$ on the hazard rate function with XGBoost. For each value $K$, we sample $J=10$ different non-randomized samples following the \textit{preferential selection} as defined previously in Appendix D.1. The \textbf{mPEHE} is then computed by averaging the \textbf{mPEHE} over the $J=10$ samples. The results are drawn in the figure below.

\begin{figure*}[h!]
    \centering
     \subfigure[]{%
        \includegraphics[width=0.48\linewidth, height=0.75\textwidth]{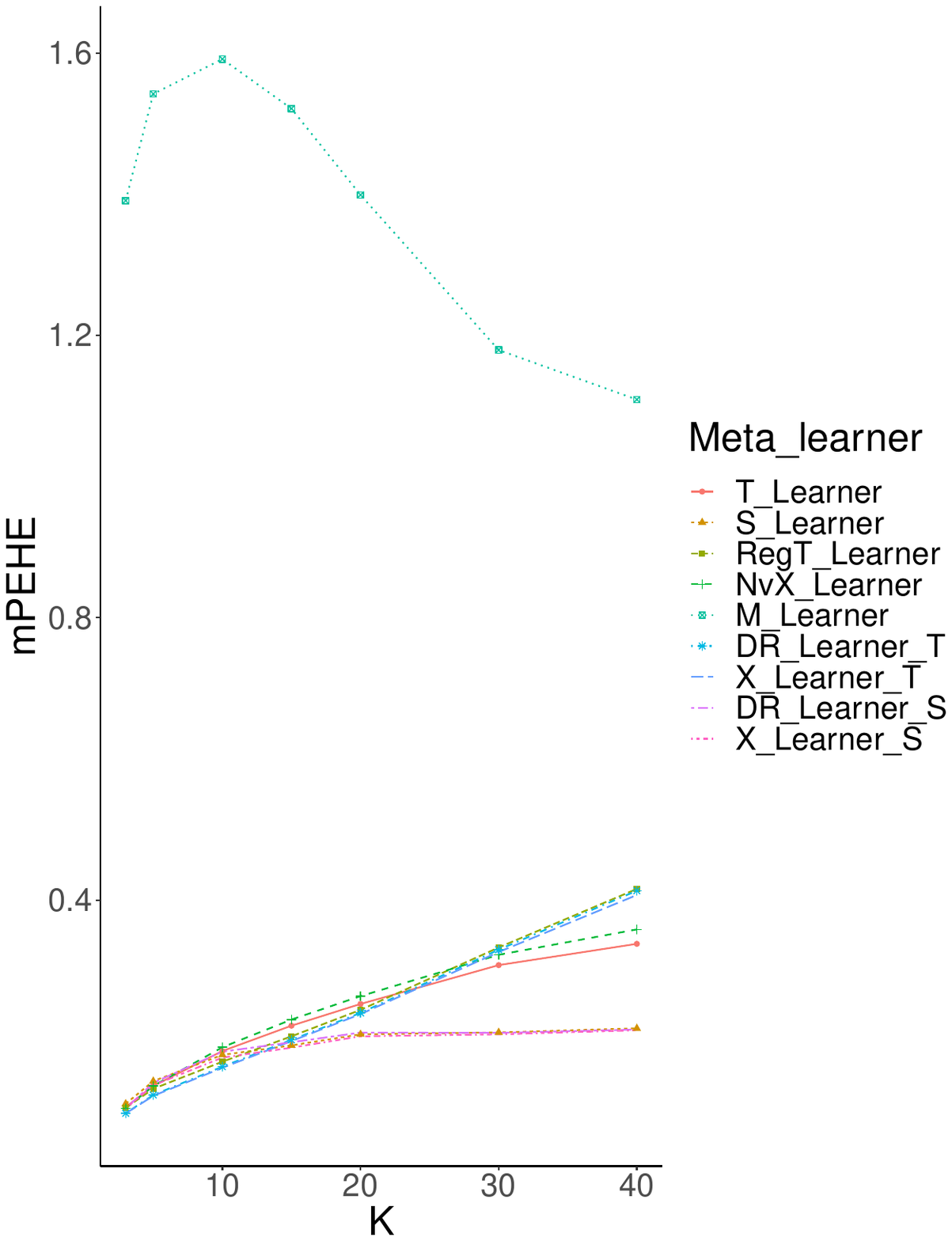}\label{fig:K_All_XGB}
        }
     \subfigure[]{
        \includegraphics[width=0.48\linewidth, height=0.75\textwidth]{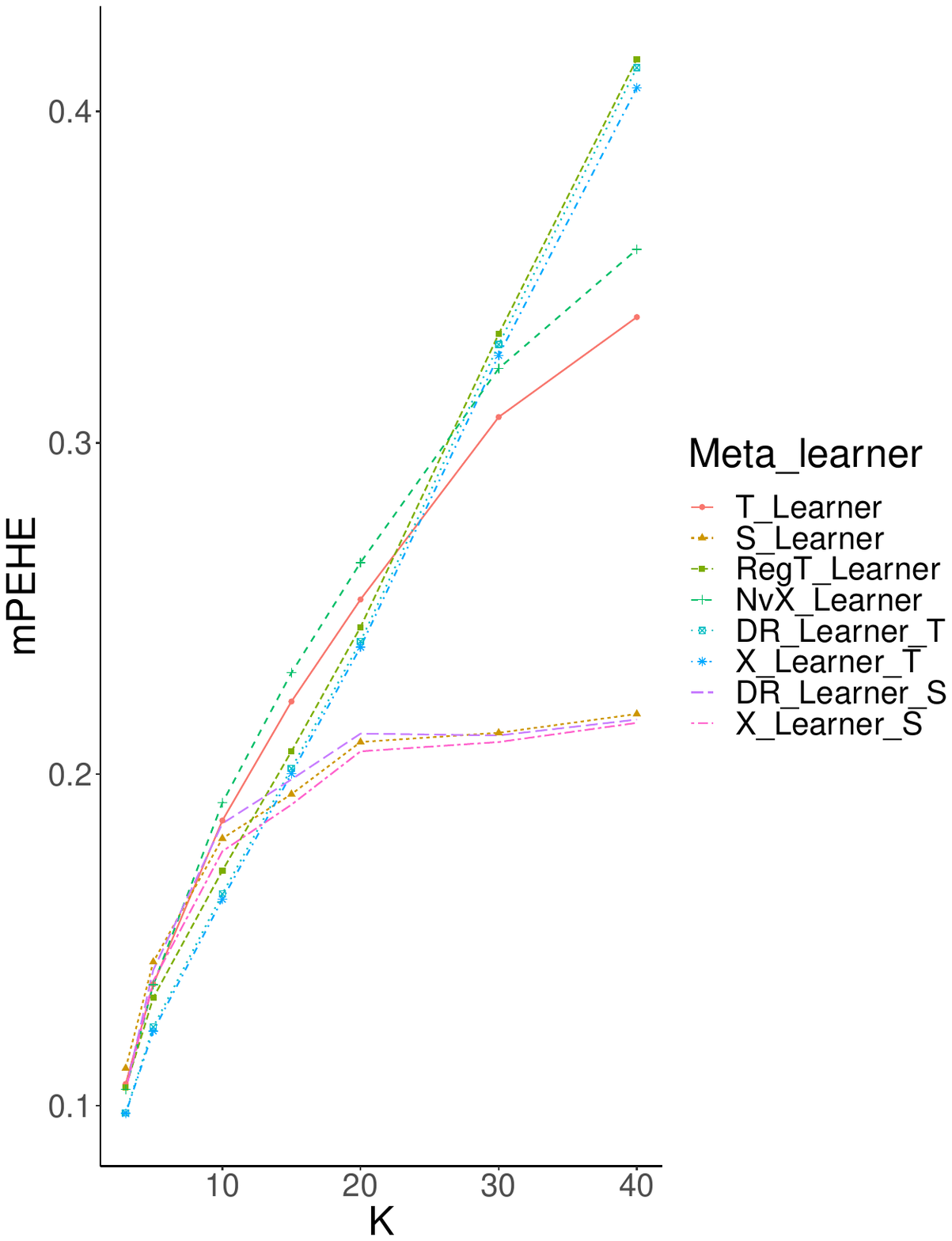}\label{fig:K_noM_T_XGB}
        }
    \caption{Variation of meta-learners' performances as functions of the number of possible treatment values $K$ for the hazard rate function in an observational design setting. \subref{fig:K_All_XGB}: All meta-learners; \subref{fig:K_noM_T_XGB}: Without the M-learner}
    \label{fig:K_asymp_XGB}
\end{figure*}

When we consider the effect of increasing $K$ on the hazard rate function with a linear model (with $p=2$), we notice the expected effect of $K$ on the M-learner: The error bound is increasing with $K$. This means that the behaviour of the M-learner with XGBoost can be explained by the over-fitting of the base-learner.

\begin{figure*}[h!]
    \centering
     \subfigure[]{%
        \includegraphics[width=0.48\linewidth, height=0.75\textwidth]{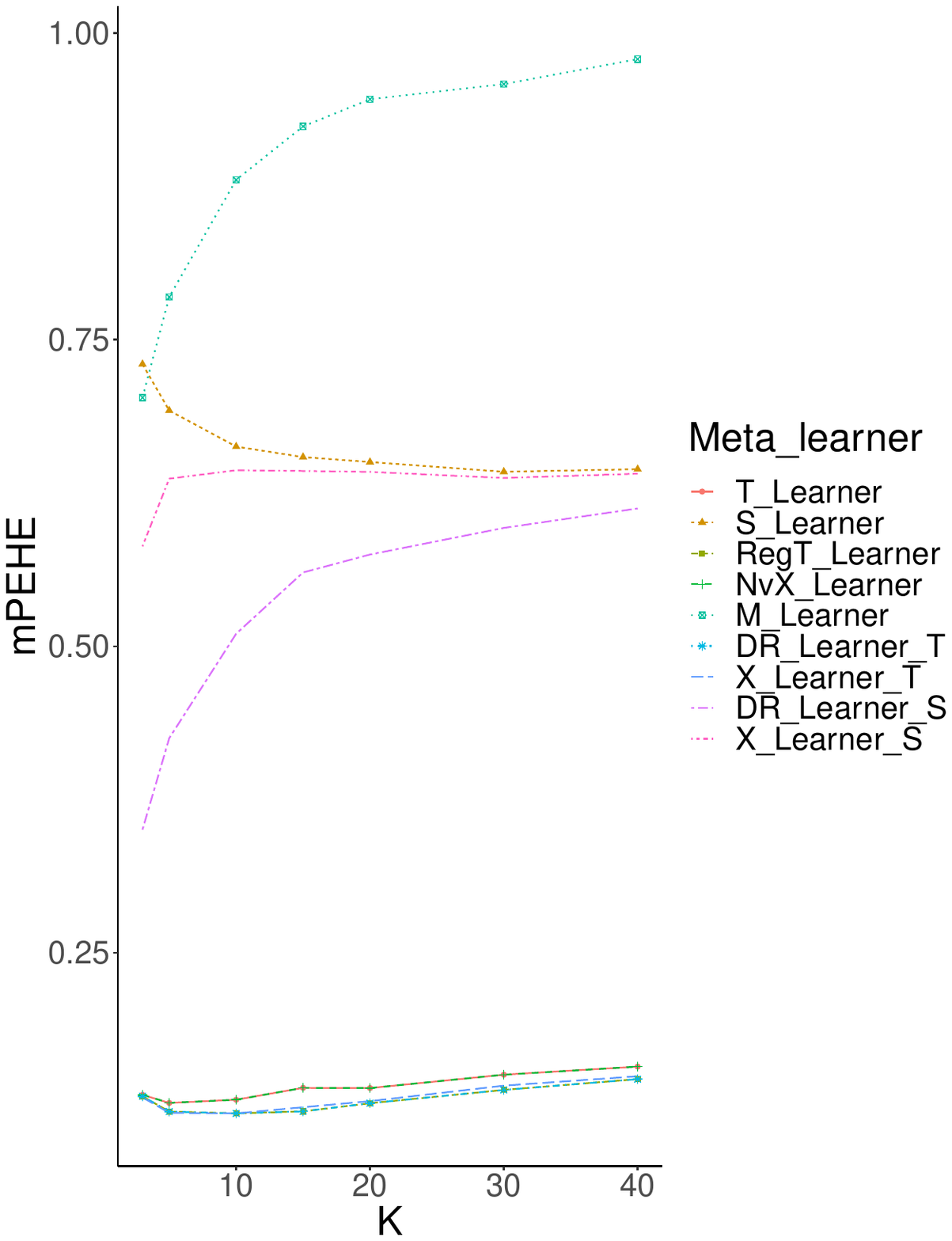}\label{fig:K_All_lm}
        }
     \subfigure[]{
        \includegraphics[width=0.48\linewidth, height=0.75\textwidth]{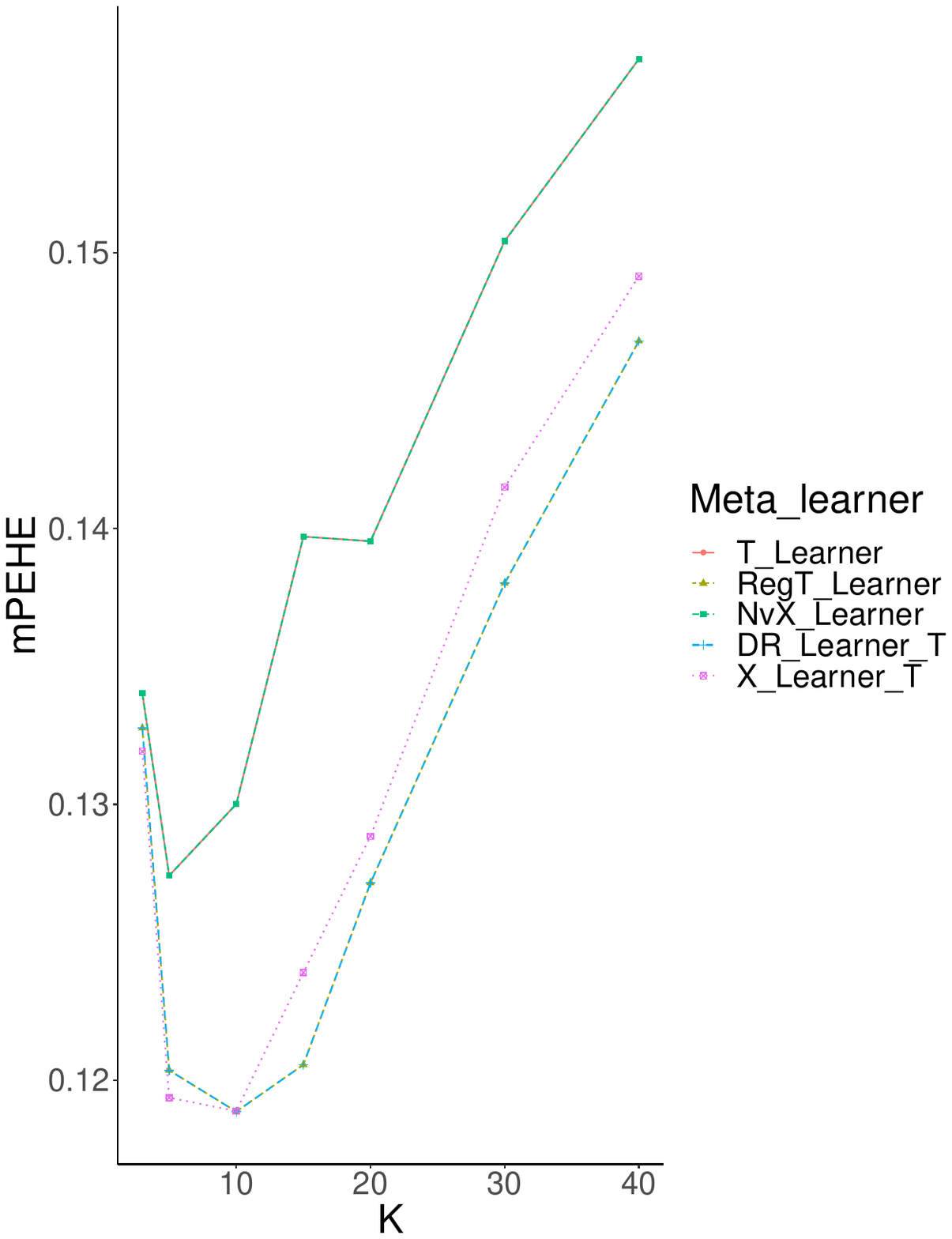}\label{fig:K_noM_T_lm}
        }
    \caption{Variation of meta-learner's performances when the number of possible treatment values $K$ for the hazard rate function in observational design setting with a linear model. \subref{fig:K_All_lm}: All meta-learners; \subref{fig:K_noM_T_lm}: Focus on the T-learning methods.}
    \label{fig:K_asymp_lm}
\end{figure*}

\newpage

\section{Description of the semi-synthetic dataset.}
\label{App:EGS_data}

\vspace{0.1in}

\subsection*{Motivation}

The difficulty in evaluating a causal model's performance in real-world applications motivates the need to create a semi-synthetic dataset. In this subsection, we consider a multistage fracturing Enhanced Geothermal System (EGS). 

Enhanced Geothermal Systems (EGS) are geothermal wells that generate geothermal energy by creating fluid connectivity in low-permeability conductive rocks through hydraulic, thermal, or chemical stimulation. The EGS concept involves extracting heat by constructing a subsurface fracture system to which water can be added via injection wells. Indeed, rocks are permeable due to slight fractures and pore spaces between mineral grains, and the injected water is heated by contact with the rock and returns to the surface through production wells. Moreover, Enhanced geothermal systems (EGS) have a high potential for developing and supplying renewable energy sources that are more efficient and cheaper than traditional hydrocarbon resources.

For energy companies, the goal is to optimize the design of the geothermal well (fracture spacing, Lateral Length etc.) to generate the maximum geothermal energy. However, some economic and operational problems present challenges: On the one hand, if the fractures are too small or too few, rocks will not be exploited sufficiently. On the other hand, if the number of fractures in a given rock is too high, the fractures may cool down faster. We would have a costly design that will not maximize the extracted heat.

We assume that the heat extraction performance of the EGS satisfies the following physical model:
\begin{equation}
\label{eq:Well_model}
    Q_{well} = Q_{fracture} \times {\ell_L}/{d} \times \eta_{d},
\end{equation}
where $Q_{well}$ is the heat extraction performance delivered by the well (output), $Q_{fracture}$ is the \textit{unknown} heat extraction performance from a single fracture that can be generated using a complex seven-parameter model, including reservoir characteristics and fracture design, $\ell_L$ is the Lateral Length of the well, $d$ is the average spacing between two fractures and $\eta_{d}$ is the stage efficiency penalizing the individual contribution when fractures are close to each other. We refer to Figure \ref{fig:EGS_DAG} for a graphical description of the EGS and its inputs/output. 

Finally, the model in (\ref{eq:Well_model}) respects the unconfoundedness assumption \ref{assump:unconfound}, and we can control all its variables in the simulations. We note that, in practice, all inputs are continuous with a given density. However, we discretize these variables in their input space to create a full factorial design.

\subsection*{Description of the data-set}

This section describes the data-generating process of our semi-synthetic dataset simulating the heat delivered by a multistage fracturing EGS. The process involved the creation of a conceptual reservoir model and the modelling of multiple wells' completion scenarios. The output (heat extraction performance) obtained from physics-based simulation experiments was tabulated with inputs in the semi-synthetic dataset.

The input data for the model were fabricated to ease confidentiality and non-disclosure information issues. However, data has been selected from reliable sources such as field observations, journals and books to be within the range of interest. Doing so allowed the building of a plain but representative reservoir model that would provide realistic results of an EGS.

The heat extraction performance from a single fracture (${Q}_{fracture}$) is determined using fracture length, fracture height, fracture width, fracture permeability, reservoir porosity, reservoir permeability and pore pressure. Modelling and simulation work were done using preprocessor and reservoir simulation tools PETREL and ECLIPSE.

The four physical parameters of the fracture were investigated, and the list of values used for each parameter can be observed in Table \ref{tab:Table_Frac}. In the end, $10 \times 10 \times 2 \times 3 = 600$ fracture's simulation cases have been realized.

\begin{table}[h!] 
\caption{Fracture parameters and their range of variation
for simulations.}
\vskip 0.10in
\centering
\label{tab:Table_Frac}
\begin{tabular}{ c  c }
  \toprule
  Variable & Range of variation  \\
  \midrule
    Fracture length (ft) & $[100, 1000]$ by a step of $100$ ft \\ 
    Fracture height (ft) & $[50, 500]$ by a step of 50 ft \\
    Fracture width (in) & $\{0.1, 0.2\}$ \\
    Fracture Permeability (md) & $\{30000, 85000, 19000\}$ \\
  \bottomrule
\end{tabular}
\end{table}

To emulate distinct reservoir schemes, it was decided to vary three main parameters; porosity, permeability and pore pressure. For porosity and permeability, the simulator takes the minimum and maximum values and estimates the physical properties across the reservoir. Three different multipliers were applied to define three (Low, Base and High) scenarios. Concerning pore pressure, three specific values were defined to simulate under-normal, normal (base) and overpressure (high) gradient conditions.
Therefore, $3 \times 3 \times 3= 27$ possible scenarios were defined. Table \ref{tab:Table_reservoir} displays the range of minimum and maximum values for the three reservoir parameters to be varied. 

\begin{figure}[h!]
    \begin{center}
    \centerline{\includegraphics[width=16cm, height=10cm]{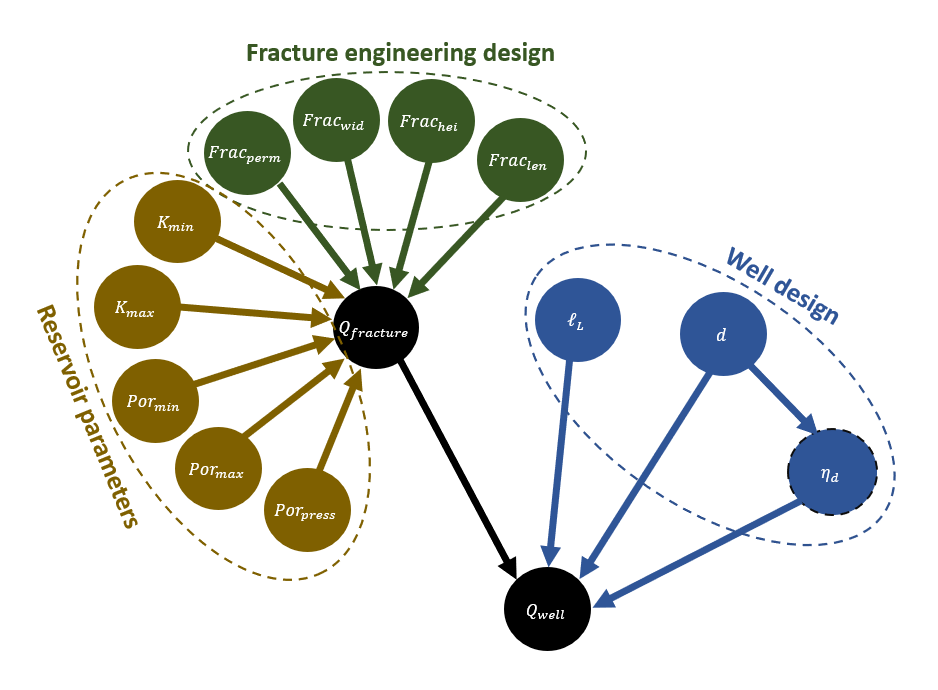}}
    \caption{The Causal DAG associated with the multistage EGS. Nodes in yellowish brown represent the reservoir characteristics, they can only be simulated, but in reality, we cannot intervene in these variables. Nodes in Dark green represent the fracture design. Engineers control them, and intervening in them is possible whenever there is a need to make a new fracture in the well. Nodes in blue represent a well's design and can be chosen arbitrarily by engineers or statisticians. Nodes in black denote the outputs. $ Q_{fracture}$ is only given by the simulator, whereas $Q_{well}$ is given by the physical model in (\ref{eq:Well_model}). Note that this graph contains nine nodes, but both $K_\mathrm{min}$ and $K_\mathrm{max}$ represent the same physical parameter $K$, and the same remark is valid for $Por_\mathrm{min}$ and $Por_\mathrm{max}$.}
    \label{fig:EGS_DAG}
    \end{center}
\end{figure}

\begin{table}[h!] 
\caption{Reservoir parameters and their range of variation for simulations.}
\vskip 0.10in
\label{tab:Table_reservoir}
\centering
\begin{tabular}{ c  c }
    \toprule
    Variable & Range of variation  \\
    \midrule
    $(\text{K}_\mathrm{min}, \text{K}_\mathrm{max})$ (md) & $\{ (0.0054, 0.0157), (0.054, 0.157), (0.109, 0.314)\}$ \\ 
    $(\text{Por}_\mathrm{min}, \text{Por}_\mathrm{max})$ (dec) & $\{ (0.0054, 0.0157), (0.054, 0.157), (0.109, 0.314)\}$ \\
    Pore pressure (psi) & $\{5000, 7000, 9000\}$ \\
  \bottomrule
\end{tabular}
\end{table}

By combining different reservoir scenarios with single fracture simulations, we obtained a single dataset with 16,200 possible cases for a fracture in a reservoir then we simulated the heat extraction performance for each experiment. Simulation's results were tabulated in the dataset "\textit{Single\_Fracture\_Simulation\_Cases\_16200.csv}".

The next step is to define well characteristics (lateral lengths and fracture spacing) to evaluate the heat extraction performance of the well when reservoir and fracture properties are not changed.

\begin{table}[h!] 
\caption{Well parameters and their range of variation.}
\vskip 0.10in
\label{tab:Table_well}
\centering
\begin{tabular}{ c  c }
    \toprule
    Variable & Range of variation \\
    \midrule
    Lateral length (ft) & $[2000, 14000]$ by a step of $1000$ ft \\ 
    Fracture spacing (ft) & $[100, 500]$ by a step of $100$ ft \\
  \bottomrule
\end{tabular}
\end{table}

Regarding the spacing efficiency coefficient, this coefficient was used to model interactions between fractures and penalize the heat extraction performance of a single fracture in the presence of other close fractures, that is, when the spacing between two fractures is small. Indeed, if the fractures are spaced too close, there may not be enough thermal energy in the rock to heat the water, decreasing the heat extraction efficiency. Modelling this efficiency led to the efficiency table "\textit{Fracture\_Efficency.csv}" that describes what would be the well's heat performance behaviour with respect to the fracture spacing selected. Based on this table, one can interpolate the efficiency to draw the curve (see Figure \ref{fig:stage_eff}) and thus obtain the spacing efficiency coefficient for any desired value fracture spacing.

\begin{figure}[h!]
    \vskip 0.2in
    \begin{center}
    \centerline{\includegraphics[width=12cm, height=8cm]{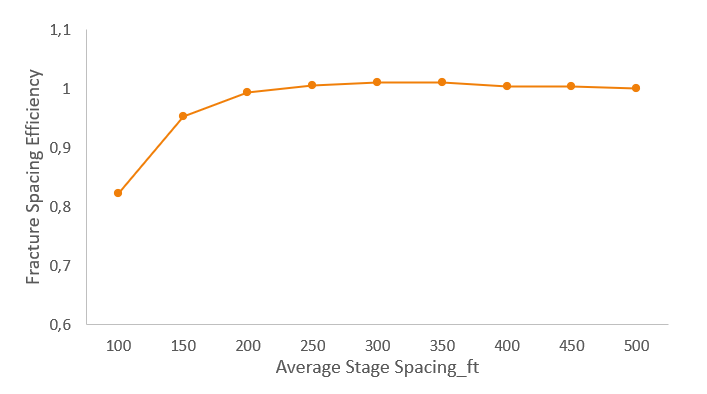}}
    \caption{Cross plot between fracture spacing efficiency and average stage spacing.}
    \label{fig:stage_eff}
    \end{center}
\end{figure}

The final generation of the semi-synthetic dataset "\textit{Main\_Dataset.csv}" was achieved by combining two main tables created using the R programming language. This table allows calculating the heat performance of a well for any lateral length and fracture spacing between 500 ft and 100 ft with the associated spacing efficiency coefficient defined in the efficiency table, following the physical model in (\ref{eq:Well_model}).

The three datasets are available at \url{https://github.com/nacharki/multipleT-MetaLearners}.

Finally, we emphasize that this study's design methodology focused on generating a semi-synthetic dataset using reservoir numerical simulation and creating a new benchmarking dataset for comparing and validating causal inference methods. Indeed, following the last step of forming the final dataset "\textit{Main\_Dataset.csv}", any user can define different distributions (with different values) on lateral lengths in the range $[2000, 14000]$ and fracture spacing in range $[100, 500]$, pick-up the corresponding spacing efficiency coefficients using the curve drawn in Figure \ref{fig:stage_eff} and generate a new semi-synthetic dataset by extrapolating them with "\textit{Single\_Fracture\_Simulation\_Cases\_16200.csv}" dataset. 

\subsection*{The creation of a non-randomized biased dataset.}

The idea of this step was to create a collection of biased data from the main semi-synthetic dataset to emulate observational data found in real-world situations. For example, geothermal wells with larger lateral lengths are likely to have more fractures (expensive wells are located in better geological areas). The opposite is seen for smaller wells that tend to be associated with fewer fractures. This situation creates a discrepancy between what engineers expect from physical models and what they observe in the field data.
The biased data, with 9,992 observations, was generated by following the \textit{preferential selection} strategy from the main dataset. Figure \ref{fig:GTM_vs_Biased} shows the difference between the \textit{real} heat extraction performance of the EGS and the observed heat extraction performance on the field: low (under-estimated) heat performance for small wells and high (over-estimated) heat performance for large wells.

\begin{figure}[h!]
    \vskip 0.2in
    \begin{center}
    \centerline{\includegraphics[width=12cm, height=8cm]{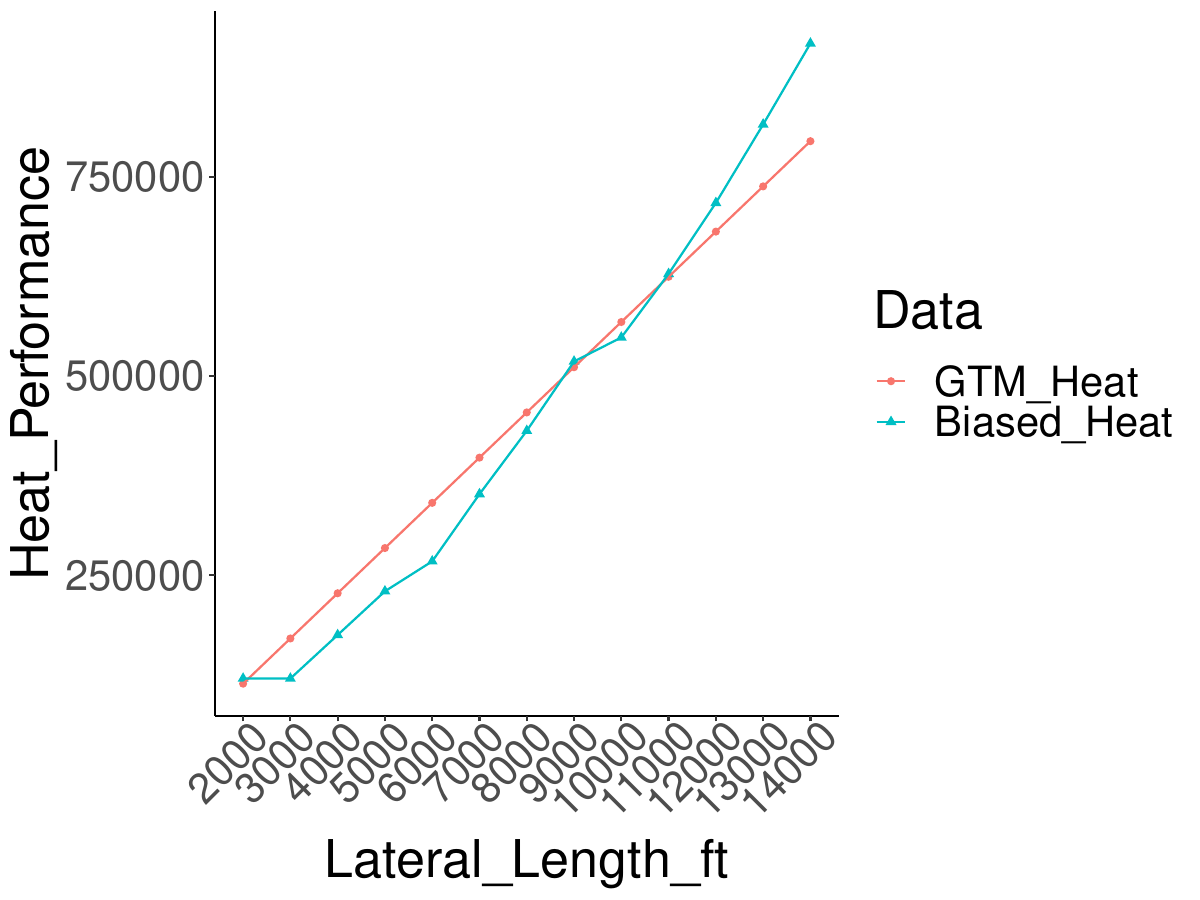}}
    \caption{An illustration of selection bias on the heat performance. Red line: The heat extraction performance on the main dataset (i.e. Ground Truth Model). Blue line: The heat performance on the biased dataset (i.e. observed response).  }
    \label{fig:GTM_vs_Biased}
    \end{center}
\end{figure}

\end{document}